\pdfoutput=1

\documentclass{article}

\usepackage{microtype}
\usepackage{graphicx}

\usepackage[bookmarksnumbered,unicode]{hyperref}

\usepackage[accepted]{icml2023}

\usepackage{amsmath}
\usepackage{amssymb}
\usepackage{mathtools}
\usepackage{amsthm}

\usepackage[capitalize,noabbrev]{cleveref}

\makeatletter
\def\@acmplainbodyfont{\itshape}
\def\@acmplainindent{\z@}
\def\@acmplainheadfont{\bfseries}
\def\@acmplainnotefont{\normalfont}
\newtheoremstyle{acmplain}%
  {.5\baselineskip\@plus.2\baselineskip
    \@minus.2\baselineskip}
  {0pt}
  {\@acmplainbodyfont}
  {\@acmplainindent}
  {\@acmplainheadfont}
  {.}
  {.5em}
  {\thmname{#1}\thmnumber{ #2}\thmnote{ {\@acmplainnotefont(#3)}}}
\def\@acmdefinitionbodyfont{\normalfont}
\def\@acmdefinitionindent{\z@}
\def\@acmdefinitionheadfont{\bfseries}
\def\@acmdefinitionnotefont{\normalfont}
\newtheoremstyle{acmdefinition}%
  {.5\baselineskip\@plus.2\baselineskip
    \@minus.2\baselineskip}
  {0pt}
  {\@acmdefinitionbodyfont}
  {\@acmdefinitionindent}
  {\@acmdefinitionheadfont}
  {.}
  {.5em}
  {\thmname{#1}\thmnumber{ #2}\thmnote{ {\@acmdefinitionnotefont(#3)}}}
\makeatother
\makeatletter
\def\@acmproofbodyfont{\normalfont}
\def\@acmproofindent{\z@}
\def\@acmproofheadfont{\itshape}
\newtheoremstyle{acmproof}%
  {.5\baselineskip\@plus.2\baselineskip
    \@minus.2\baselineskip}
  {0pt}
  {\@acmproofbodyfont}
  {\@acmproofindent}
  {\@acmproofheadfont}
  {.}
  {.5em}
  {\ifstrempty{#3}{\thmname{#1}}{#3}}
\theoremstyle{acmproof}
\newtheorem{proofnew}{Proof}[section]
\renewenvironment{proof}
               {\pushQED{\qed}\proofnew}
               {\popQED\endproofnew}  
\makeatother

\theoremstyle{acmplain}
\newtheorem{theorem}{Theorem}[section]

\newtheorem{lemma}[theorem]{Lemma}

\theoremstyle{acmdefinition}
\newtheorem{definition}[theorem]{Definition}

\usepackage{mdframed}
\definecolor{gray}{rgb}{0.96, 0.96, 0.96}
\makeatletter
\mdfdefinestyle{mystyle}{backgroundcolor=gray, linewidth=0pt,
  innertopmargin=7pt, innerbottommargin=7pt, innerleftmargin=7pt, innerrightmargin=7pt,
  skipabove=0.8\baselineskip, skipbelow=0pt}
\makeatother
\newenvironment{theorem*}
  {\begin{theorem}}%
  {\end{theorem}}

\usepackage[textsize=tiny]{todonotes}

\usepackage{booktabs}   
\usepackage{subcaption} 
\input{packages}

\newcommand{\mathboldcommand}[1]{\mathbb{#1}}

\newcommand{\bbM}{\mathboldcommand{M}}
\newcommand{\bbN}{\mathboldcommand{N}}

\newcommand{\bbR}{\mathboldcommand{R}}

\usepackage{euscript}
\newcommand{\mathcalcommand}[1]{\mathcal{#1}}
\newcommand{\mcA}{\mathcalcommand{A}}

\newcommand{\mcF}{\mathcalcommand{F}}

\newcommand{\mcI}{\mathcalcommand{I}}

\newcommand{\mcO}{\mathcalcommand{O}}

\newcommand{\mcR}{\mathcalcommand{R}}
\newcommand{\mcS}{\mathcalcommand{S}}

\DeclareMathAlphabet{\mathpzc}{T1}{pzc}{m}{it}

\allowdisplaybreaks

\newcommand{\edit}[1]{{#1}}
\newcommand*{\commentout}[1]{}

\newlength{\parskiptrue}
\definecolor{lred}{rgb}{1.0, 0.5, 0.5}
\definecolor{lorange}{rgb}{1.00, 0.90, 0.20}
\definecolor{lgreen}{rgb}{0.35, 0.95, 0.35}
\definecolor{lime}{rgb}{0.9, 1.0, 0.6}
\definecolor{lblue}{rgb}{1.0, 0.85, 0.75}

\makeatletter
\newcommand*\wt[1]{\mathpalette\wthelper{#1}}
\newcommand*\wthelper[2]{%
        \hbox{\dimen@\accentfontxheight#1%
                \accentfontxheight#11.1\dimen@
                $\m@th#1\widetilde{#2}$%
                \accentfontxheight#1\dimen@
        }%
}
\newcommand*\accentfontxheight[1]{%
        \fontdimen5\ifx#1\displaystyle
                \textfont
        \else\ifx#1\textstyle
                \textfont
        \else\ifx#1\scriptstyle
                \scriptfont
        \else
                \scriptscriptfont
        \fi\fi\fi3
}
\newcommand*\wh[1]{\mathpalette\whhelper{#1}}
\newcommand*\whhelper[2]{%
        \hbox{\dimen@\accentfontxheight#1%
                \accentfontxheight#11.2\dimen@
                $\m@th#1\widehat{#2}$%
                \accentfontxheight#1\dimen@
        }%
}
\makeatother

\makeatletter
\newcommand{\oset}[3][0ex]{%
  \mathrel{\mathop{#3}\limits^{
    \vbox to#1{\kern-3\ex@
    \hbox{$\scriptstyle#2$}\vss}}}}
\makeatother

\newcommand*{\defeq}{\triangleq}
\newcommand*{\indc}[1]{\mathbf{1}_{#1}}

\newcommand*{\relu}{\mathrm{ReLU}}

\newcommand*{\intr}{\mathit{int}}
\newcommand*{\cl}{\mathit{cl}}
\newcommand*{\bd}{\mathit{bd}}
\newcommand*{\pbd}{\mathit{pbd}}
\newcommand*{\spann}{\mathit{span}}

\newcommand*{\trsp}{\mathsf{T}} 
\newcommand*{\dom}{\mathit{dom}}

\newcommand*{\code}[1]{{#1}} 
\newcommand*{\ttP}{\texttt{P}}
\newcommand*{\ttQ}{\texttt{Q}}

\newcommand*{\ttw}{\texttt{w}}
\newcommand*{\ttx}{\texttt{x}}
\newcommand*{\ttlp}{\texttt{(}}
\newcommand*{\ttrp}{\texttt{)}}

\newcommand*{\DF}[1]{{D{#1}}} 
\newcommand*{\PDF}[2]{{D_{#2}{#1}}} 
\newcommand*{\ADF}[1]{{D^\mathtt{AD}{#1}}} 
\newcommand*{\APDF}[1]{{\partial^\mathtt{AD}{#1}}}
\newcommand*{\adf}[1]{\ADF{#1}} 

\newcommand*{\ext}[1]{\wt{#1}}
\newcommand*{\sem}[1]{\llbracket {{#1}} \rrbracket}
\newcommand*{\semad}[1]{\llbracket {{#1}} \rrbracket^\mathtt{AD}}

\newcommand*{\ndf}[1]{ \mathsf{ndf}({#1}) } 
\newcommand*{\ncdf}[1]{ \mathsf{ncdf}({#1}) } 
\newcommand*{\bdz}[1]{ \mathsf{bdz}({#1}) } 
\newcommand*{\ndfM}[1]{ \mathsf{ndf}_\Omega({#1}) }
\newcommand*{\ndfR}[1]{ \mathsf{ndf}_\bbR({#1}) }
\newcommand*{\incM}[1]{ \mathsf{inc}_\Omega({#1}) }
\newcommand*{\incR}[1]{ \mathsf{inc}_\bbR({#1}) }
\newcommand*{\Idx}{    \mathsf{Idx}}
\newcommand*{\clR}{\mcR_{\mathsf{cl}}}

\renewcommand{\paragraph}[1]{{#1}}
\setlist{noitemsep, topsep=0pt} 

\icmltitlerunning{On the Correctness of Automatic Differentiation for Neural Networks with Machine-Representable Parameters}

\begin{document}
\twocolumn[
  \icmltitle{On the Correctness of Automatic Differentiation \\ for Neural Networks with Machine-Representable Parameters}
  



\begin{icmlauthorlist}
\icmlauthor{Wonyeol Lee}{stanford}
\icmlauthor{Sejun Park}{ku}
\icmlauthor{Alex Aiken}{stanford}
\end{icmlauthorlist}

\icmlaffiliation{stanford}{Stanford University, USA}
\icmlaffiliation{ku}{Korea University, South Korea}

\icmlcorrespondingauthor{Wonyeol Lee}{wonyeol.lee.cs@gmail.com}
\icmlcorrespondingauthor{Sejun Park}{sejun.park000@gmail.com}

\icmlkeywords{Machine Learning, ICML}

\vskip 0.3in
]



\printAffiliationsAndNotice{}  


\hypersetup{linktoc=all}
\etocdepthtag.toc{mtchapter}

\begin{abstract}
  Recent work has shown that \edit{forward- and reverse-mode automatic differentiation (AD)} over the reals is almost always correct in a mathematically precise sense.
  However, actual programs work with {\em machine-representable numbers} (e.g., floating-point numbers), not reals.
  In this paper, we study the correctness of AD
  when the parameter space of a neural network consists solely of machine-representable numbers.
  \edit{%
  In particular, we analyze two sets of parameters on which AD can be incorrect:
  the incorrect set on which the network is differentiable but AD does not compute its derivative,
  and the non-differentiable set on which the network is non-differentiable.
  For a neural network {\em with bias parameters}, we first prove that the incorrect set is always empty.
  We then prove a tight bound on the size of the non-differentiable set,
  }%
  which is linear in the number of non-differentiabilities in activation functions,
  and give a simple necessary and sufficient condition for a parameter to be in this set.
  We further prove that AD always computes a Clarke subderivative even on the non-differentiable set.
  We also extend these results to neural networks possibly {without bias parameters}.
\end{abstract}

\section{Introduction}
\label{sec:intro}

\edit{Forward- and reverse-mode automatic differentiation (AD) are popular algorithms for computing the derivative of a function represented by a program \cite{GriewankW08}.} 
Diverse practical systems for AD have been developed
for general-purpose programs \cite{PearlmutterS08, Adolc12, Tapenade13, Autograd15, Adijac16, Juliadiff16, Diffsharp16},
and particularly for machine-learning programs \cite{Theano10, Torch11, Caffe14, CNTK16, Tangent18, Chainer19},
including TensorFlow \cite{Tensorflow16}, PyTorch \cite{Pytorch17}, and JAX \cite{Jax18b}. 
The development of such AD systems has been a driving force of the rapid advances
in deep learning (and machine learning in general) in the past 10 years
\cite{BaydinPRS17, LecunBH15, Schmidhuber15-DL}. 

Recently, the correctness of AD has been actively studied for various types of programs.
For programs that only use differentiable functions,
AD is correct {\em everywhere}, 
i.e., it computes the derivative of a given program at all inputs
\cite{AbadiP20, BrunelMP20, KrawiecJKEEF22, SmedingV23, RadulPFJM23, Elliott18, BartheCLG20, HuotSV20, Vakar21}.
On the other hand, for programs that use non-differentiable functions (e.g., $\relu$\footnote{$\relu(x) \defeq \max\{x,0\}$.}),
AD can be incorrect at some inputs \cite{KakadeL18}. 

There are two cases where AD is incorrect.
The first case is when the function $f$ represented by a given program is differentiable at some $x$,
but AD returns a value different from the derivative of $f$ at $x$. 
For instance, consider a program\footnote{It appeared in \citet{KakadeL18}.} 
that represents the identity function, defined as $\relu(x) - \relu(-x)$.
If AD uses zero as a ``derivative'' of $\relu$ at $x=0$, as is standard (e.g., in TensorFlow and PyTorch),
it returns zero for this program at $x=0$ while the true derivative is one.
The second case is when $f$ is non-differentiable at some $x$,
\edit{%
  but AD does not return a generalized notion of derivative (e.g., Clarke subdifferential) of $f$ at $x$.
  For example, $\relu(x) - \frac{1}{2}\relu(-x)$ represents a function that is non-differentiable at $x=0$
  with the Clarke subdifferential $[\frac{1}{2}, 1]$, but AD outputs $0$ at $x=0$.%
}%

Although AD can be incorrect, recent works show that 
for a large class of programs using non-differentiable functions, AD is correct {\em almost everywhere},
i.e., it is incorrect at most on a Lebesgue measure-zero subset of the input domain of a program 
\cite{BolteP20a, BolteP20b, LeeYRY20, HuotLMS22, MazzaP21}.

These prior works, however, have a limitation: they consider AD over the real numbers,
but in practice, inputs to a program are always {\em machine-representable numbers} such as $32$-bit floating-point numbers.
Since the set of machine-representable numbers is countable (and usually finite),
it is always a Lebesgue measure-zero subset of the real numbers.
Hence, AD could be incorrect on {\em all} machine-representable inputs according to prior works,
and this is indeed possible.
Consider a program\footnote{Inspired by \citet{BolteP20b, MazzaP21}.}
for a function from $\bbR$ to $\bbR$, defined as
\begin{gather*}
  \sum_{c \in \bbM}\! \Big[ \lambda x + \Big(\frac{1}{|\bbM|} - \lambda\Big)\!\Big(\relu(x-c) - \relu(-x+c)\Big) \Big],
\end{gather*}
where $\bbM \subseteq \bbR$ is a finite set of machine-representable numbers and $\lambda \in \bbR \setminus \{1\}$ is an arbitrary constant. 
Then, the program represents the affine function $x \mapsto x + a$ for $a = (\lambda - \smash{\frac{1}{|\bbM|}}) \times \sum_{c \in \bbM} c$,
but AD incorrectly computes its derivative at any $x \,\,{\in}\,\, \bbM$ as $\lambda$ (the arbitrarily chosen value) if zero is used as a ``derivative'' of $\relu$ at $0$ as before.%
\footnote{We can even make AD return different values at different $x \in \bbM$,
by using a different $\lambda_i$ for each $c_i \in \bbM$.
Similarly, we can also construct a program such that at all machine-representable numbers $\bbM$,
the program is non-differentiable and AD returns arbitrary values.}

Given these observations, we raise the following questions:
for a program that represents a neural network,
at which machine-representable inputs to the program (i.e., parameters to the network) can AD be incorrect,
and how many such inputs can there be?
In this work, we tackle these questions and present the first theoretical results. 
In particular, we study the two sets of machine-representable parameters of a neural network
on which AD \edit{can be} incorrect:
the {\em incorrect set}, on which the network is differentiable but AD does not compute its derivative,
and the {\em non-differentiable set}, on which the network is non-differentiable. 

{\bf Summary of results.}
We focus on neural networks consisting of alternating analytic pre-activation functions (e.g., fully-connected and convolution layers) 
and pointwise continuous activation functions (e.g., $\relu$ and $\mathrm{Sigmoid}$).
The first set of our results (\Cref{sec:bias}) is for such networks {\em with bias parameters} at every layer, and is summarized as follows.
\begin{itemize}[leftmargin=1em, itemsep=0.3em]
\item
  We prove that the incorrect set is {\em always empty}, not only over machine-representable parameters but also over real-valued ones.
  To our knowledge, this is the first result showing that the incorrect set can be empty
  for a class of neural networks using possibly non-differentiable functions; prior works only bounded the measure of this set.
\item 
  On the other hand, the non-differentiable set can be non-empty. 
  We give a tight bound on its density over all machine-representable parameters,
  which has the form $n/|\bbM|$ where $n$ is 
  the {\em total number of non-differentiable points} in activation functions.
  This result implies that in practice, the non-differentiable set often has a low density,
  especially if we use high-precision parameters (e.g., use $32$-bit floating-point numbers for $\bbM$, where $|\bbM|\;{\approx}\;2^{32}$). 
\item
  To better describe the non-differentiable set,
  we provide a simple, easily verifiable {\em necessary and sufficient condition} for a parameter to be in the non-differentiable set.
  Given that deciding the \edit{non-differentiability} of a neural network is NP-hard in general \cite{BolteBPP23},
  our result is surprising: having bias parameters is sufficient to efficiently decide the \edit{non-differentiability}.
\item 
  Given that the non-differentiable set can be non-empty, a natural question arises: what does AD compute on this set?
  We prove that AD {\em always computes a Clarke subderivative} (a generalized derivative)
  even on the non-differentiable set. That is, AD is an efficient algorithm for computing a Clarke subderivative in this case. 
\end{itemize}

The second set of our results (\Cref{sec:nobias}) extends the above results to neural networks possibly {\em without bias parameters} at some layers,
and is summarized as follows.
\begin{itemize}[leftmargin=1em, itemsep=0.3em]
\item
  As we observed in the $\relu(x)-\relu(-x)$ example, the incorrect set can be non-empty in this case. 
  Thus, we prove tight bounds on the density of both the incorrect and non-differentiable sets,
  which have the form $n'/|\bbM|$ 
  where $n'$ is linear in the total number of non-differentiable points in activation functions
  as well as the total number of boundary points in activation functions' zero sets.
\item
  We provide simple, easily verifiable sufficient conditions on parameters
  under which AD computes the standard derivative or a Clarke subderivative.
\end{itemize}

Our theoretical results carry two main practical implications: AD for neural networks is correct on most machine-representable parameters, and it is correct more often with bias parameters. 
For networks with bias parameters at all layers,
our results further provide an exact characterization of when AD is correct and what it computes.

\edit{%
  We remark that many of our results, especially all the results not about the density of certain sets,
  hold not only for machine-representable parameters but also for real-valued ones.
  On the other hand,
}%
our results may not be directly applicable to neural networks
with non-analytic pre-activation functions or non-pointwise activation functions;
we discuss such limitations in \Cref{sec:discuss}.

{\bf Organization.}
We first introduce notation and the problem setup (\Cref{sec:setup}).
We then present our main results for neural networks with bias parameters (\Cref{sec:bias})
and extend them to neural networks possibly without bias parameters (\Cref{sec:nobias}).
We conclude the paper with \edit{related work and discussion (\Cref{sec:related}--\ref{sec:concl})}.%

\section{Problem Setup}
\label{sec:setup}

\subsection{Notation and Definitions}
\label{sec:prelim}

We use the following notation and definitions. 
Let $\bbN$ and $\bbR$ be the sets of positive integers and real numbers, respectively.
For $n \in \bbN$, we use $[n] \defeq \{1, 2, \allowbreak \ldots, n\}$ and $\vec{0}_n \defeq (0, \ldots, 0) \in \bbR^n$,
and often drop $n$ from $\vec{0}_n$ when the subscript is clear from context.
For $x=(x_1, \allowbreak \ldots, \allowbreak x_n) \in \bbR^n$, we use $x_{-i}\defeq(x_1, \ldots, \allowbreak x_{i-1}, \allowbreak x_{i+1}, \allowbreak \ldots, \allowbreak x_n)$.
We call $A \subseteq \bbR$ an {\em interval} if it is $[a,b]$, $[a,b)$, $(a,b]$, or $(a,b)$ for some $a, b \in \bbR \cup \{\pm \infty\}$.
For $A \subseteq \bbR^n$, 
$\indc{A} : \bbR^n \to \{0,1\}$ denotes the indicator function of $A$.
We say that $f : \bbR^n \to \bbR^m$ is {\em analytic} if it is infinitely differentiable
and its Taylor series at any $x \in \bbR^n$ converges to $f$ on some neighborhood of $x$.
For any $f : \bbR^n \to \bbR^m$, 
\[\DF{f} : \bbR^n \to \bbR^{m \times n} \cup \{\bot\}\] 
denotes the standard derivative of $f$, where $f(x)=\bot$ denotes that $f$ is non-differentiable at $x$. 
Lastly, for $f:\bbR\to\bbR$, 
\begin{align*}
  \ndf{f}&\defeq\{x\in\bbR\mid\text{$f$ is non-differentiable at $x$}\},\\
  \bdz{f}&\defeq\bd(\{x\in\bbR\mid f(x)=0\})
\end{align*}
denote the set of non-differentiable points of $f$ and the boundary of the zero set of $f$, respectively.

\subsection{Neural Networks}
\label{sec:nn}

We define a neural network as follows. 
Given the number of layers $L \in \bbN$, let $N_0 \in \bbN$ be the dimension of input data,
$N_l \in \bbN$ and $W_l \in \bbN \cup\{0\}$ be the number of neurons and the number of parameters at layer $l \in [L]$,
and $N \defeq N_1 + \cdots + N_L$ and $W \defeq W_1 + \cdots + W_L$.
Further, for each $l \in [L]$,
let $\tau_l : \bbR^{N_{l-1}} \times \bbR^{W_l} \to \bbR^{N_l}$ be an analytic {\em pre-activation function}
and $\sigma_l : \bbR^{N_l} \to \bbR^{N_l}$ be a pointwise, continuous {\em activation function}, i.e.,
\begin{align*}
  \sigma_l(x_1,\dots,x_{N_l}) \defeq \big( \sigma_{l,1}(x_1), \ldots, \sigma_{l,N_l}(x_{N_l}) \big)
\end{align*}
for some continuous $\sigma_{l,i}: \bbR \to \bbR$.
Under this setup, we define a neural network as a function of model parameters: 
given input data $c \in \bbR^{N_0}$, a {\em neural network} $z_L(\,\cdot\,;c) : \bbR^{W} \to \bbR^{N_L}$ is defined as
\begin{align}
  \label{eq:nn-def}
  z_L(w;c) \defeq (\sigma_L \circ \tau_L^{\langle w_L\rangle} \circ \cdots \circ \sigma_1 \circ \tau_1^{\langle w_1\rangle})(c),
\end{align}
where $w \defeq (w_1, \dots, w_L)$, $w_l \defeq (w_{l,1}, \ldots, w_{l,W_l}) \in \bbR^{W_l}$, and $\smash{\tau_l^{\langle w_l \rangle}}(x) \defeq \tau_l(x, w_l)$.
We say such $z_L$ {\em has $L$ layers, $N$ neurons, and $W$ parameters}.

We next define the {\em activation neurons} $z_l(\,\cdot\,; c) : \bbR^W \to \bbR^{N_l}$ 
and the {\em pre-activation values} $y_l(\,\cdot\,; c):\bbR^W \to \bbR^{N_l}$ at layer $l\in[L]$, as we defined $z_L$ above:
\begin{align*}
  z_l(w;c) & \defeq (\sigma_l \circ \tau_l^{\langle w_l\rangle} \circ \cdots \circ \sigma_1 \circ \tau_1^{\langle w_1\rangle})(c),
  \\
  y_l(w;c) &\defeq \tau_l^{\langle w_l \rangle }(z_{l-1}(w; c)),
\end{align*}
where $z_0(w; c)\defeq c$. 
Since the input data $c$ is fixed while we compute the derivative of $z_L$ with respect to $w$ (e.g., in order to train $z_L$), 
we often omit $c$ and simply write $z_l(w)$ and $y_l(w)$ to denote $z_l(w;c)$ and $y_l(w;c)$, respectively.

For the set of all indices of neurons
\[ \Idx \defeq \{ (l, i) \mid l \in [L], i \in [N_l]\} \]
and for each $(l,i) \in \Idx$,
we use $y_{l,i}, z_{l,i}: \bbR^{W} \to \bbR$ and $\tau_{l,i} : \bbR^{N_{l-1}} \times \bbR^{W_l} \to \bbR$
to denote the functions that take only the {$i$-th} output component of $y_{l}$, $z_{l}$, and $\tau_{l}$, respectively.
Note that we defined $\sigma_{l,i}$ above in a slightly different way:
its domain is not $\bbR^{N_l}$ (i.e., the domain of $\sigma_l$) but $\bbR$.

Finally, we introduce the notion of {piecewise-analytic}%
\footnote{It is inspired by the notion of PAP in \citet{LeeYRY20}.}
to consider possibly non-differentiable activation functions.

\begin{definition*}
  \label{def:pafun}
  A function $f : \bbR \to \bbR$ is {\em piecewise-analytic}
  if there exist $n \in \bbN$, a partition $\{A_i\}_{i \in [n]}$ of $\bbR$ consisting of non-empty intervals,
  and analytic functions $\{f_i : \bbR \to \bbR \}_{i \in [n]}$ such that
  $f = f_i$ on $A_i$ for all $i \in [n]$.
\end{definition*}
\begin{assumptionnon*}
  $\sigma_{l,i}$ is piecewise-analytic for all $(l,i)\in\Idx$.
\end{assumptionnon*}

The class of piecewise-analytic functions includes not only all analytic functions
but also many non-differentiable functions widely used in neural networks such as ReLU, LeakyReLU, and HardSigmoid.
Hence, our definition of neural networks includes a rich class of practical networks: 
$\tau_l$ can be any analytic function (e.g., a fully-connected, convolution, or normalization layer),
and $\sigma_l$ can be any pointwise continuous and piecewise-analytic function (e.g., ReLU, LeakyReLU, or HardSigmoid).

In practice, we often apply AD to the composition of a neural network $z_L$ and a loss function $\ell$
(e.g., Softmax followed by CrossEntropy), to compute the derivative of the loss value of $z_L$ with respect to its parameters.
We emphasize that all of our results \edit{except for lower bounds (i.e., \Cref{thm:ndf-lbound-bias,thm:ndf-lbound-nobias,thm:inc-lbound-nobias})} continue to hold
even if we replace $z_L$ \edit{in their conclusions} by $\ell\circ z_L$ for any analytic $\ell : \bbR^{N_L} \to \bbR^m$. 
For simplicity, however, we state our results only for $z_L$ and not for $\ell \circ z_L$.

\subsection{Automatic Differentiation}
\label{sec:ad}

\edit{%
  Given a program that represents a neural network $z_L$ as in \Cref{eq:nn-def}, AD essentially computes the function
  \[\ADF{z_L} : \bbR^W \to \bbR^{N_L \times W}\]
  by applying the chain rule of differentiation to \Cref{eq:nn-def}.
  That is, $\ADF{z_L}$ is defined as the product of $\smash{\adf{\tau_{l,i}}}$ and $\smash{\adf{\sigma_{l,i}}}$ for $(l,i) \in \Idx$,
  where $\smash{\adf{\tau_{l,i}}} : \smash{\bbR^{N_{l-1}}} \times \smash{\bbR^{W_l}} \to \smash{\bbR^{1 \times (N_{l-1}+W_l)}}$ and
  $\smash{\adf{\sigma_{l,i}}} : \smash{\bbR^{N_{l}}} \to \smash{\bbR^{1 \times N_l}}$
  denote the ``derivatives'' of $\tau_{l,i}$ and $\sigma_{l,i}$ that AD uses in its computation
  (see Appendix~\ref{sec:pf-ad} for more details).
  Here $\ADF{z_L}$, $\adf{\tau_{l,i}}$, and $\adf{\sigma_{l,i}}$ can be different
  from the standard derivatives $\DF{z_L}$, $\DF{\tau_{l,i}}$, and $\DF{\sigma_{l,i}}$,
  partly because the former never return $\bot$ even at non-differentiable points while the latter always return $\bot$ at those points.
  We note that $\ADF{z_L}$ expresses what practical AD systems (e.g., TensorFlow, PyTorch) essentially compute in {\em both} forward-mode and reverse-mode.%
}%

\edit{By definition,} the output $\adf{z_L}$ of AD
depends on the choice of $\adf{\tau_{l,i}}$ and $\adf{\sigma_{l,i}}$. 
To focus on the standard choices made by practical AD systems, 
we introduce the notion of an {extended derivative}.

\begin{definition*}
  \label{def:extderiv}
  A function $g:\bbR^n \to \bbR^{m \times n}$ is an {\em extended derivative} of $f : \bbR^n \to \bbR^m$ if
  for all $x \in \bbR^n$ with $\DF{f}(x) \neq \bot$, it holds that $g(x) = \DF{f}(x)$.
\end{definition*}
\begin{assumptionnon*}
  $\adf{f}$ is an extended derivative of $f$ for all $f \in \{ \tau_{l,i}, \sigma_{l,i} \mid (l,i) \in \Idx \}$.
\end{assumptionnon*}

We note that a differentiable function $f$ has a unique extended derivative which is the standard derivative $\DF{f}$ of $f$.
In contrast, a non-differentiable function $f$ has (uncountably) many extended derivatives:
e.g., $\indc{(0, \infty)} + c \cdot \indc{\{0\}}$ is an extended derivative of $\relu$ for all $c \in \bbR$,
\edit{where $\indc{A}$ denotes the indicator function of a set $A$}.

Among many extended derivatives, some of them are used more frequently in practice,
which we characterize as {consistency}.

\begin{definition*}
  \label{def:consistent}
  For $f : \bbR^n \to \bbR^m$, an extended derivative $g$ of $f$ is {\em consistent} if
  for all $x \in \bbR^n$ with $\DF{f}(x) = \bot$, it holds that $g(x) = \lim_{k \to \infty} \DF{f}(x_k)$ for some $x_k \to x$.%
  \footnote{\edit{%
    Any consistent extended derivative of $f$ is an element of the so-called Bouligand subdifferential of $f$ \cite{CuiP21}.
    But the converse does not hold in general.%
  }}
\end{definition*}

For instance, $\indc{(0, \infty)}$ and  $\indc{[0, \infty)}$ are consistent extended derivatives of $\relu$
but $\indc{(0, \infty)} + c \cdot \indc{\{0\}}$ is not for all $c \in \bbR \setminus \{0, 1\}$;
among them, $\adf{\relu} = \indc{(0, \infty)}$ is typically used by popular AD systems (e.g., TensorFlow and PyTorch).
Although $\adf{f}$ is usually consistent in practice,
we do not assume it by default (and explicitly assume it only when necessary)
to make our results as general as possible,
and to study whether the values of extended derivatives at non-differentiable points matter to AD.

\subsection{Incorrect and Non-Differentiable Sets}
\label{sec:machrep-params}

In practice, the parameters of a neural network cannot be arbitrary real numbers (as machines cannot represent them),
but can only be machine-representable numbers $\bbM \subseteq \bbR$,
where $\bbM$ is often chosen as the set of all $32$-bit floating-point numbers. 
To this end, we consider
\[\Omega \defeq \bbM^W \subseteq \bbR^W,\]
the set of parameters that a neural network $z_L : \bbR^W \to \bbR^{N_L}$ can take in practice.
We assume that $\bbM$ is an arbitrary finite subset of $\bbR$ throughout the paper;
e.g., it can be the set of $n$-bit floating-point (or fixed-point) numbers for any $n \in \bbN$.

To better understand the correctness of AD, we study the following two disjoint subsets of $\Omega$
on which \edit{AD can return an incorrect output}.

\begin{definition*}
  \label{def:ndf-inc}
  For a neural network $z_L$, define the {\em incorrect set} and the {\em non-differentiable set} of $z_L$ as
  \setlength{\belowdisplayskip}{0.5\abovedisplayskip}
  \begin{align*}
    \!\!\!\!\!\!\!\!\!\!\!\!\!\!\!\!\!\!
    \incM{z_L} &\defeq \rlap{$\{w \in \Omega \mid \DF{z_L}(w) \,{\neq}\, \bot,\, \ADF{z_L}(w) \,{\neq}\, \DF{z_L}(w) \},$}
    \\
    \!\!\!\!\!\!\!\!\!\!\!\!\!\!\!\!\!\!
    \ndfM{z_L} &\defeq \{w \in \Omega \mid \DF{z_L}(w) \,{=}\, \bot\}.
    \hspace{60pt} 
  \end{align*}
\end{definition*}

These two sets correspond
to the two cases when AD \edit{can be} incorrect:
on the incorrect set $\incM{z_L}$, $z_L$ is differentiable but AD does not compute its standard derivative;
on the non-differentiable set $\ndfM{z_L}$, $z_L$ is non-differentiable
\edit{and AD may not compute a generalized notion of derivative (e.g., Clarke subdifferential)}.
Here $\ndfM{z_L} \subseteq \Omega$ is different from $\ndf{f} \subseteq \bbR$,
which was defined in \Cref{sec:prelim} for $f : \bbR \to \bbR$.%

\section{\mbox{Correctness of Automatic Differentiation for} Neural Networks with Bias Parameters}
\label{sec:bias}

Our main objective is to understand the incorrect and non-differentiable sets.
In particular, we focus on neural networks with bias parameters (defined below) in this section
and consider more general neural networks in \Cref{sec:nobias}.
For the former class of neural networks,
we characterize the incorrect and non-differentiable sets
in \cref{sec:bias-incorrect} and \cref{sec:bias-nondiff},
and establish a connection between AD and Clarke subderivatives (a generalized notion of derivative)
in \cref{sec:bias-clarke}.

We start by defining neural networks with {bias parameters}.

\begin{definition*}
  \label{def:has-bias}
  A pre-activation function $\tau_l: \smash{\bbR^{N_{l-1}}} \times \smash{\bbR^{W_l}} \to \smash{\bbR^{N_l}}$
  of a neural network {\em has bias parameters}
  \edit{%
    if $W_l \geq N_l$ and there exist
    $\smash{f_{1}}, \allowbreak \ldots, \allowbreak \smash{f_{N_l}} : \bbR^{N_{l-1}} \times \bbR^{W_l - N_l} \to \bbR$ such that 
    \[
    \tau_{l,i}(x, (u, v)) = f_i(x, u) + v_i
    \]
    for all $i \in [N_l]$ and $(x, u, v) \in \bbR^{N_{l-1}} \times \bbR^{W_l - N_l} \times \bbR^{N_l}$.
  }%
  Here $v_i$ is called the {\em bias parameter of $\tau_{l,i}$}.
  A neural network $z_L$ {\em has bias parameters} if $\tau_l$ has bias parameters for all $l \in [L]$.
  \end{definition*}

Many popular pre-activation functions  are typically implemented with bias parameters. 
\edit{%
  For example, fully-connected layers, attention layers (e.g., MultiheadAttention), and some normalization layers (e.g., LayerNorm) do so.
  Yet not all pre-activation functions have bias parameters in practice.
  For instance, convolutional layers and other normalization layers (e.g., BatchNorm) usually do not satisfy \Cref{def:has-bias}:
  they do contain some bias terms, but each of these terms is used to compute multiple output values (instead of a single output value as in our definition).%
}

\subsection{Characterization of the Incorrect Set} 
\label{sec:bias-incorrect}

We first show that the incorrect set of a neural network is {always empty} if the network has bias parameters,
i.e., AD computes the standard derivative wherever the network is differentiable.

\begin{theorem*}
  \label{thm:inc-zero-bias}
  If a neural network $z_L$ has bias parameters, then for all $w \in \bbR^W$ at which $z_L$ is differentiable,
  \begin{align}
    \label{eq:inc-zero-bias}
    \ADF{z_L}(w) = \DF{z_L}(w).
  \end{align}
  This implies that $|\incM{z_L}| = 0$.
\end{theorem*}

It should be emphasized that \cref{eq:inc-zero-bias} is not only for machine-representable parameters,
but also for any {real-valued} parameters.
Compared to existing results, this result is surprising.
For instance, \citet{LeeYRY20,BolteP20b} show that
the incorrect set over $\bbR^n$ (not over $\bbM^n$) has Lebesgue measure zero for some classes of programs,
but they do not give any results on whether the set can be empty.
In contrast, \Cref{thm:inc-zero-bias} states that 
the incorrect set over $\bbR^n$ is {empty} for a smaller, yet still large class of programs, i.e., neural networks with bias parameters.

In \cref{thm:inc-zero-bias}, the condition that $z_L$ has bias parameters plays a crucial role.
Namely, \Cref{thm:inc-zero-bias} does not hold if this condition is dropped.
For instance, consider a neural network $z_L : \bbR \to \bbR$ that is essentially the same as $f : \bbR \to \bbR$ with
\( f(w) = \relu(w) - \relu(-w) \) (which we discussed in~\Cref{sec:intro}). 
Then, $z_L$ does not have bias parameters, and $\incM{z_L}$ is non-empty 
if $\adf{\relu} = \smash{\indc{(0,\infty)}}$ is used.

The proof of \Cref{thm:inc-zero-bias} consists of the following two arguments:
for all $w \in \bbR^W$ with $\DF{z_L}(w) \neq \bot$, 
\begin{itemize}
\item[(i)] if $y_{l,i}(w) \in \ndf{\sigma_{l,i}}$, then $\partial z_L/\partial z_{l,i} = \vec{0}$ at $w$, and
\item[(ii)] if (i) holds, then $\ADF{z_L}(w) = \DF{z_L}(w)$. 
\end{itemize}
That is, (i) if a pre-activation value $y_{l,i}$ touches a non-differ\-entiable point of its activation function $\sigma_{l,i}$,
then the derivative of $z_L$ with respect to $z_{l,i}$ should always be zero;
and (ii) \Cref{thm:inc-zero-bias} follows from (i).
We point out that the proof of (i) relies heavily on the bias parameter condition.
For more details, see Appendix~\ref{sec:pf-inc}.

\vspace{-0.01in}
\subsection{Characterization of the Non-Differentiable Set} 
\label{sec:bias-nondiff}

We next show that if a neural network has bias parameters, then the density of the non-differentiable set in $\Omega$ is {bounded by $n/|\bbM|$},
where $n$ is 
the total number of non-differentiable points in activation functions.

\begin{theorem*}
  \label{thm:ndf-inc-ubound-bias}
  If a neural network $z_L$ has bias parameters, 
  \begin{align*}
    \frac{|\ndfM{z_L}|}{|\Omega|}
    \leq \frac{1}{|\bbM|} {\sum_{(l,i) \in \Idx}} | \ndf{\sigma_{l,i}} |
  \end{align*}
  where $\ndf{f}$ is the set of non-differentiable points of $f$.
\end{theorem*}

In many practical settings, the bound in \cref{thm:ndf-inc-ubound-bias} is often small, especially under high-precision parameters.
For example, $\bbM$ is frequently chosen as the set of $32$-bit floating-point numbers so $|\bbM|\approx 2^{32}$, 
while $|\Idx|$ (the number of neurons) is often smaller than $2^{32}$
and $|\ndf{\sigma_{l,i}}|$ is typically small (e.g., $0$ for differentiable $\sigma_{l,i}$, $1$ for $\relu$, and $2$ for $\mathrm{HardSigmoid}$).
This implies that in practice, the non-differentiable set often has a low density in $\Omega$.
We remark, however, that the bound in \cref{thm:ndf-inc-ubound-bias} can grow large
in low-precision settings (e.g., when parameters are represented by $\le 16$-bit numbers).

Although the bound in \Cref{thm:ndf-inc-ubound-bias} can be large in some cases (e.g., when $|\bbM|$ is small),
we prove that the bound is in general {tight} up to a constant multiplicative factor.

\begin{theorem*}
  \label{thm:ndf-lbound-bias}
  For any $\bbM \subseteq \bbR$ and $n, \alpha \in \bbN$ with $1 \leq |\bbM| < \infty$, $n \geq 2$, and $\alpha \leq |\bbM|/(n-1)$,
  there is a neural network $z_L : \bbR^W \to \bbR$ with bias parameters that satisfies
  \begin{align*}
    \frac{|\ndfM{z_L}|}{|\Omega|}
    \geq \frac{1}{2} \cdot \frac{1}{|\bbM|} \sum_{(l,i) \in \Idx} | \ndf{\sigma_{l,i}} |
  \end{align*}
  and the following: $z_L$ has $n+1$ neurons and $|\ndf{\sigma_{1,i}}| = \alpha$ for all $i \in [N_1]$.
\end{theorem*}

In \Cref{thm:ndf-lbound-bias}, the condition $\alpha \leq |\bbM|/(n-1)$ is for achieving the constant $1/2$ in the bound.
A similar bound can be derived for a larger $\alpha$ (i.e., $\alpha>|\bbM|/(n-1)$) but with a constant smaller than $1/2$.

\Cref{thm:ndf-inc-ubound-bias,thm:ndf-lbound-bias} describe how large the non-differ\-entiable set $\ndfM{z_L}$ can be,
but give no clue about exactly which parameters constitute this set. 
To better understand this, 
we present an easily verifiable {necessary and sufficient} condition for characterizing $\ndfM{z_L}$.

\begin{theorem*}
  \label{thm:cor-deriv-bias}
  If a neural network $z_L$ has bias parameters, then the following are equivalent for all $w \in \bbR^W$.
  \begin{itemize}
  \item $z_L$ is non-differentiable at $w$.
  \item $y_{l,i}(w) \in \ndf{\sigma_{l,i}}$
    and $\smash{\APDF{z_L}} / \partial z_{l,i} \neq \smash{\vec{0}}$ at $w$ for some $(l,i)\in\Idx$.
  \end{itemize}
\end{theorem*}

Here $\smash{\APDF{z_L}} / \partial z_{l,i}$ denotes the partial derivative of $z_L$ with respect to $z_{l,i}$ 
that {reverse-mode} AD (e.g., backpropagation) computes as a byproduct of computing $\ADF{z_L}$
(see Appendix~\ref{sec:pf-apd} for more details). 
Hence, \Cref{thm:cor-deriv-bias} implies that we can efficiently%
\footnote{in $\mcO(N_L T)$ time for a neural network $z_L \,{:}\, \bbR^W \,{\to}\, \bbR^{N_L}$
where $T$ is the time to evaluate $z_L(w)$, because reverse-mode AD takes $\mcO(N_L T)$ time to compute $\ADF{z_L}(w)$.}
decide whether a neural network with bias parameters is \edit{non-differentiable} at a (real-valued) parameter or not.
This result is surprising given a recent, relevant result that 
deciding such \edit{non-differentiability} 
is NP-hard in general \cite{BolteBPP23}.

We now sketch the proof of \Cref{thm:ndf-inc-ubound-bias}, 
to explain how we obtain the bound in the theorem and where we use the bias parameter condition. 
First, we prove that if $y_{l,i}(w)$ does not touch any non-differentiable point of $\sigma_{l,i}$ for all $(l,i) \in \Idx$,
then $z_L$ is differentiable at $w$.
In other words,
\begin{align}
  \label{eq:bias-nondiff-proof-sketch-1}
  \ndfM{z_L} \subseteq \!\!
  \bigcup_{(l,i) \in \Idx}\,
  \bigcup_{c \in \ndf{\sigma_{l,i}}} \!\!
  \{ w \in \Omega \,|\; y_{l,i}(w) = c \}.
\end{align}
Second, we prove that for all $(l,i) \in \Idx$ and $c \in \bbR$,
\begin{align}
  \label{eq:bias-nondiff-proof-sketch-2}
  {\big| \{ w \in \Omega \,|\; y_{l,i}(w) = c  \} \big|} \leq {|\bbM|^{W-1}}.
\end{align}
This inequality is invalid in general, but is valid when $\tau_{l}$ has bias parameters.
If the parameter $w$ has a value $v = (v_1, \ldots, v_W)$ and its $j$-th entry $v_{j}$ corresponds to the bias parameter of $\tau_{l,i}$,
then $y_{l,i}(v) = f(v_{-j}) + v_{j}$ for some function $f$.
Hence, for any $v_{-j} \in \bbM^{W-1}$, there is at most one $v_{j} \in \bbM$ achieving $y_{l,i}(v) = c$,
and this implies the above inequality.
Finally, we prove that \cref{thm:ndf-inc-ubound-bias} follows from the above two results.
The full proofs of \Cref{thm:ndf-inc-ubound-bias,thm:ndf-lbound-bias,thm:cor-deriv-bias} are presented
in Appendices~\ref{sec:pf-ndfinc}, \ref{sec:pf-lowerbd}, and \ref{sec:pf-corderiv}, respectively.

\subsection{Connection to Clarke Subderivatives}
\label{sec:bias-clarke}

We have so far observed that with bias parameters, the incorrect set is always empty
but the non-differentiable set may not be. 
A natural question is then:
what does AD compute on the non-differentiable set?
We answer this question by showing that AD computes a Clarke subderivative%
\footnote{The {\em Clarke subdifferential} of $f : \bbR^n \!\to \bbR^m$ at $x \in \bbR^n$ refers to the convex hull of
$\{ \smash{\lim_{n \to \infty}} \DF{f}(x_n) \mid x_n \to x,\, \allowbreak \DF{f}(x_n) \neq \bot \} \subseteq \bbR^{m \times n}$,
and an element of the Clarke subdifferential is called a {\em Clarke subderivative} \cite{Clarke90,KakadeL18}.} 
{everywhere} (including on the non-differentiable set), if it uses {\em consistent} extended derivatives for activation functions.

\begin{theorem*}
  \label{thm:clarke-subdiff-bias}
  If a neural network $z_L$ has bias parameters and $\adf{\sigma_{l,i}}$ is consistent for all $(l,i) \in \Idx$, then for all $w \in \bbR^W$, 
  \begin{align*}
    \ADF{z_L}(w) =
    \begin{cases}
      \DF{z_L}(w) & \text{if $\DF{z_L}(w) \neq \bot$}
      \\
      \begin{array}{@{}l@{}}
        \lim_{n \to \infty} \DF{z_L}(w'_n)
        \\[-2pt]
        \;\;\;\text{for some $w'_n \to w$}
      \end{array}
      & \text{if $\DF{z_L}(w) = \bot$}.
    \end{cases}
  \end{align*}
  This implies that $\ADF{z_L}$ is a Clarke subderivative of~$z_L$.
\end{theorem*}

\Cref{thm:clarke-subdiff-bias} is not only a new result about AD,
but also gives a positive answer to a long-standing open question about Clarke subgradients \cite{Clarke75,KakadeL18,BolteBPP23}:
are there a sufficiently large class $\mcF$ of {\em scalar} functions
and a deterministic algorithm $\mcA$ that computes a {\em Clarke subgradient} (i.e., subderivative)
of $f \in \mcF$ at $x \in \bbR^n$ efficiently
\edit{(i.e., in time $\mcO(T)$ that is independent of $n$, where $T$ is time to evaluate $f(x)$)}?
In other words, is there a so-called ``Cheap Subgradient Principle''?
For instance, \citet{KakadeL18} propose an efficient algorithm $\mcA'$ (for some $\mcF'$)
but $\mcA'$ is not deterministic,
whereas \citet{BartonKSW18,KhanB15} propose deterministic algorithms $\mcA''$ (for some $\mcF''$)
but $\mcA''$ are not efficient.
In contrast, \Cref{thm:clarke-subdiff-bias} implies that for 
neural networks with bias parameters, 
a Clarke subgradient at any (real-valued) parameter can be computed deterministically and efficiently,
even by the vanilla reverse-mode AD.
In this sense, we provide a new understanding on the computational aspects of Clarke subgradients.

We note that \Cref{thm:clarke-subdiff-bias} no longer holds without any of its conditions:
having bias parameters and using consistent extended derivatives.
One can confirm this using the following examples:
$z_L(w) = \relu(w) - \relu(-w)$ with $\adf{\relu} = \smash{\indc{(0, \infty)}}$
(in which $z_L$ does not have bias parameters as observed in \Cref{sec:bias-incorrect}),
and $\wh{z}_L(w) = \relu(w)$ with $\adf{\relu} = \smash{\indc{(0, \infty)}} + c \cdot \smash{\indc{\{0\}}}$ for any $c \in \bbR \setminus [0,1]$
(in which $\adf{\relu}$ is not consistent).
For the proof of \Cref{thm:clarke-subdiff-bias}, see Appendix~\ref{sec:pf-subdiff}.
\section{\mbox{Correctness of Automatic Differentiation for} Neural Networks without Bias Parameters}
\label{sec:nobias}

In this section, we investigate the correctness of AD for neural networks that may or may not have bias parameters.
For such general networks, 
however, considering only the properties of activation functions such as $\ndf{\sigma_{l,i}}$ (as we did in \Cref{sec:bias}) 
is insufficient to derive non-trivial bounds on the size of the incorrect and non-differentiable sets, as long as general pre-activation functions are used.

To illustrate this, consider neural networks $z_L, \smash{\wh{z}_L} : \bbR \to \bbR$ 
that are essentially the same as $f, \smash{\wh{f}}: \bbR \to \bbR$
with $f(w) = \relu(h(w)) - \relu(-h(w))$ and $\smash{\wh{f}}(w) = \relu(h(w))$,
where $h : \bbR \to \bbR$ is some analytic pre-activation function satisfying $h(x) = 0$ and $\DF{h}(x) = 1$ for all $x \in \bbM$.
Suppose that $\adf{\relu} = \indc{(0,\infty)}$.
Then, we have $\incM{z_L} = \ndfM{\smash{\wh{z}_L}} = \Omega$
even though $z_L$ and $\smash{\wh{z}_L}$ have only $\leq 2$ non-differentiable points in their activation functions.
The main culprit of having such large $\incM{z_L}$ and $\ndfM{\smash{\wh{z}_L}}$,
even with a tiny number of non-differentiable points in activation functions,
is that $z_L$ and $\smash{\wh{z}_L}$ use the unrealistic pre-activation function $h$ which does not have bias parameters.

To exclude such extreme cases and focus on realistic neural networks,
we will often consider {\em well-structured biaffine} pre-activation functions when they do not have bias parameters.

\begin{definition*}
  \label{def:biaffine}
  A pre-activation function $\tau_l : \bbR^{N_{l-1}} \times \bbR^{W_l} \to \bbR^{N_l}$ is {\em well-structured biaffine}
  if there are $M_i \in \bbR^{{N_{l-1}} \times W_l}$ and $c_i \in \bbR$ for all $i \in [N_l]$ such that
  \[
  \tau_{l,i}(x,u) = x^\trsp M_i u + c_i
  \]
  and each column of $M_i$ has at most one non-zero entry.
\end{definition*}

Any fully-connected or convolution layers are well-structured biaffine when they do not have bias parameters.
Thus, a large class of neural networks is still under our consideration even after we impose the above restriction.
\edit{Yet some pre-activation functions (e.g., normalization and attention layers) are not well-structured biaffine whether or not they have bias parameters.}

We now present our results for neural networks possibly without bias parameters,
extending Theorems~\ref{thm:inc-zero-bias}--\ref{thm:clarke-subdiff-bias}.

\subsection{Bounds for Non-Differentiable and Incorrect Sets}

We first bound the density of the non-differentiable and incorrect sets in $\Omega$, extending \Cref{thm:ndf-inc-ubound-bias}. 

\begin{theorem*}
  \label{thm:ndf-inc-ubound-nobias}
  If a pre-activation function $\tau_l$ has bias parameters or is well-structured biaffine for all $l\in[L]$, then 
  \begin{align*}
    & \frac{|\ndfM{z_L} \cup \incM{z_L}|}{|\Omega|}
    \\
    & \qquad
    \leq \frac{1}{|\bbM|} {\sum_{(l,i) \in \Idx}}
    \Big| \ndf{\sigma_{l,i}} \cup \big(\bdz{\sigma_{l,i}} \cap S_{l+1}\big) \Big|,
  \end{align*}
  where $\bdz{f}$ is the boundary of $f$'s zero set (see \Cref{sec:prelim}), and
  \begin{align*}
    S_l \defeq
    \begin{cases}
      \emptyset~~ & \text{if $l > L$ or $\tau_l$ has bias parameters}
      \\ \bbR~~ & \text{otherwise}.
    \end{cases}
  \end{align*}
\end{theorem*}

We note that if $z_L$ has bias parameters, \Cref{thm:ndf-inc-ubound-nobias} reduces to \cref{thm:ndf-inc-ubound-bias}
since $\incM{z_L} = \emptyset$ (by \Cref{thm:inc-zero-bias}) and $S_l = \emptyset$ for all $l$ (by its definition) in such a case.

As in \cref{thm:ndf-inc-ubound-bias}, the bound in \cref{thm:ndf-inc-ubound-nobias} is often small 
for neural networks that use practical activation functions, 
since $|\ndf{\sigma_{l,i}} \cup \bdz{\sigma_{l,i}}|$ is typically small for those activation functions
(e.g., $1$ for $\relu$ and $2$ for $\mathrm{HardSigmoid}$).

We now show that the additional term $\bdz{\sigma_{l,i}}$ in \cref{thm:ndf-inc-ubound-nobias} is indeed necessary
by providing a matching lower bound up to a constant factor. 

\begin{theorem*}
  \label{thm:ndf-lbound-nobias}
  For any $\bbM \subseteq \bbR$ and $n, \alpha \in \bbN$ with $1 \leq |\bbM| < \infty$, $n \geq 4$, and $\alpha \leq |\bbM|/(n-1)$,
  there is a neural network $z_L : \bbR^W \to \bbR$ that satisfies
  \begin{align*}
    &\frac{|\ndfM{z_L}|}{|\Omega|}
    \geq \frac{1}{9} \cdot \frac{1}{|\bbM|} {\sum_{(l,i) \in \Idx}}
    \Big| \ndf{\sigma_{l,i}} \cup \bdz{\sigma_{l,i}} \Big|
  \end{align*}
  and the following:
  (i) $\tau_l$ is well-structured biaffine without bias parameters for all $l < L$, and has bias parameters for $l = L$; 
  (ii) $z_L$ has $n + 1$ neurons;
  and (iii) $|\ndf{\sigma_{1,i}}| = \alpha$ and $|\bdz{\sigma_{1,i}}| = 0$ for all $i$. 
  We obtain the same result for (i), (ii'), and (iii'): 
  (ii') $z_L$ has $2n + 1$ neurons;
  and (iii') $|\ndf{\sigma_{1,i}}| = 0$ and $|\bdz{\sigma_{1,i}}| = \alpha$ for all~$i$. 
\end{theorem*}

We next give an intuition for why the zero set of $\sigma_{l,i}$ (from which the additional term $\bdz{\sigma_{l,i}}$ is defined)
appears in \Cref{thm:ndf-inc-ubound-nobias}, by examining its proof. 
The proof consists of two main parts that extend Eqs.~\eqref{eq:bias-nondiff-proof-sketch-1} and \eqref{eq:bias-nondiff-proof-sketch-2}
from the proof sketch of \Cref{thm:ndf-inc-ubound-bias}: 
we first show
\begin{align*}
  \ndfM{z_L} \cup \incM{z_L} \subseteq
  \hspace{-15pt}
  \bigcup_{ (l,i) \in \Idx, \, c \in \ndf{\sigma_{l,i}} }
  \hspace{-15pt}
  \{ w \in \Omega \mid y_{l,i}(w) = c \}
  \vphantom{\bigcup_{(l,i) \in \Idx}}
\end{align*}
and then find a reasonable bound on $|\Lambda_{l,i,c}|$ for 
$\Lambda_{l,i,c} \defeq \{ w \in \Omega \mid y_{l,i}(w) = c \}$, 
the set of parameters on which the pre-activation value $y_{l,i}$ touches the non-differentiable point $c$ of $\sigma_{l,i}$.
Among the two parts, the zero set of $\sigma_{l,i}$ arises from the second part (i.e., bounding $|\Lambda_{l,i,c}|$),
especially when $\tau_l$ does not have bias parameters and is well-structured biaffine.
For simplicity, assume that $\tau_l$ is a fully-connected layer with constant biases,
i.e., $y_{l,i}(w) = \smash{\sum_{j \in [N_{l-1}]}}\, z_{l-1,j}(w) \cdot w_{j+a} + b$ for some constants $a,b$.
Based on this, we decompose $\Lambda_{l,i,c}$ into $\Lambda' \cup \Lambda''$:
\begin{align*}
  \Lambda' &\defeq \{ w \in \Omega \mid y_{l,i}(w) = c, \, z_{l-1,j}(w) \neq 0 \text{ for some }j \},
  \\
  \Lambda'' &\defeq \{ w \in \Omega \mid y_{l,i}(w) = c, \, z_{l-1,j}(w) = 0 \text{ for all }j \}.
\end{align*}
Then, we can show $|\Lambda'| \leq |\bbM|^{W-1}$ as in \Cref{eq:bias-nondiff-proof-sketch-2},
since $w_{j+a}$ acts like a bias parameter of $y_{l,i}$ for any $j$ with $z_{l-1,j}(w) \neq 0$.
To bound $|\Lambda''|$, however, we cannot apply a similar approach due to the lack of $j$ with $z_{l-1,j}(w) \neq 0$. 
Instead, we directly count the number of parameters $w \in \Omega$ achieving $z_{l-1,j}(w) = 0$ for all~$j$
(i.e., $\sigma_{l-1,j}(y_{l-1,j}(w)) = 0$ for all~$j$),
and this requires the zero set of $\sigma_{l-1,j}$.
For the full proofs of \Cref{thm:ndf-inc-ubound-nobias,thm:ndf-lbound-nobias},
see Appendices~\ref{sec:pf-ndfinc} and \ref{sec:pf-lowerbd}.

\subsection{Bounds for the Incorrect Set}

For the non-differentiable set, \Cref{thm:ndf-inc-ubound-nobias,thm:ndf-lbound-nobias} provide tight bounds on its size.
For the incorrect set, it turns out that we can further improve the upper bound in \Cref{thm:ndf-inc-ubound-nobias}
and get a similar lower bound to \Cref{thm:ndf-lbound-nobias}.

\begingroup
\addtolength{\abovedisplayskip}{-0.5pt}
\setlength{\belowdisplayskip}{\abovedisplayskip}
\begin{theorem*}
  \label{thm:inc-ubound-nobias}
  If a pre-activation function $\tau_l$ has bias parameters or is well-structured biaffine for all $l\in[L]$, then 
  \begin{align*}
    \frac{|\incM{z_L}|}{|\Omega|}
    &\leq \frac{1}{|\bbM|} \smash{\sum_{(l,i) \in \Idx}}
    \Big| \big(\ndf{\sigma_{l,i}} \cap S_l\big)
    \\[-7pt]
    & \qquad\qquad\qquad\;\;
    \cup \big(\bdz{\sigma_{l,i}} \cap S_{l+1} \big) \Big|,
  \end{align*}
  where $S_l$ is defined as in \Cref{thm:ndf-inc-ubound-nobias}.
\end{theorem*}
\begin{theorem*}
  \label{thm:inc-lbound-nobias}
  For any $\bbM \subseteq \bbR$ and $n, \alpha \in \bbN$ with $1 \leq |\bbM| < \infty$, $n \geq 4$, and $\alpha \leq |\bbM|/(n-1)$,
  there is a neural network $z_L : \bbR^W \to \bbR$ that satisfies
  \begin{align*}
    &\frac{|\incM{z_L}|}{|\Omega|}
    \geq \frac{1}{13} \cdot \frac{1}{|\bbM|}
    \sum_{(l,i) \in \Idx} \Big| \ndf{\sigma_{l,i}} \cup \bdz{\sigma_{l,i}} \Big|
  \end{align*}
  and the following:
  (i) $\tau_l$ is well-structured biaffine without bias parameters for all $l < L$, and has bias parameters for $l = L$; 
  (ii) $z_L$ has $2n + 1$ neurons;
  and (iii) $|\ndf{\sigma_{1,i}}| = \alpha$ and $|\bdz{\sigma_{1,i}}| = 0$ for all $i$. 
  We obtain the same result for (i), (ii'), and (iii'): 
  (ii') $z_L$ has $3n + 1$ neurons;
  and (iii') $|\ndf{\sigma_{1,i}}| = 0$ and $|\bdz{\sigma_{1,i}}| = \alpha$ for all~$i$. 
\end{theorem*}

We note that if $z_L$ has bias parameters, \Cref{thm:inc-ubound-nobias} reduces to $|\incM{z_L}|=0$ as in \cref{thm:inc-zero-bias}
since $S_l = \emptyset$ for all $l$ in the case.
On the other hand, if $z_L$ does not have bias parameters,
then the incorrect set can be non-empty as discussed in \cref{sec:bias-incorrect},
and more importantly, its size can be bounded by \Cref{thm:inc-ubound-nobias}.
To see why the bounds on $|\incM{z_L}|$ depend on both $\ndf{\sigma_{l,i}}$ and $\bdz{\sigma_{l,i}}$,
refer to the proofs of \Cref{thm:inc-ubound-nobias,thm:inc-lbound-nobias} in Appendices \ref{sec:pf-inc} and \ref{sec:pf-lowerbd}.

\vspace{-1.5pt}
\subsection{Sufficient Conditions for Computing \qquad\qquad Standard Derivatives and Clarke Subderivatives}

We extend \Cref{thm:cor-deriv-bias,thm:clarke-subdiff-bias} to general neural networks
without the well-structured biaffinity restriction,
by characterizing two sufficient conditions on parameters
under which AD computes the standard derivative or a Clarke subderivative.

\begin{theorem*}
  \label{thm:cor-deriv-nobias}
  Let $w \in \bbR^W$.
  If $y_{l,i}(w) \notin \ndf{\sigma_{l,i}}$ for all $(l,i) \in \Idx$ such that
  $\tau_l$ does not have bias parameters or $\smash{\APDF{z_L}} / \partial z_{l,i} \neq \smash{\vec{0} }$ at $w$,
  then
  \begin{align*}
    \ADF{z_L}(w) = \DF{z_L}(w) \neq \bot.
    \\[-12pt]
  \end{align*}
\end{theorem*}
\begin{theorem*}
  \label{thm:clarke-subdiff-nobias}
  Let $w \in \bbR^W$.
  Assume that $\adf{\sigma_{l,i}}$ is {consistent} for all $(l,i) \in \Idx$.
  If $y_{l,i}(w) \notin \ncdf{\sigma_{l,i}}$ for all $(l,i) \in \Idx$
  such that $\tau_l$ does not have bias parameters, 
  then
  \begin{align*}
    \ADF{z_L}(w) =
    \begin{cases}
      \DF{z_L}(w) & \text{if $\DF{z_L}(w) \neq \bot$}
      \\
      \begin{array}{@{}l@{}}
        \lim_{n \to \infty} \DF{z_L}(w'_n)
        \\[-2pt]
        \;\;\;\text{for some $w'_n \to w$}
      \end{array}
      & \text{if $\DF{z_L}(w) = \bot$}
    \end{cases}
  \end{align*}
  and so $\ADF{z_L}(w)$ is a Clarke subderivative of $z_L$ at $w$.
  Here $\ncdf{f}$ denotes the set of real numbers at which $f : \bbR \to \bbR$ is not continuously differentiable.
\end{theorem*}
\endgroup

The two sufficient conditions on $w$ given in \Cref{thm:cor-deriv-nobias,thm:clarke-subdiff-nobias}
are simple enough to be checked efficiently in practice;
thus, we can use them to validate whether the output of AD is the standard derivative or a Clarke subderivative.
If $w$ does not satisfy either of the sufficient conditions,
AD may not compute the standard derivative or a Clarke subderivative;
the first example discussed in \cref{sec:bias-clarke} illustrates both cases.
We remark that the sufficient condition in \Cref{thm:clarke-subdiff-nobias} involves $\ncdf{\sigma_{l,i}}$ (not $\ndf{\sigma_{l,i}}$),
since we use continuous differentiability (not differentiability)
in the proof to properly handle the limit of derivatives $\DF{z_L}(w'_n)$. 
For the proofs of \Cref{thm:cor-deriv-nobias,thm:clarke-subdiff-nobias}, see Appendices \ref{sec:pf-corderiv} and \ref{sec:pf-subdiff}.

\section{Related Work}
\label{sec:related}

\edit{%
  The correctness of AD has been extensively studied, especially in the past few years.
  When a program uses only differentiable functions,
  AD is shown to compute its standard derivative at all real-valued inputs
  \cite{AbadiP20, Elliott18, BartheCLG20, HuotSV20, Vakar21, BrunelMP20, KrawiecJKEEF22, SmedingV23, RadulPFJM23}.
  In contrast, when a program uses non-differentiable functions,
  the program itself can be non-differentiable, and AD can return a value different from its standard derivative, at some real-valued inputs.
  Nevertheless, for a large class of programs,
  such inputs are shown to be in a Lebesgue measure-zero subset of the real-valued input domain
  \cite{BolteP20a, BolteP20b, LeeYRY20, HuotLMS22, MazzaP21}.
  All these works consider the case when inputs to AD are {real-valued},
  while our work focuses on the case when the inputs are {machine-representable}.

  The Clarke subdifferential and its connection to AD have been studied for decades.
  Some classes of functions (e.g., subdifferentially regular or strictly differentiable)
  are shown to admit exact chain rules for the Clarke subdifferential
  (e.g., Theorems 2.3.9, 2.3.10, and 2.6.6 of \citet{Clarke90} and Theorem 10.6 of \citet{RockafellarW98}),
  and this implies that AD always computes a Clarke subderivative for a certain class of programs.
  However, this class of programs is restrictive, excluding even simple neural networks (e.g., $(1\,{-}\,\relu(x))^2$) \cite{DavisDKL20}.
  In contrast, our \Cref{thm:clarke-subdiff-bias} shows that
  AD always computes a Clarke subderivative of neural networks with bias parameters. 
  For piecewise differentiable functions, the Clarke sub\-differential can be expressed in terms of the standard derivatives of underlying differentiable functions
  (e.g., Proposition 4.3.1 of \citet{Scholtes12}), but this result is not directly related to~AD.

  A variety of algorithms (other than AD) have been proposed to compute a Clarke subgradient of a scalar program, correctly and efficiently.
  For a large class of programs $f : \bbR^n \to \bbR$ and an input $x \in \bbR^n$,
  the algorithm by \citet{KakadeL18} computes a Clarke subgradient of $f$ at $x$ in time $\mcO(T)$ almost surely,
  while the algorithms by \citet{BartonKSW18,KhanB15} compute the quantity in time $\mcO(nT)$ deterministically,
  where $T$ denotes time to evaluate $f(x)$.
  Our \Cref{thm:clarke-subdiff-bias} provides a relevant result as described above,
  but we point out that our work is about analyzing the correctness of vanilla (forward/reverse-mode) AD,
  not about proposing a new algorithm.

  Recently, \citet{BertoinBGP21} empirically studied how the choice of $\adf{\relu}(0)$ changes the output of AD and the training of neural networks.
  In contrast, our work theoretically studies the correctness of AD.
  Further connections between this and our work are discussed in \Cref{sec:discuss}.
}%

\section{Discussion}
\label{sec:discuss}

\edit{%
  {\bf Connections to \citet{BertoinBGP21}.}
  \citeauthor{BertoinBGP21} empirically studied the {\em bifurcation zone} of a neural network with $\relu$, given an input dataset:
  the set of the network parameters on which the output of AD using $\adf{\relu}(0)=0$ is different from that using $\adf{\relu}(0)=1$ for some input data.
  The bifurcation zone is closely related to the non-differentiable and incorrect sets as follows:
  the bifurcation zone (over machine-representable parameters) is always a subset of the union of the non-differentiable set and two incorrect sets
  (one for $\adf{\relu}(0)=0$ and the other for $\adf{\relu}(0)=1$) over all input data in the given dataset.
  
  For various neural networks (MLP, VGG, ResNet) and datasets (MNIST, CIFAR10, SVHN, ImageNet),
  \citeauthor{BertoinBGP21} estimated the density of the bifurcation zone over 32-bit floating-point parameters
  (i.e., the number of 32-bit parameters in the bifurcation zone over the total number of 32-bit parameters) using Monte Carlo sampling.
  They reported two results among many others: when AD uses 64-bit precision in its computation,
  the estimated density is exactly $0$ in all cases they considered;
  and when AD uses 32- or 16-bit precision, the estimated density is often large and even goes up to $1$.
  The first result is consistent with our results:
  if we use 32-bit parameters, the non-differentiable and incorrect sets would often have small densities in practice.
  Meanwhile, the second result does not contradict our results,
  since our results assume that AD computes its output without any rounding errors.
  Given these observations, it would be an interesting direction to rigorously study the correctness of AD under floating-point operations.

  {\bf Extensions.}
  As mentioned in \Cref{sec:nn}, all our theorems except for those on lower bounds
  (i.e., \Cref{thm:ndf-lbound-bias,thm:ndf-lbound-nobias,thm:inc-lbound-nobias})
  continue to hold even if we replace $z_L$ in their conclusions by $\ell\circ z_L$ for any analytic $\ell : \bbR^{N_L} \to \bbR^m$.
  Among them, \Cref{thm:ndf-inc-ubound-bias,thm:ndf-inc-ubound-nobias} are easily extended to a more general case with multiple input data:
  they remain valid even if we replace $z_L$ in their conclusions by $\ell(z_L(\cdot; x_1), \ldots, z_L(\cdot; x_k))$
  for any $x_1, \ldots, x_k \in \bbR^{N_0}$ and analytic $\ell : \bbR^{N_L} \to \bbR^m$,
  where we need to multiply $k$ to the upper bounds in the theorems.
  The remaining theorems
  (i.e., \Cref{thm:inc-zero-bias,thm:cor-deriv-bias,thm:inc-ubound-nobias,thm:cor-deriv-nobias,thm:clarke-subdiff-bias,thm:clarke-subdiff-nobias}),
  on the other hand, are not easily extended to the case with multiple input data, at least based on our current proofs.
  Studying such extensions could be another interesting future direction.
}%

\edit{\bf Limitations.}
Our results have some limitations. For example, all of our results are for a class of neural networks consisting of alternating analytic pre-activation functions and pointwise continuous activation functions. Hence, if a network contains non-pointwise activation functions (e.g., MaxPool) or a residual connection bypassing a non-analytic activation function (e.g., ReLU), then our results may not be directly applicable. Our results for general neural networks (e.g., \cref{thm:ndf-inc-ubound-nobias,thm:inc-ubound-nobias}) additionally assume pre-activation functions to have bias parameters or to be well-structured biaffine, which does not allow,
\edit{e.g., BatchNorm layers and attention layers without bias parameters}.
Nevertheless, we believe that our results still cover a large class of neural networks, especially compared to prior works studying theoretical aspects of neural networks \cite{lu17,LaurentB18a, JacotHG18,park21,kidger20}.
We believe extending our work to more general neural networks is an interesting direction for future work.

\section{Conclusion}
\label{sec:concl}

In this paper, we \edit{theoretically} study for the first time the correctness of AD for neural networks with machine-representable parameters.
In particular, we provide various theoretical results on the incorrect and non-differentiable sets of a neural network,
as well as closely related questions such as when AD is correct and what it computes. 
Our results have two major practical implications:
AD is correct at most machine-representable parameters when applied to neural networks,
and it is correct more often if more layers of the network have bias parameters.
Furthermore, our theoretical analyses suggest new applications of AD for identifying differentiability and computing Clarke subderivatives, not only for machine-representable parameters but also for any real-valued ones.

\section*{Acknowledgments}

\edit{%
  We thank anonymous reviewers for providing helpful comments.
  WL and AA were supported by the Advanced Simulation and Computing (ASC) program of the US Department of Energy’s National Nuclear Security Administration (NNSA) via the PSAAP-III Center at Stanford, Grant No. DE-NA0002373 and by the Department of Energy's Office of Advanced Scientific Computing Research (ASCR) under contract DE-AC03-76SF00515.
  SP was supported by Institute of Information \& communications Technology Planning \& Evaluation (IITP) grant funded by the Korea government (MSIT) (No. 2019-0-00079, Artificial Intelligence Graduate School Program, Korea University).
}%

\bibliography{reference}
\bibliographystyle{icml2023}

\newpage
\appendix
\onecolumn
\etocdepthtag.toc{mtappendix}
\etocsettagdepth{mtchapter}{none}
\etocsettagdepth{mtappendix}{subsection}
\begingroup
\section*{Contents of Appendix}
\vskip-0.2\baselineskip
\parindent=0mm
\parskip=0.0\baselineskip
\etocsettocstyle
    {} 
    {} 
\tableofcontents
\endgroup

\clearpage
\section{Formal Setup}
\label{sec:pf-setup}

\edit{%
  In the appendix, we use the following notation.
  For $A \subseteq \bbR^n$, $\intr(A)$ and $\bd(A)$ denote the interior and the boundary of $A$.
}%

\subsection{Piecewise-Analytic Functions}
\label{sec:pf-pafunc}

\begin{definition}
  For $A \subseteq \bbR^n$, define $\pbd(A)$ as
  \[ \pbd(A) \defeq A \setminus \intr(A). \]
  We call $\pbd(A)$ the {\em proper boundary of $A$}.
  Note that $\pbd(A) = \bd(A) \cap A$ holds for any $A$.
\end{definition}

\begin{definition*}
  \label{def:pafun-refined}
  A function $f : \bbR \to \bbR$ is {\em piecewise-differentiable} (or {\em piecewise-$C^1$})
  if there exist $n \in \bbN$, a partition $\{A_i\}_{i \in [n]}$ of $\bbR$ consisting of non-empty intervals,
  and differentiable (or $C^1$) functions $\{f_i : \bbR \to \bbR \}_{i \in [n]}$ such that
  $f = f_i$ on $A_i$ for all $i \in [n]$.
  We call such $\{(A_i, f_i)\}_{i \in [n]}$ a {\em piecewise-differentiable (or piecewise-$C^1$) representation of~$f$}.
  Moreover, for an extended derivative $g: \bbR \to \bbR$ of $f$, 
  we say that the representation $\{(A_i, f_i)\}_{i \in [n]}$ {\em defines} $g$ if $g = \DF{f_i}$ on $A_i$ for all $i \in [n]$.
  We define a {\em piecewise-analytic representation of $f$} in a similar way.
\end{definition*}

\begin{lemma}
  \label{lem:partition-bd-pbd}
  Let $\{A_i\}_{i \in S}$ be any partition of $\bbR^n$.
  Then,
  \begin{align}
    \text{$\bigcup_{i \in S} \bd(A_i) = \bigcup_{i \in S} \pbd(A_i)$}.
  \end{align}
\end{lemma}
\begin{proof}
  The direction $\supseteq$ is clear, since $\pbd(X) \subseteq \bd(X)$ for any $X \subseteq \bbR^n$.
  To prove the other direction $\subseteq$, it suffices to show that for any $i\in S$ and $x\in\bd(A_i)$, we have $x \in \pbd(A_j)$ for some $j \in S$.
  Here we assume $x \notin A_i$; if not, choosing $j=i$ completes the proof.
  Let $j \in S$ be the index with $x \in A_j$, where such $j$ always exists since $\{A_i\}_{i \in S}$ is a partition of $\bbR$.
  Then, it suffices to show $x\in\bd(A_j)$, because this and $x \in A_j$ implies $x\in\pbd(A_j)$.
  To prove $x\in\bd(A_j)$, consider any open neighborhood $U \subseteq \bbR^n$ of $x$.
  Then, there is $x' \in U \cap A_i$ (by $x \in \bd(A_i)$ and $x \notin A_i$).
  This implies that $x' \notin U \cap A_j$ (by $A_i \cap A_j = \emptyset$ from $i \neq j$) and $x \in U \cap A_j$ (by $x \in A_j$).
  Hence, we have $x \in \bd(A_j)$ as desired.
\end{proof}

\begin{theorem}
  \label{thm:int-deriv-basic}
  Let $f : \bbR \to \bbR$ be a continuous, piecewise-analytic function,
  and $g : \bbR \to \bbR$ be an extended derivative of $f$.
  Then, the following hold.
  \begin{enumerate}[label=(\roman*)]
  \item 
    There is a piecewise-differentiable representation $\{(A_i, f_i)\}_{i \in [n]}$ of $f$ that defines $g$ and
    satisfies the following: 
    \begin{gather*}
      \text{$\bigcup_{i \in [n]} \bd(A_i) = \bigcup_{i \in [n]} \pbd(A_i) = \ndf{f}$}.
    \end{gather*}
  \item
    If $g$ is consistent,
    there is a piecewise-$C^1$ representation $\{(A_i, \allowbreak f_i)\}_{i \in [n]}$ of $f$ that defines $g$ and
    satisfies the following: 
    \begin{gather*}
      \text{$\bigcup_{i \in [n]} \bd(A_i) = \bigcup_{i \in [n]} \pbd(A_i) = \ncdf{f}$},
      \qquad\quad
      \text{$\intr(A_i) \neq \emptyset$ for all $i \in [n]$},
    \end{gather*}
    where $\ncdf{f} \subseteq \bbR$ denotes the set of real numbers at which $f:\bbR \to \bbR$ is not continuously differentiable.
  \end{enumerate}
\end{theorem}
\begin{proof}
  We prove the two claims as follows.
  Note that by \Cref{lem:partition-bd-pbd}, we do not need to prove the equality between the union of $\bd(A_i)$ and that of $\pbd(A_i)$ in the claims.

  {\bf Claim (i).}
  Let $\{(\wt{A}_i, \wt{f}_i)\}_{i \in [\wt{n}]}$ be a piecewise-analytic representation of $f$ 
  that defines $g$ and satisfies
  \begin{align*}
    (\wt{A}_1, \ldots, \wt{A}_{\wt{n}}) = \Big( (x_0, x_1), \ldots, (x_k, x_{k+1}), \{x_{1}\}, \ldots, \{x_k\} \Big)
  \end{align*}
  for some $-\infty = x_0 < x_1 < \cdots < x_k < x_{k+1} = \infty$.
  Such a representation always exists, because $f$ is piecewise-analytic and $g$ is an extended derivative of $f$.
  Note that $\ndf{f} \subseteq \{x_1, \ldots, x_k\}$ because $f$ is differentiable on $(x_{i-1}, x_{i})$ for all $i \in [k+1]$
  (since $\wt{f}_i$ is analytic and it coincides with $f$ on $\wt{A}_i = (x_{i-1}, x_{i})$).
  We then construct $\{(A_i, f_i)\}_{i \in [n]}$ from $\{(\smash{\wt{A}_i}, f_i)\}_{i \in [\wt{n}]}$,
  by merging all adjacent intervals $\smash{\wt{A}_i}$ (and associated functions $\smash{\wt{f}_i}$)
  into a single interval (and a single function) such that
  the class of the singleton interval in $\{A_i\}$ are the same as $\{ \{x\} \mid x \in \ndf{f} \}$.
  Then,
  \[ \bigcup_{i \in [n]} \pbd(A_i) = \bigcup_{x \in \ndf{f}} \{x\} = \ndf{f} \]
  by construction;
  $f_i$ is differentiable for all $i \in [n]$;
  and $\{(A_i, f_i)\}_{i \in [n]}$ defines $g$ since $g$ is an extended derivative of $f$.
  Hence, $\{(A_i, f_i)\}_{i \in [n]}$ is a piecewise-differentiable representation of $f$
  that defines $g$ and satisfies the equation in the statement.

  {\bf Claim (ii).}
  By a similar argument, there is a piecewise-$C^1$ representation $\{(\wt{A}_i, \wt{f}_i)\}_{i \in [\wt{n}]}$ of $f$ that defines $g$ and satisfies
  \[ \bigcup_{i \in [\wt{n}]} \pbd(\wt{A}_i) = \ncdf{f}. \]
  Note that here we need $\ncdf{f}$ (instead of $\ndf{f}$) in the above equation,
  to obtain a piecewise-$C^1$ (instead of piecewise-differentiable) representation of $f$.
  We then construct $\{(A_i, f_i)\}_{i \in [n]}$ from $\{(\wt{A}_i, \wt{f}_i)\}_{i \in [\wt{n}]}$,
  by merging each singleton interval $\smash{\wt{A}_i}$ (and the associated function $\smash{\wt{f}_i}$)
  with one of the two adjacent intervals (and its associated function) such that
  $\{(A_i, f_i)\}_{i \in [n]}$ defines $g$.
  Such a construction always exists, because $f$ is continuous, $g$ is consistent, and $\wt{f}_i$ is $C^1$ for all $i \in [\wt{n}]$.
  Then, 
  \[ \bigcup_{i \in [n]} \pbd(A_i) = \ncdf{f},
  \qquad\qquad \intr(A_i) \neq \emptyset \qquad \text{for all $i \in [n]$} \]
  by construction; and $f_i$ is $C^1$ for all $i \in [\wt{n}]$ since $f$ is continuous.
  Hence, $\{(A_i, f_i)\}_{i \in [n]}$ is a piecewise-$C^1$ representation of $f$
  that defines $g$ and satisfies the equation given in the statement.
\end{proof}

\subsection{Neural Networks}
\label{sec:pf-nn}

\begin{definition*}
  \label{def:minrep-sigma}
  For each $(l,i) \in \Idx$, let \[ \{(\mcI_{l,i}^k, \sigma_{l,i}^k) \}_{k \in [K_{l,i}]} \]
  be a piecewise-differentiable representation of $\sigma_{l,i} : \bbR \to \bbR$
  that defines $\adf{\sigma_{l,i}}$ (an extended derivative of $\sigma_{l,i}$ defined in \Cref{sec:ad}),
  where $K_{l,i} \in \bbN$, $\smash{\mcI_{l,i}^k} \subseteq \bbR$, and $\smash{\sigma_{l,i}^k} : \bbR \to \bbR$.
  We assume that the representation satisfies the following:
  \begin{gather*}
    \text{$\bigcup_{k \in [K_{l,i}]} \bd(\mcI_{l,i}^k) = \bigcup_{k \in [K_{l,i}]} \pbd(\mcI_{l,i}^k) = \ndf{\sigma_{l,i}}$}.
  \end{gather*}
  Note that such a representation always exists by \Cref{thm:int-deriv-basic}.
\end{definition*}

\begin{definition*}
  \label{def:gamma}
  Define $\Gamma$, the set of indices denoting which piece of each activation function is used, as
  \begin{align*}
    \Gamma &\defeq \{ \gamma : \Idx \to \bbN \;|\; \gamma(l,i) \in [K_{l,i}]  \text{ for all $(l,i) \in \Idx$} \}.
  \end{align*}
\end{definition*}

\begin{definition*}
  \label{def:ryz-gamma}
  Let $\gamma \in \Gamma$ and $l \in [L]$.
  Define $\mcR^\gamma \subseteq \bbR^W$, $y_l^\gamma, z_l^\gamma : \bbR^W \to \bbR^{N_l}$, $\sigma_l^\gamma : \bbR^{N_l} \to \bbR^{N_l}$ as:
  \begin{align*}
    \mcR^\gamma &\defeq  \{  w \in \bbR^W \;|\; y_{l,i}(w) \in \mcI_{l,i}^{\gamma({l,i})} \text{ for all } (l,i) \in \Idx \},
    \\
    y_l^\gamma(w) &\defeq \tau_l\big(z_{l-1}^\gamma(w), \pi_l(w)\big),
    \qquad
    z_l^\gamma(w) \defeq \sigma_{l}^\gamma\big(y_{l}^\gamma(w)\big),
    \\
    \sigma_{l}^\gamma(x) &\defeq \big(\sigma_{l,1}^{\gamma(l,1)}(x_1), \ldots, \sigma_{l,N_l}^{\gamma(l,N_l)}(x_{N_l})\big),
  \end{align*}
  where
  $\pi_l : \bbR^W \to \bbR^{W_l}$ denotes the projection function that extracts $w_l \in \bbR^{W_l}$ from $(w_1, \ldots, w_L) \in \bbR^W$,
  and $z_0^\gamma : \bbR^W \to \bbR^{N_0}$ is defined as $\smash{z_0^\gamma} \defeq z_0$.
\end{definition*}

\begin{lemma}
  \label{lem:r-gamma-partition}
  $\{\mcR^\gamma\}_{\gamma \in \Gamma}$ is a partition of $\bbR^W$.
\end{lemma}
\begin{proof}
  This follows immediately from that $\smash{\{\mcI_{l,i}^k\}_{k \in [K_{l,i}]}}$ is a partition of $\bbR$ for all $(l,i) \in \Idx$
  (since $\smash{\{(\mcI_{l,i}^k, \sigma_{l,i}^k)\}_{k \in [K_{l,i}]}}$ is a representation of $\sigma_{l,i}$).
\end{proof}

\begin{lemma}
  \label{lem:yz-yz-gamma-good}
  For all $l \in [L]$ and $\gamma \in \Gamma$,
  $y_l$ and $z_l$ are continuous, and $\smash{y_l^\gamma}$ and $\smash{z_l^\gamma}$ are differentiable.
\end{lemma}
\begin{proof}
  The continuity of $y_l$ and $z_l$ follows directly from that
  $\tau_{l'}$, $\pi_{l'}$, and $\sigma_{l',i'}$ are continuous for all $(l',i') \in \Idx$.
  Similarly, the differentiability of $\smash{y_l^\gamma}$ and $\smash{z_l^\gamma}$ follows directly from that
  $\tau_{l'}$, $\pi_{l'}$, and $\smash{\sigma_{l',i'}^{k'}}$ are differentiable for all $(l',i') \in \Idx$ and $k' \in [K_{l',i'}]$.
\end{proof}

\begin{lemma}
  \label{lem:r-gamma-equiv}
  Let $\gamma \in \Gamma$. Then,
  \begin{align*}
    \mcR^\gamma &= \{  w \in \bbR^W \;|\; y_{l,i}^\gamma(w) \in \mcI_{l,i}^{\gamma({l,i})} \text{ for all } (l,i) \in \Idx \}.
  \end{align*}
  Note that the RHS uses $\smash{y_{l,i}^\gamma}$ instead of $y_{l,i}$.
\end{lemma}
\begin{proof}
  Let $\gamma \in \Gamma$.
  Define $\mcR^{\gamma}_{\leq l}, \mcS^{\gamma}_{\leq l} \subseteq \bbR^W$ for $l \in [L]$ as
  \begin{align*}
    \mcR^{\gamma}_{\leq l}
    & \defeq  \{  w \in \bbR^W \;|\; y_{l',i}(w) \in \mcI_{l',i}^{\gamma({l',i})} \text{ for all $(l',i) \in \Idx$ with $l' \leq l$} \},
    \\
    \mcS^{\gamma}_{\leq l}
    &\defeq \{  w \in \bbR^W \;|\; y_{l',i}^\gamma(w) \in \mcI_{l',i}^{\gamma({l',i})} \text{ for all $(l',i) \in \Idx$ with $l' \leq l$}\}.
  \end{align*}
  It suffices to show the following claim which generalizes this lemma: all $l \in [L]$,
  \begin{align*}
    y_l(w) = y_l^\gamma(w) \;\;\text{for all $w \in \mcR^\gamma_{\leq l-1}$},
    \qquad\quad
    \mcR^\gamma_{\leq l} = \mcS^\gamma_{\leq l}.
  \end{align*}
  We prove this claim by induction on $l$.

  \paragraph{\bf Case $l=1$.}
  Since $z_0 = z_0^\gamma$, we have the first claimed equation:
  \[ y_1(w) = \tau_1(z_0(w), w_1) = \tau_1(z_0^\gamma(w), w_1) = y_1^\gamma(w) \]
  for all $w \in \bbR^W$.
  From this, we have the second claimed equation:
  \begin{align*}
    \mcR^\gamma_{\leq 1}
    = \bigcap_{i \in [N_1]} \{w \in \bbR^W \mid y_{1,i}(w) \in \mcI_{1,i}^{\gamma(1,i)} \}
    = \bigcap_{i \in [N_1]} \{w \in \bbR^W \mid y_{1,i}^\gamma(w) \in \mcI_{1,i}^{\gamma(1,i)} \}
    = \mcS^\gamma_{\leq 1}.
  \end{align*}
  
  \paragraph{\bf Case $l>1$.}
  We obtain the first claimed equation as follows:
  for all $w \in \mcR^\gamma_{\leq l-1}$,
  \begin{align*}
    y_l^\gamma(w)
    &= \tau_l\big( \sigma_{l-1}^\gamma(y_{l-1}^\gamma(w)), \pi_l(w) \big)
    \\
    &= \tau_l\big( \sigma_{l-1}^\gamma(y_{l-1}(w)), \pi_l(w) \big)
    \\
    &= \tau_l\big( \sigma_{l-1}(y_{l-1}(w)), \pi_l(w) \big)
    = y_{l}(w).
  \end{align*}
  Here the second line uses ${y_{l-1}^\gamma(w)} = y_{l-1}(w)$,
  which holds by induction hypothesis on $l-1$ with $w \in {\mcR^\gamma_{\leq l-1}} \subseteq {\mcR^\gamma_{\leq l-2}}$.
  And the third line uses ${\sigma_{l-1,i}^{\gamma(l-1,i)}}(y_{l-1,i}(w)) = \sigma_{l-1,i}(y_{l-1,i}(w))$ for all $i \in [N_{l-1}]$,
  which holds because $y_{l-1,i}(w) \in \smash{\mcI_{l-1,i}^{\gamma(l-1,i)}}$ (by $w \in \mcR^\gamma_{\leq l-1}$)
  and $\smash{\{(\mcI_{l-1,i}^k, \sigma_{l-1,i}^k)\}_{k \in [K_{l-1,i}]}}$ is a representation of $\sigma_{l-1,i}$.
  Using this result, we obtain the second claimed equation as follows:
  \begin{align*}
    \mcR^\gamma_{\leq l}
    &= \mcR^\gamma_{\leq l-1} \cap \bigcap_{i \in [N_l]} \{w \in \mcR^\gamma_{\leq l-1} \mid y_{l,i}(w) \in \mcI_{l,i}^{\gamma(l,i)} \}
    \\
    &= \mcR^\gamma_{\leq l-1} \cap \bigcap_{i \in [N_l]} \{w \in \mcR^\gamma_{\leq l-1} \mid y_{l,i}^\gamma(w) \in \mcI_{l,i}^{\gamma(l,i)} \}
    \\
    &= \mcS^\gamma_{\leq l-1} \cap \bigcap_{i \in [N_l]} \{w \in \mcS^\gamma_{\leq l-1} \mid y_{l,i}^\gamma(w) \in \mcI_{l,i}^{\gamma(l,i)} \}
    = \mcS^\gamma_{\leq l},
  \end{align*}
  where the second line uses $\smash{y_{l,i}^\gamma(w)} = y_{l,i}(w)$ for all $w \in \smash{\mcR^\gamma_{\leq l-1}}$, which we already proved,
  and the third line uses $\mcR^\gamma_{\leq l-1} = \mcS^\gamma_{\leq l-1}$, which holds by induction hypothesis on $l-1$.
\end{proof}

\begin{lemma}
  \label{lem:r-gamma-f-df}
  Let $\gamma \in \Gamma$. Then, for all $l \in [L]$ and $w \in \mcR^\gamma$,  
  \begin{align*}
    y_l^\gamma(w) &= y_l(w),
    \qquad\quad z_l^\gamma(w) = z_l(w).
  \end{align*}
\end{lemma}
\begin{proof}
  Let $\gamma \in \Gamma$.
  The claim shown in the proof of \Cref{lem:r-gamma-equiv} implies the first part of the conclusion
  (since $\mcR^\gamma_{\leq l-1} \supseteq \mcR^\gamma$):
  for all $l \in [L]$ and $w \in \mcR^\gamma$, $\smash{y_l^\gamma}(w) = y_l(w)$.
  From this, we obtain the second part of the conclusion:
  for all $l \in [L]$ and $w \in \mcR^\gamma$,
  \begin{align*}
    z_l^\gamma(w)
    = \sigma_l^\gamma(y_l^\gamma(w))
    = \sigma_l^\gamma(y_l(w))
    = \sigma_l(y_l(w))
    = z_l(w),
  \end{align*}
  where the second equality follows from the first part of the conclusion,
  and the third equality from $\smash{\sigma_{l,i}^{\gamma(l,i)}}(y_{l,i}(w)) = \sigma_{l,i}(y_{l,i}(w))$
  which holds because $y_{l,i}(w) \in \smash{\mcI_{l,i}^{\gamma(l,i)}}$ (by $w \in \mcR^\gamma$)
  and $\smash{\{(\mcI_{l,i}^k, \sigma_{l,i}^k)\}_{k \in [K_{l,i}]}}$ is a representation of $\sigma_{l,i}$.
\end{proof}

\subsection{Automatic Differentiation}
\label{sec:pf-ad}

\edit{%
As discussed in \Cref{sec:intro}, AD operates not on mathematical functions,
but on programs that represent those functions.
To this end, 
we define a program $\code{\ttP}$ that represents a function from $\bbR^W$ to $\bbR$ as follows:
\begin{align*}
  \code{\ttP} ::= \code{r} \;|\; \code{\ttw_{l,j}} \;|\; \code{f \ttlp \ttP_1, \ldots, \ttP_n \ttrp}
\end{align*}
where $\code{r} \in \bbR$, $l \in [L]$, $j \in [W_l]$, $f \in \{ \tau_{l,i}, \sigma_{l,i} \mid (l,i) \in \Idx \}$, and $n \in \bbN$.
This definition says that a program $\ttP$ can be either a real-valued constant $r$, a real-valued parameter $\ttw_{l,j}$,
or the application of a function $f : \bbR^n \to \bbR$ to subprograms $\smash{\ttP_1}, \ldots, \smash{\ttP_n}$.
In this paper, we focus on particular programs $\smash{\code{\ttP_{y_{l,i}}}}$ and $\smash{\code{\ttP_{z_{l,i}}}}$
that represent the functions $y_{l,i}(\,\cdot\,;c), z_{l,i}(\,\cdot\,;c) : \bbR^W \to \bbR$ and are defined in a canonical way as follows:
\begin{align*}
  \code{\ttP_{y_{l,i}}}
  &\defeq \code{\tau_{l,i} \ttlp \ttP_{z_{l-1,1}}, \ldots, \ttP_{z_{l-1,N_{l-1}}}, \ttw_{l,1}, \ldots, \ttw_{l,W_l} \ttrp},
  \\
  \code{\ttP_{z_{l,i}}}
  &\defeq \code{\sigma_{l,i} \ttlp \ttP_{y_{l,i}} \ttrp},
\end{align*}
where $\smash{\code{\ttP_{z_{0,i'}}}} \defeq \smash{c_{i'}}$ for $i' \in [N_0]$
represents the constant function $z_{0,i'}(\,\cdot\,;c) : \bbR^W \to \bbR$.

Given a program $\code{\ttP}$, we define $\sem{\code{\ttP}} : \bbR^W \to \bbR$ as the function represented by $\ttP$,
and $\semad{\code{\ttP}} : \bbR^W \to \bbR^{1 \times W}$ as the function that AD essentially computes when applied to $\ttP$. These functions are defined inductively as follows \cite{LeeYRY20, AbadiP20, BaydinPRS17}:
\begin{align*}
  \sem{\code{r}}(w) &\defeq r,
  \\
  \sem{\code{\ttw_{l,j}}}(w) &\defeq w_{l,j},
  \\
  \sem{\code{f \ttlp \ttP_1, \ldots, \ttP_n \ttrp}}(w)
  & \defeq f\big(\sem{\code{\ttP_1}}(w), \ldots, \sem{\code{\ttP_n}}(w)\big),
  \\
  \semad{\code{r}}(w) &\defeq \mathbb{0}, 
  \\
  \semad{\code{\ttw_{l,j}}}(w) &\defeq \mathbb{1}_{l,j},
  \\
  \!\!
  \semad{\code{f \ttlp \ttP_1, \ldots, \ttP_n \ttrp}}(w)
  & \defeq \adf{f}\big(\sem{\code{\ttP_1}}(w), \ldots, \sem{\code{\ttP_n}}(w)\big)
  \cdot \big[ \semad{\code{\ttP_1}}(w) \,\big/ \cdots \big/\, \semad{\code{\ttP_n}}(w) \big].
\end{align*}
Here $w_{l,j} \in \bbR$ is defined as $(w_{1,1}, w_{1,2}\ldots, w_{L, W_L}) \defeq w$,
$\mathbb{0}, \mathbb{1}_{l,j} \in \bbR^{1 \times W}$ denote the zero matrix
and the matrix whose entries are all zeros except for a single one at the $(W_1+\cdots+W_{l-1}+j)$-th entry,
$\adf{f} : \bbR^n \to \bbR^{1 \times n}$ denotes a ``derivative'' of $f$ used by AD,
and $[M_1 \,/\cdots/\, M_n]$ denotes the matrix that stacks up matrices $M_1, \ldots, M_n$ vertically.
Note that $\semad{\code{f \ttlp \ttP_1, \ldots, \ttP_n \ttrp}}$ captures the essence of AD:
it computes derivatives based on the chain rule for differentiation.

Using the above definitions, we define $\ADF{z_L} : \bbR^W \to \bbR^{N_L \times W}$
as what AD essentially computes when applied
to a program 
that canonically represents a neural network $z_L : \bbR^W \to \bbR^{N_L}$: 
\begin{align*}
  \ADF{z_L}(w) \defeq \big[ \semad{\code{\ttP_{z_{L,1}}}}(w) \,\big/ \cdots \big/\, \semad{\code{\ttP_{z_{L,N_L}}}}(w) \big].
\end{align*}
Note that $\ADF{z_L}$ depends on the ``derivative'' of (pre-)activation functions (i.e., $\adf{\sigma_{l,i}}$ and $\adf{\tau_{l,i}}$) used by AD.
}%

\begin{lemma}
  \label{lem:r-gamma-adf}
  For any $\gamma \in \Gamma$ and $w \in \mcR^\gamma$,  
  \begin{align*}
    \ADF{z_L}(w) = \DF{z_L^\gamma}(w).
  \end{align*}
\end{lemma}
\begin{proof}
  Let $\gamma \in \Gamma$.
  We prove the following claim: for all $l \in [L] \cup \{0\}$, $i \in [N_l]$, and $w \in \mcR^\gamma$,
  \begin{align*}
    \DF{z_{l,i}^\gamma}(w) = \semad{\code{\ttP_{z_{l,i}}}}(w).
  \end{align*}
  Note that this claim implies the conclusion since
  $$ \smash{\DF{z_L^\gamma}(w)}
  = [ \smash{\DF{z_{L,1}^\gamma}(w)} \,/\cdots/\, \smash{\DF{z_{L,N_L}^\gamma}(w)} ]
  = [ \smash{\semad{\code{\ttP_{z_{L,1}}}}(w)} \,/\cdots/\, \smash{\semad{\code{\ttP_{z_{L,N_L}}}}(w)} ]
  = \smash{\ADF{z_L}(w)}.$$
  We prove the claim by induction on $l$.
  
  \paragraph{\bf Case $l=0$.}
  Let $i \in [N_l]$ and $w \in \mcR^\gamma$.
  Since $\smash{\code{\ttP_{z_{0,i}}}}$ is a constant program, 
  $\smash{\DF{z_{0,i}^\gamma}}(w) = \mathbb{0} = \smash{\semad{\code{\ttP_{z_{0,i}}}}}(w)$ as desired.

  \paragraph{\bf Case $l>0$.}
  Let $i \in [N_l]$ and $w \in \mcR^\gamma$.
  Observe that
  \begin{align}
    \nonumber
    \semad{\code{\ttP_{y_{l,i}}}}(w)
    &= \semad{\code{\tau_{l,i} \ttlp \ttP_{z_{l-1,1}}, \ldots, \ttP_{z_{l-1,N_{l-1}}}, \ttw_{l,1}, \ldots, \ttw_{l,W_l} \ttrp}}(w)
    \\ \nonumber
    &= \DF{\tau_{l,i}}\big( \sem{\code{\ttP_{z_{l-1,1}}}}(w), \ldots, \sem{\code{\ttP_{z_{l-1,N_{l-1}}}}}(w),
    \sem{\code{\ttw_{l,1}}}(w), \ldots, \sem{\code{\ttw_{l,N_l}}}(w) \big)
    \\ \nonumber
    &\qquad \cdot \big[
      \semad{\code{ \ttP_{z_{l-1,1}} }}(w) \,/ \cdots /\, \semad{\code{ \ttP_{z_{l-1,N_{l-1}}} }}(w)
      \,/\,
      \semad{\code{ \ttw_{l,1} }}(w) \,/ \cdots /\, \semad{\code{ \ttw_{l,W_l} }}(w)
    \big]
    \\ \nonumber
    &= \DF{\tau_{l,i}} \big(z_{l-1}(w), \pi_l(w)\big)
    \cdot \big[
      \DF{z_{l-1,1}^\gamma}(w) \,/ \cdots /\, \DF{z_{l-1,N_{l-1}}^\gamma}(w)
      \,/\,
      \mathbb{1}_{l,1} \,/\cdots/\, \mathbb{1}_{l,N_l}
    \big]
    \\ \nonumber
    &= \DF{\tau_{l,i}} \big(z_{l-1}(w), \pi_l(w)\big) \cdot \big[ \DF{z_{l-1}^\gamma}(w) \,/\, \DF{\pi_l}(w) \big]
    \\ \label{eq:lem:r-gamma-adf-1}
    &= \DF{\tau_{l,i}}\big((z_{l-1}, \pi_l)(w)\big) \cdot \DF{(z_{l-1}^\gamma, \pi_l)}(w),
    \vphantom{\big[}
  \end{align}
  where $(f,g) : \bbR^n \to \bbR^{m_1+m_2}$ is defined as $(f,g)(x) \defeq (f(x), g(x))$
  for $f: \bbR^n \to \bbR^{m_1}$ and $g: \bbR^n \to \bbR^{m_2}$.
  Here the third line uses $\smash{\sem{\code{\ttP_{z_{l-1,i'}}}}(w)} = z_{l-1,i'}(w)$
  and $\smash{\semad{\code{\ttP_{z_{l-1,i'}}}}(w)} = \smash{\DF{z_{l-1,i'}^\gamma}}(w)$ for all $i' \in [N_{l-1}]$,
  where the latter holds by induction hypothesis on $l-1$.
  
  Using the observation above, we obtain the claim:
  \begin{align*}
    \semad{\code{\ttP_{z_{l,i}}}}(w)
    &= \semad{\code{\sigma_{l,i} \ttlp \ttP_{y_{l,i}} \ttrp}}(w)
    \\
    &= \adf{\sigma_{l,i}}\big( \sem{\code{\ttP_{y_{l,i}}}}(w) \big) \cdot \semad{\code{\ttP_{y_{l,i}}}}(w)
    \\
    &= \adf{\sigma_{l,i}}\big(y_{l,i}(w)\big) \cdot \DF{\tau_{l,i}}\big( (z_{l-1}, \pi_l)(w) \big) \cdot \DF{(z_{l-1}^\gamma, \pi_l)}(w)
    \\
    &= \DF{\sigma_{l,i}^\gamma} \big(y_{l,i}(w)\big) \cdot \DF{\tau_{l,i}}\big( (z_{l-1}, \pi_l)(w) \big) \cdot \DF{(z_{l-1}^\gamma, \pi_l)}(w)
    \\
    &= \DF{\sigma_{l,i}^\gamma} \big( (\tau_{l,i} \circ (z_{l-1}^\gamma, \pi_l))(w) \big)
    \cdot \DF{\tau_{l,i}}\big( (z_{l-1}^\gamma, \pi_l)(w) \big) \cdot \DF{(z_{l-1}^\gamma, \pi_l)}(w)
    \\
    &= \DF{(\sigma_{l,i}^\gamma \circ \tau_{l,i} \circ (z_{l-1}^\gamma, \pi_l))}(w)
    \\
    &= \DF{z_{l,i}^\gamma}(w).
  \end{align*}
  Here the third line uses $\smash{\sem{\code{\ttP_{y_{l,i}}}}(w)} = y_{l,i}(w)$ and \Cref{eq:lem:r-gamma-adf-1},
  and the fourth line uses $\adf{\sigma_{l,i}}(y_{l,i}(w)) = \smash{\DF{\sigma_{l,i}^{\gamma(l,i)}}}(y_{l,i}(w))$,
  which holds because $y_{l,i}(w) \in \smash{\mcI_{l,i}^{\gamma(l,i)}}$ (by $w \in \mcR^\gamma$)
  and $\smash{\{(\mcI_{l,i}^k, \sigma_{l,i}^k)\}_{k \in [K_{l,i}]}}$ defines $\adf{\sigma_{l,i}}$.
  The fifth line uses $y_{l,i}(w) = \smash{y_{l,i}^\gamma(w)}$ and $z_{l-1}(w) = {z_{l-1}^\gamma(w)}$
  (by \Cref{lem:r-gamma-f-df} with $w \in \mcR^\gamma$),
  and the sixth line uses the chain rule,
  which is applicable to $({\sigma_{l,i}^\gamma} \circ \tau_{l,i} \circ (\smash{z_{l-1}^\gamma}, \pi_l))$
  because $\smash{\sigma_{l,i}^\gamma}$, $\tau_{l,i}$, $\smash{z_{l-1}^\gamma}$, and $\pi_l$ are differentiable
  (as $\smash{z_{l-1}^\gamma}$ is differentiable by \Cref{lem:yz-yz-gamma-good}).
\end{proof}

\clearpage
\section{Upper Bounds on $|\ndfM{z_L} \cup \incM{z_L}|$}
\label{sec:pf-ndfinc}

\subsection{Lemmas (Basic)}

\begin{lemma}
  \label{lem:pbd-basic}
  For any $A, B \subseteq \bbR^n$,
  \begin{align*}
    \pbd(A \cup B) &\subseteq \pbd(A) \cup \pbd(B),
    \qquad\quad
    \pbd(A \cap B) \subseteq \pbd(A) \cup \pbd(B).
  \end{align*}
\end{lemma}
\begin{proof}
  Let $A, B \subseteq \bbR^n$.
  Then, $\intr(A \cup B) \supseteq \intr(A) \cup \intr(B)$ and $\intr(A \cap B) = \intr(A) \cap \intr(B)$.
  Using these, we obtain:
  \begin{align*}
    \pbd(A \cup B)
    &= (A \cup B) \setminus \intr(A \cup B)
    \\
    &= (A \setminus \intr(A \cup B)) \cup (B \setminus \intr(A \cup B))
    \\
    &\subseteq (A \setminus \intr(A))  \cup (B \setminus \intr(B))
    \\
    &=\pbd(A) \cup \pbd(B),
    \\
    \pbd(A \cap B)
    &= (A \cap B) \setminus \intr(A \cap B)
    \\
    &= (A \cap B) \setminus (\intr(A) \cap \intr(B))
    \\
    &= ((A \cap B) \setminus \intr(A)) \cup ((A \cap B) \setminus \intr(B))
    \\
    &\subseteq (A \setminus \intr(A))  \cup (B \setminus \intr(B))
    \\
    &=\pbd(A) \cup \pbd(B).
    \qedhere
  \end{align*}
\end{proof}

\begin{lemma}
  \label{lem:sol-count-basic}
  Let $f: \bbR^{n} \to \bbR$ be a function defined as
  \( f(x) = g(x_{-n}) + c \cdot x_n \)
  for any $g : \bbR^n \to \bbR$ and $c \in \bbR \setminus \{0\}$, where $x_{-n}$ denotes $(x_1, \ldots, x_{n-1})$.
  Then,
  \[
  \big|\{ x \in \bbM^n \mid f(x) = 0 \}\big| \leq |\bbM|^{n-1}.
  \]
\end{lemma}
\begin{proof}
  Using the definition of $f$ and $c \neq 0$, we obtain the conclusion:
  \begin{align*}
    \big|\{ x \in \bbM^n \mid f(x) = 0 \}\big|
    &=
    \big|\{ (x_{-n}, x_n) \in \bbM^{n-1} \times \bbM \mid f(x_{-n}, x_n) = 0 \}\big|
    \\
    &= \textstyle \sum_{x_{-n} \in \bbM^{n-1}}
    \big|\{ x_n \in \bbM \mid x_n = -g(x_{-n}) / c \}\big|
    \\
    &\leq \textstyle \sum_{x_{-n} \in \bbM^{n-1}} 1
    = |\bbM|^{n-1}.
    \qedhere
  \end{align*}
\end{proof}

\subsection{Lemmas (Technical: Part 1)}

\begin{definition*}
  For a neural network $z_L : \bbR^W \to \bbR^{N_L}$,
  define the incorrect set and the non-differentiable set of $z_L$ {\em over $\bbR^W$} (not over $\Omega$) as:
  \begin{align*}
    \incR{z_L} &\defeq \{w \in \bbR^W \mid \DF{z_L}(w) \neq \bot,\, \ADF{z_L}(w) \neq \DF{z_L}(w) \},
    \\
    \ndfR{z_L} &\defeq \{w \in \bbR^W \mid \DF{z_L}(w) = \bot\}.
  \end{align*}
\end{definition*}

\begin{lemma}
  \label{lem:ndf-inc-bound}
  We have
  \begin{align*}
    \ndfR{z_L} \cup \incR{z_L} \subseteq \bigcup_{\gamma \in \Gamma} \pbd(\mcR^\gamma).
  \end{align*}
\end{lemma}
\begin{proof}
  First, observe that for all $\gamma \in \Gamma$,
  \begin{alignat*}{6}
    \ADF{z_L}(w) &= \DF{z_L^\gamma}(w) = \DF{z_L}(w)
    \qquad\text{for all $w \in \intr(\mcR^\gamma)$},
  \end{alignat*}
  where the first equality is by \Cref{lem:r-gamma-adf},
  and the second equality is obtained by applying the following fact to $(\smash{z_L^\gamma}, z_L, \intr(\mcR^\gamma))$:
  for any $f,g : \bbR^n \to \bbR^m$ and open $U \subseteq \bbR^n$,
  if $f$ is differentiable on $U$ and $f = g$ on $U$, then $g$ is differentiable on $U$ and $\DF{f} = \DF{g}$ on $U$.
  Note that the previous fact is applicable
  since $\intr(\mcR^\gamma)$ is open, $\smash{z_{L}^\gamma}$ is differentiable (by \Cref{lem:yz-yz-gamma-good}),
  and $\smash{z_L^\gamma} = z_L$ on $\intr(\mcR^\gamma)$ by \Cref{lem:r-gamma-f-df}.
  
  From the above equation, we have
  \begin{align*}
    \bigcup_{\gamma \in \Gamma} \intr(\mcR^\gamma) &\subseteq \bbR^W \setminus \big(\ndfR{z_L} \cup \incR{z_L}\big).
  \end{align*}
  From this, we obtain the conclusion:
  \begin{align*}
    \ndfR{z_L} \cup \incR{z_L}
    &\subseteq \bbR^W \setminus \bigcup_{\gamma \in \Gamma} \intr(\mcR^\gamma)
    \\
    &= \Big( \bigcup_{\gamma \in \Gamma} \mcR^\gamma \Big) \setminus \Big( \bigcup_{\gamma \in \Gamma} \intr(\mcR^\gamma) \Big)
    =  \bigcup_{\gamma \in \Gamma} \big( \mcR^\gamma  \setminus \intr(\mcR^\gamma) \big)
    =  \bigcup_{\gamma \in \Gamma} \pbd(\mcR^\gamma),
  \end{align*}
  where the first equality is by \Cref{lem:r-gamma-partition},
  and the last equality is by the definition of $\pbd(-)$.
\end{proof}

\begin{lemma}
  \label{lem:pbd-bound-basic}
  We have
  \begin{align*}
    \bigcup_{\gamma \in \Gamma} \pbd(\mcR^\gamma)
    &\subseteq \bigcup_{(l,i) \in \Idx} \, \bigcup_{c \in \ndf{\sigma_{l,i}}}
    \pbd\big(\{ w \in \bbR^W \;|\; y_{l,i}(w) = c \}\big).
  \end{align*}
\end{lemma}
\begin{proof}
  First, we have
  \begin{align}
    \nonumber
    \bigcup_{\gamma \in \Gamma} \pbd(\mcR^\gamma)
    &= \,\,\,\, \bigcup_{\gamma \in \Gamma}  \, \pbd\Big( \bigcap_{(l,i) \in \Idx} 
    \{ w \in \bbR^W \;|\; y_{l,i}(w) \in \mcI_{l,i}^{\gamma(l,i)} \} \Big)
    \\ \nonumber
    &\subseteq \,\,\,\, \bigcup_{\gamma \in \Gamma}  \,\,\,\, \bigcup_{(l,i) \in \Idx} 
    \pbd\Big( \{ w \in \bbR^W \;|\; y_{l,i}(w) \in \mcI_{l,i}^{\gamma(l,i)} \}\Big)
    \\ \nonumber
    &=            \bigcup_{(l,i) \in \Idx} \,\,\,\, \bigcup_{\gamma \in \Gamma}    \,\,\,\,
    \pbd\Big( \{ w \in \bbR^W \;|\; y_{l,i}(w) \in \mcI_{l,i}^{\gamma(l,i)} \}\Big)
    \\ \label{eq:pbd-bound-basic-1}
    &=            \bigcup_{(l,i) \in \Idx}  \,      \bigcup_{k \in [K_{l,i}]} \,
    \pbd\Big( \{ w \in \bbR^W \;|\; y_{l,i}(w) \in \mcI_{l,i}^{k} \}\Big),
  \end{align}
  where the first line uses the definition of $\mcR^\gamma$,
  the second line uses \Cref{lem:pbd-basic}, and the last line uses that $\{ \gamma(l,i) \mid \gamma \in \Gamma \} = [K_{l,i}]$ for all $(l,i)$.
  Note that in the last two lines, we change the way we count the proper boundary of all subregions:
  from per subregion to per activation neuron.
  
  Next, for any $(l,i) \in \Idx$ and $k \in [K_{l,i}]$, we have
  \begin{align}
    \nonumber
    &\pbd\Big( \{ w \in \bbR^W \mid y_{l,i}(w) \in \mcI_{l,i}^{k} \}\Big)
    \\ \nonumber
    & = \pbd\Big(\{ w \in \bbR^W \mid y_{l,i}(w) \in \pbd(\mcI_{l,i}^{k}) \}
    \cup \{ w \in \bbR^W \mid y_{l,i}(w) \in \intr(\mcI_{l,i}^{k}) \}\Big)
    \\ \nonumber
    & \subseteq \pbd\Big(\{ w \in \bbR^W \mid y_{l,i}(w) \in \pbd(\mcI_{l,i}^{k}) \}\Big)
    \cup \pbd\Big(\{ w \in \bbR^W \mid y_{l,i}(w) \in \intr(\mcI_{l,i}^{k}) \}\Big)
    \\ \label{eq:pbd-bound-basic-2}
    & = \pbd\Big(\{ w \in \bbR^W \mid y_{l,i}(w) \in \pbd(\mcI_{l,i}^{k}) \}\Big),
  \end{align}
  where the third line is by \Cref{lem:pbd-basic} and the last line is by the following:
  $\pbd(A) = \emptyset$ for any open $A \subseteq \bbR^n$; and
  $\{ w \in \bbR^W \mid y_{l,i}(w) \in \intr(\smash{\mcI_{l,i}^{k}}) \}$ is open,
  because $y_{l,i}$ is continuous (by \Cref{lem:yz-yz-gamma-good}) and
  the inverse image of an open set by a continuous function is open.

  Finally, combining the above results, we obtain the conclusion:
  \begin{align*}
    \bigcup_{\gamma \in \Gamma} \pbd(\mcR^\gamma)
    & \subseteq \bigcup_{(l,i) \in \Idx} \,\,\, \bigcup_{k \in [K_{l,i}]} \,\,\,
    \pbd\Big(\{ w \in \bbR^W \;|\; y_{l,i}(w) \in \pbd(\mcI_{l,i}^{k}) \}\Big)
    \\
    & \subseteq \bigcup_{(l,i) \in \Idx} \,        \bigcup_{c \in \ndf{\sigma_{l,i}}}
    \pbd\Big(\{ w \in \bbR^W \;|\; y_{l,i}(w) =c  \}\Big),
  \end{align*}
  where the first line uses \Cref{eq:pbd-bound-basic-1,eq:pbd-bound-basic-2},
  and the second line uses $\smash{\bigcup_{k \in [K_{l,i}]}} \pbd(\smash{\mcI_{l,i}^k}) = \ndf{\sigma_{l,i}}$
  (by \Cref{def:minrep-sigma})
\end{proof}

\subsection{\Cref{thm:ndf-inc-ubound-bias} (Main Lemmas)}

\begin{lemma}
  \label{lem:pbd-bound-bias}
  We have
  \begin{align*}
    \ndfM{z_L} \cup \incM{z_L}
    &\subseteq \bigcup_{(l,i) \in \Idx} \, \bigcup_{c \in \ndf{\sigma_{l,i}}}
    \{ w \in \Omega \;|\; y_{l,i}(w) = c \}.
  \end{align*}
\end{lemma}
\begin{proof}
  We obtain the conclusion
  by chaining \Cref{lem:ndf-inc-bound}, \Cref{lem:pbd-bound-basic},
  and the following: $\pbd(A) \subseteq A$ for any $A \subseteq \bbR^W$,
  and $\ndfM{z_L} \cup \incM{z_L} = \big(\ndfR{z_L} \cup \incR{z_L}\big) \cap \Omega$.
\end{proof}

\begin{lemma}
  \label{lem:sol-count-bias}
  Let $(l,i) \in \Idx$ and $c \in \bbR$.
  Suppose that $\tau_{l}$ has bias parameters.
  Then, for $S = \{ w \in \Omega \;|\; y_{l,i}(w) = c  \}$,
  \begin{align*}
    | S | \leq |\bbM|^{W-1}.
  \end{align*}
\end{lemma}
\begin{proof}
  Suppose that $\tau_l$ has bias parameters and  $S$ is given as above.
  Then, by the definition of having bias parameters,
  $W_l \geq N_l$ and there is $\smash{\tau'_{l,i}} : \bbR^{N_{l-1}} \times \bbR^{W_l - N_l} \to \bbR$ for all $i \in [N_l]$ such that
  \begin{align*}
    \tau_{l,i}(x, (u, v))
    &= \smash{\tau'_{l,i}}(x, u) + v_i \qquad\text{for all $(u, v) \in \bbR^{W_l - N_l} \times \bbR^{N_l}$}.
  \end{align*}
  From this, we have
  \begin{align*}
    \!\!
    y_{l,i}(w)
    &= \tau_{l,i}(z_{l-1}(w), w_l)
    = \tau_{l,i}'\big(z_{l-1}(w_{1,1}, \ldots, w_{l-1, W_{l-1}}, 0, \ldots, 0), (w_{l,1}, \ldots, w_{l,W_l-N_l})\big) + w_{l,W_l-N_l+i},
  \end{align*}
  where we also use that $z_{l-1}$ depends only on $w_1, \ldots, w_{l-1}$.
  Note that the function $f : \bbR^W \to \bbR$ defined by $f(w) \defeq y_{l,i}(w) -c$
  satisfies the preconditions of \Cref{lem:sol-count-basic} (after reordering the input variables of $f$) due to the term $w_{l,W_l-N_l+i}$.
  Using this, we obtain the desired result:
  \begin{align*}
    |S |
    = |\{ w\in \Omega \mid f(w) = 0 \}|
    \leq |\bbM|^{W-1},
  \end{align*}
  where the inequality is by \Cref{lem:sol-count-basic} applied to $f$.
\end{proof}

\subsection{\Cref{thm:ndf-inc-ubound-bias} (Main Proof)}

\begin{theorem*}
  \label{thm:ndf-inc-ubound-bias-appdx}
  If $z_L$ has bias parameters, then
  \begin{align*}
    \frac{|\ndfM{z_L} \cup \incM{z_L}|}{|\Omega|}
    \leq \frac{1}{|\bbM|} {\sum_{(l,i) \in \Idx}} | \ndf{\sigma_{l,i}} |.
  \end{align*}
\end{theorem*}
\begin{proof}
  Observe that
  \begin{align}
    \label{eq:ndf-inc-ubound-bias-1}
    \ndfM{z_L} \cup \incM{z_L}
    &\subseteq 
    \bigcup_{(l,i) \in \Idx} \, \bigcup_{c \in A_{l,i}}  B_{l,i}(c),
    \qquad\quad
    |B_{l,i}(c) | \leq |\bbM|^{W-1},
  \end{align}
  where $A_{l,i} \defeq \ndf{\sigma_{l,i}}$ and $B_{l,i}(c) \defeq \{ w \in \Omega \;|\; y_{l,i}(w) = c \}$.
  Here the first equation is by \Cref{lem:pbd-bound-bias},
  and the second equation is by \Cref{lem:sol-count-bias} (which is applicable since $\tau_l$ has bias parameters by assumption).
  Combining the above observations, we obtain the conclusion:
  \begin{align*}
    \frac{|\ndfM{z_L} \cup \incM{z_L}|}{|\Omega|}
    &\leq \sum_{(l,i) \in \Idx} \, \sum_{c \in A_{l,i}} \frac{|B_{l,i}(c)|}{|\Omega|}
    \leq \sum_{(l,i) \in \Idx} |\ndf{\sigma_{l,i}}| \cdot \frac{|\bbM|^{W-1}}{|\bbM|^W}, 
  \end{align*}
  where the two inequalities use \Cref{eq:ndf-inc-ubound-bias-1}.
\end{proof}

\begin{remark*}
  \Cref{thm:ndf-inc-ubound-bias} is a direct corollary of \Cref{thm:ndf-inc-ubound-bias-appdx}
  and \Cref{thm:inc-zero-bias} (which we prove in \Cref{sec:pf-inc}).
\end{remark*}

\subsection{Lemmas (Technical: Part 2)}

\begin{lemma}
  \label{lem:strong-bilinear}
  Let $l \in [L]$.
  Suppose that $\tau_l : \bbR^{N_{l-1}} \times \bbR^{W_l} \to \bbR^{N_l}$ is well-structured biaffine.
  Then, for every $i \in [N_l]$, there is a partial map $\phi_{l,i} : [W_l] \rightharpoonup [N_{l-1}]$
  and associated matrix $M \in \bbR^{N_{l-1} \times W_l}$ and constant $d \in \bbR$ such that
  \begin{align*}
    y_{l,i}(w)
    &= d + \sum_{j \in \dom(\phi_{l,i})} z_{l-1, \phi_{l,i}(j)}(w) \cdot M_{\phi_{l,i}(j), j} \cdot w_{l, j}
  \end{align*}
  and $M_{\phi_{l,i}(j), j} \neq 0$ for all $j \in \dom(\phi_{l,i})$.
\end{lemma}
\begin{proof}
  Let $l \in [L]$, $\tau_l : \bbR^{N_{l-1}} \times \bbR^{W_l} \to \bbR^{N_l}$ be a well-structured biaffine function, and $i \in [N_l]$.
  Then, there is a matrix $M \in \bbR^{N_{l-1} \times W_l}$ and a constant $d \in \bbR$ such that
  $\tau_{l,i}(x,u) = x^\trsp M u + d$ for all $(x,u)$ and each column of $M$ has at most one non-zero entry.
  Define a partial map $\phi_{l,i} : [W_l] \rightharpoonup [N_{l-1}]$ as:
  \begin{align*}
    \phi_{l,i}(j)
    & \defeq
    \begin{cases}
      i' & \text{if $M_{i',j} \neq 0$ for some $i' \in [N_{l-1}]$} \\
      \text{undefined} & \text{otherwise}.
    \end{cases}
  \end{align*}
  Here $\phi_{l,i}$ is well-defined because $M_{-,j}$ contains at most one non-zero entry for all $j \in [W_l]$.
  We claim that $\phi_{l,i}$, $M$, and $d$ satisfy the conditions in this lemma.
  First, by the definition of $\phi_{l,i}$, $M_{\phi_{l,i}(j),j} \neq 0$ for all $j \in \dom(\phi_{l,i})$.
  Also, we have the desired equation as follows:
  \begin{align*}
    y_{l,i}(w)
    &= \tau_{l,i}(z_{l-1}(w), w_l)
    \\
    &= d + (\smash{z_{l-1}(w)^\trsp} M) \cdot w_l
    \\
    &= d + \smash{(v_1, \ldots, v_{W_{l-1}})^\trsp} \cdot w_l
    \\
    &= d + \textstyle \,\smash{\sum_{j \in [W_{l-1}]}} \, v_j \cdot w_{l,j}
    \\
    &= d + \textstyle \, \smash{\sum_{j \in \dom(\phi_{l,i})}} \, z_{l-1,\phi_{l,i}(j)}(w) \cdot M_{\phi_{l,i}(j), j} \cdot w_{l,j},
  \end{align*}
  where $v_{j} \in \bbR$ is defined as
  $v_j \defeq z_{l-1,\phi_{l,i}(j)}(w) \cdot M_{\phi_{l,i}(j), j}$ if $j \in \dom(\phi_{l,i})$, and $v_j \defeq 0$ otherwise.
  Here the second line uses the definition of $M$ and $d$,
  and the third and last lines use the definition of $v_{j}$.
  This concludes the proof.
\end{proof}

\begin{lemma}
  \label{lem:pbd-bound-nobias-aux}
  For every $(l,i) \in \Idx$ and $c \in \bbR$,
  let $A_{l,i} \subseteq \bbR$ be any set and $B_{l,i}(c) \subseteq \bbR^W$ be the set $\{ w \in \bbR^W \;|\; y_{l,i}(w) = c \}$.
  Suppose that for every $l \in[L]$, one of the following holds:
  \begin{itemize}
    \item[(a)] $\tau_l$ has bias parameters, or
    \item[(b)] $\tau_l$ is well-structured biaffine.
  \end{itemize}
  In the case of (b), let $\phi_{l,i}$ be the partial map described in \Cref{lem:strong-bilinear} for all $i\in[N_{l}]$.
  Then,
  \begin{align*}
    \bigcup_{(l,i) \in \Idx} \, \bigcup_{c \in A_{l,i}} \pbd(B_{l,i}(c))
    &\subseteq \bigcup_{(l,i) \in \Idx} \, \bigcup_{c' \in A'_{l,i}} B'_{l,i}(c'),
  \end{align*}
  where $\smash{A'_{l,i}} \subseteq \bbR$ and $\smash{B'_{l,i}(c')} \subseteq \bbR^W$ are defined as
  \begin{align*}
    A'_{l,i} &\defeq
    \begin{cases}
      A_{l,i}
      & \text{if $\tau_{l+1}$ satisfies the condition (a) or $l=L$} 
      \\
      A_{l,i} \cup \bdz{\sigma_{l,i}}
      & \text{if $\tau_{l+1}$ satisfies the condition (b)},
    \end{cases}
    \\
    B'_{l,i}(c') &\defeq
    \begin{cases}
      B_{l,i}(c')
      & \text{if $\tau_{l}$ satisfies the condition (a)} 
      \\
      B_{l,i}(c') \cap \bigcup_{j \in \dom(\phi_{l,i})} \{ w \in \bbR^W \;|\; z_{l-1,\phi_{l,i}(j)}(w) \neq 0  \}
      & \text{if $\tau_{l}$ satisfies the condition (b)}.
    \end{cases}
  \end{align*}
\end{lemma}
\begin{proof}
  We claim that the following holds: for all $l \in [L]$, $i \in [N_l]$, and $c \in \smash{A'_{l,i}}$,
  \begin{align}
    \label{eq:pbd-bound-nobias-aux-goal}
    \pbd(B_{l,i}(c)) &\subseteq \bigcup_{(l',i') \in \Idx} \, \bigcup_{c' \in A'_{l',i'}} B'_{l',i'}(c').
  \end{align}
  This claim implies the conclusion because $A_{l,i} \subseteq \smash{A'_{l,i}}$ for all $(l,i) \in \Idx$
  (by the definition of $\smash{A'_{l,i}}$).
  We prove the claim by induction on $l$.

  \paragraph{\bf Case $l=1$.}
  Let $i \in [N_l]$ and $c \in \smash{A'_{l,i}}$.
  We prove \Cref{eq:pbd-bound-nobias-aux-goal} by case analysis on $\tau_l$.

  {\it Subcase 1:} $\tau_l$ satisfies the condition (a).
  In this subcase, \Cref{eq:pbd-bound-nobias-aux-goal} holds since
  \begin{align*}
    \pbd(B_{l,i}(c)) \subseteq B_{l,i}(c) = B'_{l,i}(c), \qquad c \in A'_{l,i},
  \end{align*}
  where the equality uses the definition of $\smash{B'_{l,i}}$. 

  {\it Subcase 2:} $\tau_l$ satisfies the condition (b).
  In this subcase, we have
  \begin{align}
    \nonumber
    \pbd(B_{l,i}(c))
    & =
    \pbd\Big(
    \Big( B_{l,i}(c) \cap \bigcup_{j \in \dom(\phi_{l,i})} \{ w \in \bbR^W \;|\; z_{l-1,\phi_{l,i}(j)}(w) \neq 0  \}  \Big)
    \\ \nonumber
    & \qquad \cup
    \Big( B_{l,i}(c) \cap \bigcap_{j \in \dom(\phi_{l,i})} \{ w \in \bbR^W \;|\; z_{l-1,\phi_{l,i}(j)}(w) = 0  \}  \Big)\Big)
    \\ \nonumber
    & \subseteq
    \pbd \Big( B_{l,i}(c) \cap \bigcup_{j \in \dom(\phi_{l,i})} \{ w \in \bbR^W \;|\; z_{l-1,\phi_{l,i}(j)}(w) \neq 0  \}  \Big)
    \\ \nonumber
    & \quad \cup
    \pbd \Big( B_{l,i}(c) \cap \bigcap_{j \in \dom(\phi_{l,i})} \{ w \in \bbR^W \;|\; z_{l-1,\phi_{l,i}(j)}(w) = 0  \}  \Big),
  \end{align}
  where the inclusion uses \Cref{lem:pbd-basic}.
  To prove \Cref{eq:pbd-bound-nobias-aux-goal}, 
  it suffices to show that the two terms in the last two lines are contained in the RHS of \Cref{eq:pbd-bound-nobias-aux-goal}.
  The first term does so because
  \begin{align*}
    \pbd\Big( B_{l,i}(c) \cap \bigcup_{j \in \dom(\phi_{l,i})} \{ w \in \bbR^W \;|\; z_{l-1,\phi_{l,i}(j)}(w) \neq 0  \}  \Big)
    & =
    \pbd(B'_{l,i}(c))
    \subseteq B'_{l,i}(c),
    \qquad c \in A'_{l,i},
  \end{align*}
  where the equality is by the definition of $\smash{B'_{l,i}}$ and that $\tau_l$ does not have bias parameters.
  The second term is also contained in the RHS of \Cref{eq:pbd-bound-nobias-aux-goal} as follows.
  Let \( S \defeq {\bigcap_{j \in \dom(\phi_{l,i})}} \{ w \in \bbR^W \mid  z_{l-1,\phi_{l,i}(j)}(w) = 0 \}, \)
  and $M \in \bbR^{N_{l-1} \times W_l}$ and $d \in \bbR$ be a matrix and a constant associated with $\phi_{l,i}$
  that are described in \Cref{lem:strong-bilinear}.
  Then, 
  \begin{align*}
    B_{l,i}(c) \cap S
    &= 
    \begin{cases}
      S & \text{if $c = d$} \\
      \emptyset & \text{if $c \neq d$},
    \end{cases}
  \end{align*}
  because $w \in S$ implies
  \(
  y_{l,i}(w)
  = d + \sum_{j \in \dom(\phi_{l,i})} z_{l-1, \phi_{l,i}(j)}(w) \cdot M_{\phi_{l,i}(j), j} \cdot w_{l, j}
  = d
  \)
  by \Cref{lem:strong-bilinear} (which is applicable since $\tau_l$ is well-structured biaffine by assumption).
  From this, we have
  \begin{align*}
    \pbd\Big( B_{l,i}(c) \cap \bigcap_{j \in \dom(\phi_{l,i})} \{ w \in \bbR^W \mid z_{l-1,\phi_{l,i}(j)}(w) = 0 \}\Big)
    & = \pbd(B_{l,i}(c) \cap S) 
    \subseteq \pbd(S) \cup \pbd(\emptyset).
  \end{align*}
  Hence, it suffices to show that $\pbd(S)$ is contained in the RHS of \Cref{eq:pbd-bound-nobias-aux-goal} (since $\pbd(\emptyset) = \emptyset$).
  Using $l=1$, we obtain this: 
  \begin{align*}
    \pbd(S) \subseteq \pbd(\bbR^W) \cup \pbd(\emptyset) = \emptyset,
  \end{align*}
  where the inclusion follows from $S \in \{\bbR^W, \emptyset\}$
  which holds because $z_{l-1,\phi_{l,i}(j)}$ is a constant function for all $j \in [N_{l-1}]$ (by $l=1$ and the assumption on $z_0$).

  \paragraph{\bf Case $l>1$.}
  Let $i \in [N_l]$ and $c \in \smash{A'_{l,i}}$.
  We prove \Cref{eq:pbd-bound-nobias-aux-goal} in the exact same way as we did for the case $l=1$.
  Note that the above proof for the previous case ($l=1$) applies directly to the current case ($l>1$), except for the following subclaim:
  if $\tau_l$ does not have bias parameters, then $\pbd(S)$ is contained in the RHS of \Cref{eq:pbd-bound-nobias-aux-goal}.
  This subclaim holds also for $l>1$, as follows:
  \begin{align*}
    \pbd(S)
    & =
    \pbd\Big(\bigcap_{j \in \dom(\phi_{l,i})} \{ w \in \bbR^W \mid z_{l-1,\phi_{l,i}(j)}(w) = 0 \}\Big)
    \\
    & \subseteq \bigcup_{j \in \dom(\phi_{l,i})} 
    \pbd\Big(\{ w \in \bbR^W \mid  z_{l-1,\phi_{l,i}(j)}(w) = 0 \}\Big)
    \\
    & =  \bigcup_{j \in \dom(\phi_{l,i})} 
    \pbd\Big(\big\{ w \in \bbR^W \;\big|\;  y_{l-1,\phi_{l,i}(j)}(w) \in \pbd\big(\sigma_{l-1,\phi_{l,i}(j)}^{-1}(0)\big) \big\}
    \\[-0.5em]
    & \qquad\qquad\qquad\quad
    \cup \big\{ w \in \bbR^W \;\big|\;  y_{l-1,\phi_{l,i}(j)}(w) \in \intr\big(\sigma_{l-1,\phi_{l,i}(j)}^{-1}(0)\big) \big\} \Big)
    \\
    & \subseteq  \bigcup_{j \in \dom(\phi_{l,i})}
    \pbd\Big(\big\{ w \in \bbR^W \;\big|\; y_{l-1,\phi_{l,i}(j)}(w) \in \pbd\big(\sigma_{l-1,\phi_{l,i}(j)}^{-1}(0)\big) \big\}\Big)
    \\[-0.5em]
    & \qquad\qquad\;\;\,
    \cup \pbd\Big( \big\{ w \in \bbR^W \;\big|\;  y_{l-1,\phi_{l,i}(j)}(w) \in \intr\big(\sigma_{l-1,\phi_{l,i}(j)}^{-1}(0)\big) \big\} \Big)
    \\
    & =  \bigcup_{j \in \dom(\phi_{l,i})} 
    \pbd\Big(\big\{ w \in \bbR^W \;\big|\; y_{l-1,\phi_{l,i}(j)}(w) \in \pbd\big(\sigma_{l-1,\phi_{l,i}(j)}^{-1}(0)\big) \big\}\Big)
    \\
    & =  \bigcup_{j \in \dom(\phi_{l,i})} \,
    \bigcup_{b  \in \bdz{\sigma_{l-1,\phi_{l,i}(j)}}} \pbd(B_{l-1,\phi_{l,i}(j)}(b)),
    \\
    \pbd(B_{l-1,\phi_{l,i}(j)}(b)) &\subseteq \bigcup_{(l',i') \in \Idx} \, \bigcup_{c' \in A'_{l',i'}} B'_{l',i'}(c')
    \qquad\text{for all $j \in \dom(\phi_{l,i})$ and $b \in \bdz{\sigma_{l-1,\phi_{l,i}(j)}}$}.
  \end{align*}
  Here the first and second inclusions use \Cref{lem:pbd-basic},
  and the second last equality uses that $y_{l-1,\phi_{l,i}(j)}$ is continuous (by \Cref{lem:yz-yz-gamma-good}).
  The last equality uses $\smash{\pbd(\sigma_{l-1,\phi_{l,i}(j)}^{-1}(0))} = \bdz{\sigma_{l-1,\phi_{l,i}(j)}}$
  (which holds since $\sigma_{l-1,\phi_{l,i}(j)}$ is continuous and the preimage of a closed set by a continuous map is closed), 
  and the definition of $B_{l-1,\phi_{l,i}(j)}$.
  The last inclusion is by the induction hypothesis applied to $(l-1, j, b)$ for $j \in \dom(\phi_{l,i})$ and $b \in \bdz{\sigma_{l-1,\phi_{l,i}(j)}}$,
  together with $\dom(\phi_{l,i}) \subseteq [N_{l-1}]$ and $\bdz{\sigma_{l-1,\phi_{l,i}(j)}} \subseteq \smash{A'_{l-1,\phi_{l,i}(j)}}$
  (which holds by the definition of $\smash{A'_{l-1,\phi_{l,i}(j)}}$ with $l-1 \neq L$ and that $\tau_l$ does not have bias parameters).
  Hence, \Cref{eq:pbd-bound-nobias-aux-goal} holds for $l>1$, and this concludes the proof.
\end{proof}

\subsection{\Cref{thm:ndf-inc-ubound-nobias} (Main Lemmas)}

\begin{lemma}
  \label{lem:pbd-bound-nobias}
  For every $l \in [L]$, suppose that $\tau_l$ satisfies either the condition (a) or (b) in \cref{lem:pbd-bound-nobias-aux}.
  Then,
  \begin{align*}
    \ndfM{z_L} \cup \incM{z_L}
    &\subseteq \bigcup_{(l,i) \in \Idx} \, \bigcup_{c \in A_{l,i}} B_{l,i}(c),
  \end{align*}
  where $A_{l,i} \subseteq \bbR$ and $B_{l,i}(c) \subseteq \Omega$ are defined as
  \begin{align*}
    A_{l,i} &\defeq
    \begin{cases}
      \ndf{\sigma_{l,i}}
      & \text{if $\tau_{l+1}$ satisfies the condition (a) or $l=L$} 
      \\
      \ndf{\sigma_{l,i}} \cup
      \bdz{\sigma_{l,i}}
      & \text{if $\tau_{l+1}$ satisfies the condition (b)},
    \end{cases}
    \\
    B_{l,i}(c) &\defeq
    \begin{cases}
      \{ w \in \Omega \;|\; y_{l,i}(w) = c  \} 
      & \text{if $\tau_{l}$ satisfies the condition (a)} 
      \\
      \{ w \in \Omega \;|\; y_{l,i}(w) = c  \land \bigvee_{j \in \dom(\phi_{l,i})} z_{l-1,\phi_{l,i}(j)}(w) \neq 0  \}
      & \text{if $\tau_{l}$ satisfies the condition (b)}.
    \end{cases}
  \end{align*}
\end{lemma}
\begin{proof}
  We obtain the conclusion
  by chaining \Cref{lem:ndf-inc-bound}, \Cref{lem:pbd-bound-basic},
  \Cref{lem:pbd-bound-nobias-aux}
  (which is applicable by assumption),
  and $\ndfM{z_L} \cup \incM{z_L} = \big(\ndfR{z_L} \cup \incR{z_L}\big) \cap \Omega$.
\end{proof}

\begin{lemma}
  \label{lem:sol-count-nobias}
  Let $(l,i) \in \Idx$ and $c \in \bbR$.
  Suppose that $\tau_l$ is well-structured biaffine.
  Consider $S = \{ w \in \Omega \;|\; y_{l,i}(w) = c  \land \smash{\bigvee_{j \in \dom(\phi_{l,i})}} z_{l-1,\phi_{l,i}(j)}(w) \neq 0  \}$,
  where $\phi_{l,i}$ denotes the partial map described in \Cref{lem:strong-bilinear}. 
  Then,
  \begin{align*}
    | S | \leq |\bbM|^{W-1}.
  \end{align*}
\end{lemma}
\begin{proof}
  Suppose that  $\tau_l$ is well-structured biaffine, and $S$ is given as above.
  We make three observations.
  First,
  \begin{align}
    \nonumber 
    S
    &= \big\{ (u,v) \in \bbM^{W'} \times \bbM^{W-W'} \;\big|\; \big(\exists j \in \dom(\phi_{l,i}).\, z_{l-1,\phi_{l,i}(j)}(u,0,\ldots,0) \neq 0\big)
    \land y_{l,i}(u,v) = c\}
    \\ \label{eq:sol-count-nobias-1}
    &= \bigcup_{u \in U} \bigcup_{v \in \bbM^{W-W'}} \{(u,v) \;|\; y_{l,i}(u,v) = c\},
  \end{align}
  where the first line uses $\smash{W'} \defeq W_1 + \cdots + W_{l-1}$
  and that $z_{l-1}$ depends only on $w_1, \ldots, w_{l-1}$,
  and the second line uses $U \defeq \{u \in \smash{\bbM^{W'}} \;|\; \exists j \in \dom(\phi_{l,i}).\, z_{l-1,\phi_{l,i}(j)}(u,0,\ldots,0) \neq 0\}$.
  Second, by \Cref{lem:strong-bilinear} (which is applicable since $\tau_l$ is well-structured biaffine by assumption),
  there are $M \in \bbR^{N_{l-1} \times W_l}$ and $d \in \bbR$ such that
  $M_{\phi_{l,i}(j),j} \neq 0$ for all $j \in \phi_{l,i}$, and
  \begin{align}
    \nonumber
    y_{l,i}(u,v)
    &= d + \hspace{0em} \sum_{j \in \dom(\phi_{l,i})} \hspace{0em} z_{l-1, \phi_{l,i}(j)}(u,v) \cdot M_{\phi_{l,i}(j), j} \cdot v_{j}
    \\
    \label{eq:sol-count-nobias-2}
    &= d + \hspace{0em} \sum_{j \in \dom(\phi_{l,i})} \hspace{0em} z_{l-1, \phi_{l,i}(j)}(u,0,\ldots,0) \cdot M_{\phi_{l,i}(j), j} \cdot v_{j}
  \end{align}
  for all $(u,v)\in \smash{\bbR^{W'} \times \bbR^{W-W'}}$,
  where the second equality uses that $z_{l-1}$ depends only on $u$.
  Third, for any $u \in U$,
  the function $f_u : \bbR^{W-W'} \to \bbR$ defined by $f_u(v) \defeq y_{l,i}(u,v) -c$ satisfies the preconditions of \Cref{lem:sol-count-basic}
  (after reordering the input variables of $f_u$) due to the following:
  $z_{l-1,\phi_{l,i}(j)}(u,0,\ldots,0) \neq 0$ for some $j \in \dom(\phi_{l,i})$ since $u \in U$;
  and the coefficient of $v_j$ in $f_u(v)$ is $z_{l-1,\phi_{l,i}(j)}(u,0,\ldots,0) \cdot M_{\phi_{l,i}(j),j} \neq 0$
  by \Cref{eq:sol-count-nobias-2} and $M_{\phi_{l,i}(j),j} \neq 0$.

  By combining the above observations, we obtain the conclusion:
  \begin{align*}
    |S|
    &= \Big| \bigcup_{u \in U} \bigcup_{v \in \bbM^{W-W'}} \{(u,v) \;|\; y_{l,i}(u,v) = c\} \Big|
    \\
    &=  \sum_{u \in U} \Big| \bigcup_{v \in \bbM^{W-W'}} \{(u,v) \;|\; y_{l,i}(u,v) = c\} \Big|
    \\
    &=  \sum_{u \in U}  \big| \{v \in \bbM^{W-W'} \mid f_u(v) = 0\} \big|
    \\
    &\leq |\bbM|^{W'} \cdot |\bbM|^{W-W'-1} = |\bbM|^{W-1},
  \end{align*}
  where the first line uses \Cref{eq:sol-count-nobias-1}, the third line uses the definition of $f_u$,
  and the last line uses \Cref{lem:sol-count-basic} applied to $f_u$.
\end{proof}

\subsection{\Cref{thm:ndf-inc-ubound-nobias} (Main Proof)}

{\bf \Cref{thm:ndf-inc-ubound-nobias}.}
{\it
  If $\tau_l$ either has bias parameters or is well-structured biaffine for all $l\in[L]$, then
  \begin{align*}
    & \frac{|\ndfM{z_L} \cup \incM{z_L}|}{|\Omega|}
    \leq \frac{1}{|\bbM|} {\sum_{(l,i) \in \Idx}}
    \Big| \ndf{\sigma_{l,i}} \cup \big(\bdz{\sigma_{l,i}} \cap S_{l+1}\big) \Big|,
  \end{align*}
  where $S_l \subseteq \bbR$ is defined by 
  \begin{align*}
    S_l \defeq
    \begin{cases}
      \emptyset~~ & \text{if $l > L$ or $\tau_l$ has bias parameters}
      \\ \bbR~~ & \text{otherwise}.
    \end{cases}
  \end{align*}
}
\begin{proof} 
  Observe that
  \begin{align}
    \label{eq:ndf-inc-ubound-nobias-1}
    \ndfM{z_L} \cup \incM{z_L}
    &\subseteq 
    \bigcup_{(l,i) \in \Idx} \, \bigcup_{c \in A_{l,i}}  B_{l,i}(c),
    \qquad\quad
    {|B_{l,i}(c)|} \leq {|\bbM|^{W-1}},
  \end{align}
  where $A_{l,i} \subseteq \bbR$ and $B_{l,i}(c) \subseteq \Omega$ are defined as in \Cref{lem:pbd-bound-nobias}.
  Here the first equation is by \Cref{lem:pbd-bound-nobias}
  and the second equation is by \Cref{lem:sol-count-bias,lem:sol-count-nobias},
  where these lemmas are applicable by the definition of $B_{l,i}(c)$ and
  because $\tau_l$ either has bias parameters or is well-structured biaffine (both by assumption).
  Observe further that
  \begin{align}
    \label{eq:ndf-inc-ubound-nobias-2}
    A_{l,i} = \ndf{\sigma_{l,i}} \cup (\bdz{\sigma_{l,i}} \cap S_{l+1})
  \end{align}
  by the definition of $A_{l,i}$ and $S_l$,
  where $S_l$ is defined in the statement of this theorem.
  Combining the above observations, we obtain the conclusion:
  \begingroup
  \begin{align*}
    \frac{|\ndfM{z_L} \cup \incM{z_L}|}{|\Omega|}
    &\leq \sum_{(l,i) \in \Idx} \, \sum_{c \in A_{l,i}} \frac{|B_{l,i}(c)|}{|\Omega|}
    \leq \sum_{(l,i) \in \Idx} \big|\ndf{\sigma_{l,i}} \cap (\bdz{\sigma_{l,i}} \cap S_{l+1}) \big| \cdot \frac{|\bbM|^{W-1}}{|\bbM|^W}, 
  \end{align*}
  \endgroup
  where the first inequality is by \Cref{eq:ndf-inc-ubound-nobias-1}
  and the second inequality is by \Cref{eq:ndf-inc-ubound-nobias-1,eq:ndf-inc-ubound-nobias-2}.
\end{proof}


\clearpage
\section{Upper Bounds on $|\incM{z_L}|$}
\label{sec:pf-inc}

In the rest of the appendix, we use the following notation.
For a vector $v \in \bbR^n$, $v_{a:b}$ denotes the vector $(v_a, \ldots, v_b)$.
For a matrix $M \in \bbR^{n \times m}$,
$M_{a:b,\,c:d}$ denotes the matrix $(M_{i,j})_{a \leq i \leq b,\, c \leq j \leq d}$;
$M_{a:b,\,c}$ denotes the vector $(M_{a,c}, \ldots, M_{b,c})$;
and $M_{*,\,c:d}$ denotes $M_{1:n,\,c:d}$ (and similarly for $M_{a:b,\,*}$ and $M_{*,\,c}$).

\subsection{Lemmas (Basic)}

\begin{lemma}
  \label{lem:seq-img-same-intvl}
  Let $n \in \bbN$.
  For each $j \in [n]$, let $f_j : \bbR \to \bbR$ and $\mcA_j$ be a finite cover of $\bbR$
  (i.e., $\smash{\bigcup_{A \in \mcA_j}} A = \bbR$ and $|\mcA_j| < \infty$).
  Consider $x \in \bbR$.
  Then, there is $\{x_i\}_{i \in \bbN} \subseteq (x, \infty)$ such that
  $\smash{\lim_{i \to \infty}} x_i = x$ and for all $j \in [n]$,
  \begin{align*}
    \{ f_j( x_i ) \mid i \in \bbN \} \subseteq A \qquad\text{for some $A \in \mcA_j$}.
  \end{align*}
  Further, there is $\{x'_i\}_{i \in \bbN} \subseteq (-\infty, x)$ that satisfies the same conditions stated above.
\end{lemma}
\begin{proof}
  Consider $f_j$, $\mcA_j$, and $x$ stated above ($j \in [n]$).
  Let $x_i \defeq x + 1/i$ for $i \in \bbN$.
  Then,
  \begin{align}
    \label{eq:seq-img-same-intvl-1}
    \{x_i\}_{i \in \bbN} \subseteq (x, \infty), \qquad\quad {\lim_{i \to \infty}} x_i = x.
  \end{align}
  For each $(i,j) \in \bbN \times [n]$, let $A_{i,j} \in \mcA_j$ be the set satisfying $f_j(x_i) \in A_{i,j}$,
  and $A_i \defeq (A_{i,1}, \ldots, A_{i,n}) \in \mcA_1 \times \cdots \times \mcA_n$,
  where $A_{i,j}$ always exists since $\mcA_j$ is a cover of $\bbR$.
  Observe that since $|\mcA_1 \times \cdots \times \mcA_n| < \infty$ (by $|\mcA_j| < \infty$ and $n < \infty$) and $|\bbN| = \infty$,
  there must exist $\{k_i\}_{i \in \bbN} \subseteq \bbN$ such that
  \begin{align}
    \label{eq:seq-img-same-intvl-2}
    k_1 < k_2 < \cdots, \qquad\quad A_{k_1} = A_{k_2} = \cdots.
  \end{align}
  We claim that $\smash{\{x_{k_i}\}_{i \in \bbN}}$ satisfies the desired conditions.
  First, by \Cref{eq:seq-img-same-intvl-1} and $\smash{\lim_{i \to \infty}} k_i = \infty$ (due to \Cref{eq:seq-img-same-intvl-2}),
  $\{x_{k_i}\}_{i \in \bbN} \subseteq (x,\infty)$ and $\smash{\lim_{i \to \infty}} x_{k_i} = x$.
  Second, by \Cref{eq:seq-img-same-intvl-2}, $\{f_j(x_{k_i}) \mid i \in \bbN\} \subseteq A_{k_1, j}$ for all $j \in [n]$.
  Hence, the claim holds and this concludes the proof.
\end{proof}

\begin{lemma}
  \label{lem:f-same-df-same}
  Let $f, g : \bbR \to \bbR$ and $x \in \bbR$.
  Suppose that $f$ and $g$ are differentiable at $x$,
  and there is $\{x_i\}_{i \in \bbN} \subseteq \bbR \setminus \{x\}$ such that $\smash{\lim_{i \to \infty}} x_i = x$
  and $f(x_i)=g(x_i)$ for all $i \in \bbN$.
  Then,
  \begin{align*}
    \DF{f}(x) = \DF{g}(x).
  \end{align*}
\end{lemma}
\begin{proof}
  Consider $f,g : \bbR \to \bbR$, $x \in \bbR$, and $\{x_i\}_{i \in \bbN} \subseteq \bbR \setminus \{x\}$ stated above.
  Then,
  \[f(x) = {\lim_{i \to \infty}} f(x_i) = {\lim_{i \to \infty}} g(x_i) = g(x),\]
  where the first and third equalities are by that $f$ and $g$ are continuous at $x$ (as they are differentiable at $x$) and $x_i \to x$,
  and the second equality by that $f(x_i)=g(x_i)$ for all $i \in \bbN$.
  Using this, we obtain
  \begin{align*}
    \DF{f}(x)
    &= \lim_{i \to \infty} \frac{f(x_i) - f(x)}{x_i - x}
    = \lim_{i \to \infty} \frac{g(x_i) - g(x)}{x_i - x}
    = \DF{g}(x),
  \end{align*}
  where the first and third equalities are by that $f$ and $g$ are differentiable at $x$, $x_i \to x$, and $x_i \neq x$ for all $i \in \bbN$,
  and the second equality by that $f(x_i)=g(x_i)$ for all $i \in \bbN$.
  This completes the proof.
\end{proof}

\subsection{Lemmas (Technical: Part 1)}

\begin{definition*}
  Let $\gamma \in \Gamma$.
  Define $\clR^\gamma \subseteq \bbR^W$ as
  \begin{align*}
    \clR^\gamma & \defeq \bigcap_{(l,i) \in \Idx} \{  w \in \bbR^W \;|\; y_{l,i}(w) \in \cl(\mcI_{l,i}^{\gamma({l,i})}) \}.
  \end{align*}
  Note that when defining $\mcR^\gamma$ in \Cref{def:ryz-gamma}, we used $\mcI_{l,i}^{\gamma({l,i})}$ instead of $\cl(\mcI_{l,i}^{\gamma({l,i})})$.
\end{definition*}

\begin{definition*}
  For $\gamma \in \Gamma$ and $l \in [L]$, define
  \begin{align*}
    \ext{\tau}_l &: \bbR^{N_{l-1}}\times \bbR^{W_l + W_{l+1} + \cdots + W_L} \to \bbR^{N_l}\times\bbR^{W_{l+1} + \cdots + W_L},
    \\
    \ext{\sigma}_l, \ext{\sigma}_l^\gamma &: \bbR^{N_l}\times\bbR^{W_{l+1} + \cdots + W_L} \to \bbR^{N_l}\times\bbR^{W_{l+1} + \cdots + W_L},
    \\
    \ext{z}_l, \ext{z}_l^\gamma &:  \bbR^{N_{l-1}}\times\bbR^{W_l + W_{l+1} + \cdots + W_L} \to \bbR^{N_L}
  \end{align*}
  as follows:
  \begin{align*}
    \ext{\tau}_l(x, u) & \defeq \rlap{$\big( \tau_l(x, u_1, \ldots, u_{W_l}), u_{W_l+1}, \ldots, u_{W_l + W_{l+1}+ \cdots +W_L} \big),$}
    \\
    \ext{\sigma}_l(x, u) & \defeq ( \sigma_l(x), u ),
    &
    \ext{\sigma}_l^\gamma(x, u) & \defeq ( \sigma_l^\gamma(x), u ),
    \\
    \ext{z}_l(x,u) & \defeq (\ext{z}_{l+1} \circ \ext{\sigma}_{l} \circ \ext{\tau}_{l})(x,u)
    &
    \ext{z}_l^\gamma(x,u) & \defeq (\ext{z}_{l+1}^\gamma \circ \ext{\sigma}_{l}^\gamma \circ \ext{\tau}_{l})(x,u)
  \end{align*}
  where $\ext{z}_{L+1}, \smash{\ext{z}_{L+1}^\gamma} : \bbR^{N_L} \to \bbR^{N_L}$ are defined as the identity function.
\end{definition*}
\begin{lemma}
  \label{lem:extz-good}
  For all $l \in [L]$ and $\gamma \in \Gamma$,
  $\ext{z}_l$ is continuous and $\smash{\ext{z}_l^\gamma}$ is differentiable.
\end{lemma}
\begin{proof}
  Since the proof is similar to that of \Cref{lem:yz-yz-gamma-good}, we omit it.
\end{proof}

\begin{lemma}
  \label{lem:exttau-extsigma-one-step}
  Let $\gamma \in \Gamma$, $w = (w_1, \ldots, w_L) \in \smash{\clR^\gamma}$, $l \in [L]$, and $x = (z_{l-1}(w), w_l, \ldots, w_L)$.
  Then,
  \begin{align*}
    \ext{\tau}_l(x)
    & = \big(y_l(w), w_{l+1}, \ldots, w_{L}\big),
    \qquad\quad
    (\ext{\sigma}_l^\gamma \circ \ext{\tau}_l)(x)
    = \big(z_l(w), w_{l+1}, \ldots, w_{L}\big).
  \end{align*}
\end{lemma}
\begin{proof}
  By the definition of $\ext{\tau}_l$ and $\smash{\ext{\sigma}_l^\gamma}$, we get the conclusion:
  \begin{align*}
    \ext{\tau}_l(x)
    &= \big(\tau_l(z_{l-1}(w), w_l), w_{l+1}, \ldots, w_{L}\big)
    = \big(y_l(w), w_{l+1}, \ldots, w_{L}\big),
    \\
    (\ext{\sigma}_l^\gamma \circ \ext{\tau}_l)(x)
    &= \big(\sigma_l^\gamma(y_l(w)), w_{l+1}, \ldots, w_{L}\big)
    = \big(z_l(w), w_{l+1}, \ldots, w_{L}\big),
  \end{align*}
  where the last equality is by the observation that
  $\smash{\sigma_{l,i}^{\gamma(l,i)}}(y_{l,i}(w)) = \sigma_{l,i}(y_{l,i}(w))$ for all $i \in [N_l]$.
  Here the observation holds because 
  $\smash{\sigma_{l,i}^{\gamma(l,i)}}$ and $\sigma_{l,i}$ coincide on $\smash{\cl(\mcI_{l,i}^{\gamma(l,i)})}$
  (as they coincide on ${\mcI_{l,i}^{\gamma(l,i)}}$ and are both continuous)
  and $y_{l,i}(w) \in \smash{\cl(\mcI_{l,i}^{\gamma(l,i)})}$ (by $w \in \smash{\clR^\gamma}$).
\end{proof}

\begin{lemma}
  \label{lem:extz-z-equiv}
  Let $\gamma \in \Gamma$ and $l \in [L]$. Then, for all $w = (w_1, \ldots, w_L) \in \bbR^W$,
  \begin{align*}
    \ext{z}_l \big(z_{l-1}(w), w_l, \ldots, w_L \big) &= z_L(w),
    \qquad\quad
    \ext{z}_l^\gamma \big(z_{l-1}^\gamma(w), w_l, \ldots, w_L \big) = z_L^\gamma(w).
  \end{align*}
\end{lemma}

\begin{proof}
  Let $\gamma \in \Gamma$. The proof is by induction on $l \in [L]$ (starting from $l=L+1$).

  \paragraph{\bf Case $l = L+1$.}
  Since $\smash{\ext{z}_{L+1}}$ and $\smash{\ext{z}_{L+1}^\gamma}$ are identity functions, the desired equations clearly hold.

  \paragraph{\bf Case $l < L+1$.}
  We obtain the first desired equation as follows:
  \begin{align*}
    \ext{z}_l \big(z_{l-1}(w), w_l, \ldots, w_L\big)
    &= (\ext{z}_{l+1} \circ \ext{\sigma}_l \circ \ext{\tau}_l)\big(z_{l-1}(w), w_l, \ldots, w_L\big)
    \\
    &= (\ext{z}_{l+1} \circ \ext{\sigma}_l)\big( {\tau}_l(z_{l-1}(w), w_l), w_{l+1}, \ldots, w_L\big)
    \\
    &= (\ext{z}_{l+1} \circ \ext{\sigma}_l)\big(y_{l}(w), w_{l+1}, \ldots, w_L\big)
    \\
    &= \ext{z}_{l+1} \big( {\sigma}_l(y_{l}(w)), w_{l+1}, \ldots, w_L\big)
    \\
    &= \ext{z}_{l+1} \big( z_{l}(w), w_{l+1}, \ldots, w_L\big)
    \\
    &= z_L(w),
  \end{align*}
  where all but last lines use the definition of $\ext{z}_l$, $\ext{\tau}_l$, $\ext{\sigma}_l$, $y_l$, and $z_l$,
  and the last line uses induction hypothesis on $l+1$.
  We can obtain the second desired equation similarly, by using induction hypothesis on $l+1$
  and the definition of $\smash{\ext{z}_l^\gamma}$, $\ext{\tau}_l$, $\smash{\ext{\sigma}_l^\gamma}$, $\smash{y_l^\gamma}$, and~$\smash{z_l^\gamma}$.
\end{proof}

\begin{lemma}
  \label{lem:extz-extzgamma-equiv}
  Let $\gamma \in \Gamma$ and $l \in [L]$. Then, for all $w = (w_1, \ldots, w_L) \in \mcR^\gamma$,  
  \begin{align*}
    \ext{z}_l \big(z_{l-1}(w), w_l, \ldots, w_L \big) &= \ext{z}_l^\gamma \big(z_{l-1}(w), w_l, \ldots, w_L \big).
  \end{align*}
\end{lemma}
\begin{proof}
  By \Cref{lem:extz-z-equiv}, we have the conclusion as follows:
  \begin{align*}
    \ext{z}_l(z_{l-1}(w), w_l, \ldots, w_L)
    &= z_L(w)
    = z_L^\gamma(w)
    = \ext{z}_l^\gamma(z_{l-1}(w), w_l, \ldots, w_L),
  \end{align*}
  where the second equality is by \Cref{lem:r-gamma-f-df} with $w \in \mcR^\gamma$.
\end{proof}

\subsection{Lemmas (Technical: Part 2)}

\begin{definition*}
  Let $f : \bbR^n \to \bbR^m$ and $i \in [n]$.
  Define \[\PDF{f}{i} : \bbR^n \to \bbR^m \cup \{\bot\}\] be the partial derivative of $f$ with respect to its $i$-th argument,
  where $\bot$ denotes non-differentiability.
  Hence, for any $x \in \bbR^n$ and $i \in [n]$, $\DF{f}(x) \neq \bot$ implies
  $\PDF{f}{i}(x) = (\DF{f}(x))_{1:m,\, i}$.
\end{definition*}
\begin{lemma}
  \label{lem:djz-gamma-exist-gen}
  Let $l \in [L]$, $w = (w_1, \ldots, w_L) \in \bbR^W$, and $j \in [W]$ with $j > W_{<l}$, where $W_{<l} \defeq W_1 + \cdots + W_{l-1}$.
  Suppose that $\ext{z}_l$ is differentiable with respect to its $(N_{l-1}+(j-W_{<l}))$-th argument at
  $(z_{l-1}(w), w_l, \ldots, w_L)$, i.e., $\smash{\PDF{\ext{z}_l}{N_{l-1}+(j-W_{<l})}}(z_{l-1}(w), w_l, \ldots, w_L) \neq \bot$.
  Then, there are $\gamma \in \Gamma$ and $\{t_n\}_{n \in \bbN} \subseteq (v_j, \infty)$ satisfying the following conditions:
  \begin{itemize}
    \item $w \in {\clR^{\gamma}}$,
    \item $\PDF{\ext{z}_{l}}{N_{l-1}+(j-W_{<l})}\big(z_{l-1}(w), w_{l}, \ldots, w_L\big)=\PDF{\ext{z}_{l}^{\gamma}}{N_{l-1}+(j-W_{<l})}\big(z_{l-1}(w), w_{l}, \ldots, w_L\big)$,
    \item $\lim_{n \to \infty} t_n = v_j$, and
    \item $(v_1, \ldots, v_{j-1}, t_n, v_{j+1}, \ldots, v_W) \in \mcR^{\gamma} \quad\text{for all $n \in \bbN$}$,
  \end{itemize}
  where $(v_1, \ldots, v_W) \defeq w$ denotes the scalar values of $w$ (recall that $w_l \in \bbR^{W_l}$ is not scalar by definition).
  Further, there are $\gamma' \in \Gamma$ and $\{t'_n\}_{i \in \bbN} \subseteq (-\infty, v_j)$
  that satisfy the same conditions stated above.
\end{lemma}

\begin{proof}
  Consider $l \in [L]$, $w \in \bbR^W$, and $j \in [W]$ stated above.
  We show the existence of $\gamma$ and $\{t_n\}_{n \in \bbN}$,
  and will omit the proof of the existence of $\gamma'$ and $\{t'_n\}_{n \in \bbN}$
  since the proof is almost identical. 
  
  First, we show that there is $\{t_n\}_{n \in \bbN} \subseteq (v_j, \infty)$
  such that $\lim_{n \to \infty} t_n = v_j$ and
  \begin{align}
    \label{eq:yyy-1}
    \big\{ (v_1, \ldots, v_{j-1}, t_n, v_{j+1}, \ldots, v_W) \;\big|\; n \in \bbN \big\} \subseteq \mcR^{\gamma}
    \qquad\text{for some $\gamma \in \Gamma$}.
  \end{align}
  By the definition of $\mcR^\gamma$, \Cref{eq:yyy-1} is equivalent to the following: for all $(l,i) \in \Idx$,
  \begin{align*}
    \big\{ f_{l,i}(t_n) \;\big|\; n \in \bbN \big\}
    \subseteq \mcI_{l,i}^{k}
    \qquad\text{for some $k \in [K_{l,i}]$},
  \end{align*}
  where $f_{l,i} : \bbR \to \bbR$ is defined as $f_{l,i}(t) \defeq y_{l,i}(v_1, \ldots, v_{j-1}, t, v_{j+1}, \ldots, v_W)$.
  Note that \Cref{lem:seq-img-same-intvl} is applicable to $(f_{l,i}, \smash{\{\mcI_{l,i}^k\}_{k \in [K_{l,i}]}}, v_j)$,
  since $\smash{\{\mcI_{l,i}^k\}_{k \in [K_{l,i}]}}$ is a finite cover of $\bbR$ for all $(l,i)$.
  Hence, by the lemma, there is $\{t_n\}_{n \in \bbN} \subseteq (v_j, \infty)$ such that $\lim_{n \in \infty} t_n = v_j$ and \Cref{eq:yyy-1} holds
  with some $\gamma \in \Gamma$.
  
  Next, we show that $w \in \smash{\clR^{\gamma}}$.
  By the definition of $\smash{\clR^\gamma}$, this is equivalent to
  $y_{l,i}(w) \in \smash{\cl(\mcI_{l,i}^{\gamma(l,i)})}$ for all $(l,i) \in \Idx$.
  To show this, let $(l,i) \in \Idx$.
  By \Cref{eq:yyy-1} and the definition of $\mcR^\gamma$, we have
  \begin{align}
    \label{eq:yyy-2}
    \big\{ y_{l,i}(v_1, \ldots, v_{j-1}, t_n, v_{j+1}, \ldots, v_W) \;\big|\; n \in \bbN \big\} \subseteq \mcI_{l,i}^{\gamma(l,i)}.
  \end{align}
  Using this, we obtain
  \begin{align*}
    y_{l,i}(w) &= \lim_{n \to \infty} y_{l,i}(v_1, \ldots, v_{j-1}, t_n, v_{j+1}, \ldots, v_W)
    \in \cl(\mcI_{l,i}^{\gamma(l,i)}),
  \end{align*}
  where the equality is from the continuity of $y_{l,i}$ (by \Cref{lem:yz-yz-gamma-good})
  and $\lim_{n \to \infty} t_n = v_j$ (by the above), and the inclusion is by \Cref{eq:yyy-2}.
  Hence, we have $w \in \smash{\clR^{\gamma}}$ as desired.
  
  Lastly, we show that
  \(
  \smash{\PDF{\ext{z}_{l}}{N_{l-1}+(j-W_{<l})}}(z_{l-1}(w), w_{l}, \ldots, w_L)
  = \smash{\PDF{\ext{z}_{l}^{\gamma}}{N_{l-1}+(j-W_{<l})}}(z_{l-1}(w), w_{l}, \ldots, w_L).
  \)
  To do so, define $g, g^\gamma: \bbR \to \bbR^{N_L}$ as:
  \begin{align*}
    g(t) &\defeq \ext{z}_l\big(
    z_{l-1}(w), v_{W_{<l}+1}, \ldots, v_{j-1}, t, v_{j+1}, \ldots, v_W
    \big),
    \\
    g^\gamma(t) &\defeq \ext{z}_l^\gamma\big(
    z_{l-1}(w), v_{W_{<l}+1}, \ldots, v_{j-1}, t, v_{j+1}, \ldots, v_W
    \big).
  \end{align*}
  Using them, we obtain the desired equation as follows:
  \begin{align*}
    \PDF{\ext{z}_l}{N_{l-1}+(j-W_{<l})}\big(z_{l-1}(w), w_{l}, \ldots, w_L\big)
    &= \DF{g}(v_j)
    = \DF{g^\gamma}(v_j) = \PDF{\ext{z}_l^\gamma}{N_{l-1}+(j-W_{<l})}\big(z_{l-1}(w), w_{l}, \ldots, w_L\big),
  \end{align*}
  where the first and third equalities are by the definition of partial derivatives,
  and the second equality comes from \Cref{lem:f-same-df-same} applied to $(g, g^{\gamma}, v_j, \{t_n\}_{n \in \bbN})$.
  Here \Cref{lem:f-same-df-same} is applicable due to the following:
  $g$ is differentiable at $v_j$ (as $\DF{g}(v_j) = \smash{\PDF{\ext{z}_l}{N_{l-1}+(j-W_{<l})}}(z_{l-1}(w), w_l, \ldots, w_L) \neq \bot$ by assumption);
  $g^{\gamma}$ is differentiable (as $\smash{\ext{z}_l^\gamma}$ is differentiable by \Cref{lem:extz-good});
  $\lim_{n \to \infty} t_n = v_j$ with $t_n \neq v_j$ (by the above);
  and $g(t_n) = g^{\gamma}(t_n)$ for all $n \in \bbN$ because
  \begin{align*}
    g(t_n)
    &= \ext{z}_l\big( z_{l-1}(w), v_{W_{<l}+1}, \ldots, v_{j-1}, t_n, v_{j+1}, \ldots, v_W  \big)
    \\
    &= \ext{z}_l\big( z_{l-1}(v_1, \ldots, v_{j-1}, t_n, v_{j+1}, \ldots, v_W), v_{W_{<l}+1}, \ldots, v_{j-1}, t_n, v_{j+1}, \ldots, v_W  \big)
    \\
    &= \ext{z}_l^\gamma\big( z_{l-1}(v_1, \ldots, v_{j-1}, t_n, v_{j+1}, \ldots, v_W), v_{W_{<l}+1}, \ldots, v_{j-1}, t_n, v_{j+1}, \ldots, v_W  \big)
    \\
    &= \ext{z}_l^\gamma\big( z_{l-1}(w), v_{W_{<l}+1}, \ldots, v_{j-1}, t_n, v_{j+1}, \ldots, v_W  \big)
    = g^{\gamma}(t_n),
  \end{align*}
  where the second and fourth lines use that $z_{l-1}$ depends only on its first $W_{<l}$ arguments and $j > W_{<l}$,
  and the third line is by \Cref{lem:extz-extzgamma-equiv} and \Cref{eq:yyy-1}.
  This completes the proof.
\end{proof}

\begin{lemma}
  \label{lem:djz-gamma-zero}
  Let $w = (w_1, \ldots, w_L) \in \bbR^W$ and $(l,i) \in \Idx$.
  Suppose the following hold: $z_L$ is differentiable at $w$;
  $\tau_l$ has bias parameters; $\sigma_{l,i}$ is not differentiable at $y_{l,i}(w)$;
  and for all $\gamma_1, \gamma_2 \in \Gamma$ with $w \in \smash{\clR^{\gamma_1}} \cap \smash{\clR^{\gamma_2}}$,
  \(
  \smash{\PDF{\ext{z}_{l+1}^{\gamma_1}}{i}}(z_l(w), w_{l+1}, \ldots, w_L)
  = \smash{\PDF{\ext{z}_{l+1}^{\gamma_2}}{i}}(z_l(w), w_{l+1}, \ldots, w_L).
  \)
  Then, for all $\gamma \in \Gamma$ with $w \in \clR^{\gamma}$,
  \begin{align*}
    \PDF{\ext{z}_{l+1}^{\gamma}}{i}\big(z_l(w), w_{l+1}, \ldots, w_L\big) = (0, \ldots, 0).
  \end{align*}
\end{lemma}
\begin{proof}
  Consider $w \in \bbR^W$ and $(l,i) \in \Idx$ satisfying the conditions in the lemma.
  First, we show that \[ \PDF{\ext{z}_l^\gamma}{N_{l-1}+(W_l-N_l+i)}(z_{l-1}(w), w_l, \ldots, w_L) =
    \PDF{\ext{z}_{l+1}^\gamma}{i}(z_l(w), w_{l+1}, \ldots, w_L)
    \cdot \DF{\sigma_{l,i}^{\gamma(l,i)}}(y_{l,i}(w)),
    \phantom{\Big)_{*}} \]
  for any $\gamma \in \Gamma$ with $w \in \smash{\clR^\gamma}$.
  To this end, we derive two derivatives: $$\big(\DF{\ext{\tau}_l}(x)\big)_{*,\, N_{l-1}+(W_l-N_l+i)}\quad\text{and}
    \quad
    \big(\DF{\ext{\sigma}_l^\gamma}(x')\big)_{*,\,i}.$$
  Since $\tau_l$ has bias parameters (by assumption), and by the definitions of $\ext{\tau}_l$ and $\smash{\ext{\sigma}_l^\gamma}$,
  we have the following: for all $\gamma \in \Gamma$ and $i' \in [N_l+W_{l+1}+\cdots+W_L]$,
  there is $\smash{\tau'_{l,i'}} : \bbR^{N_{l-1} + (W_l - N_l)} \to \bbR$ such that
  \begin{alignat*}{4}
    \ext{\tau}_{l,i'}(x)
    &=
    \begin{cases}
      \tau'_{l,i'}(x_1, \ldots, x_{N_{l-1}+(W_l-N_l)}) + x_{N_{l-1}+(W_l-N_l+i')}
      \\
      x_{N_{l-1}+(W_l-N_l+i')}
    \end{cases}
    &&
    \begin{array}{l}
      \text{if $i' \leq N_l$} \vphantom{\tau'_{l,i'}}
      \\[2pt]
      \text{if $i' > N_l$}, \vphantom{x_{N_{l-1}+(W_l-N_l+i')}}
    \end{array}
    \\
    \ext{\sigma}_{l,i'}^\gamma(x')
    &=
    \begin{cases}
      \sigma_{l,i'}^{\gamma(l,i')}(x'_{i'})
      \\
      x'_{i'}
    \end{cases}
    &&
    \begin{array}{l}
      \text{if $i' \leq N_l$} \vphantom{\sigma_{l,i'}^{\gamma(l,i')}(x'_{i'})}
      \\[2pt]
      \text{if $i' > N_l$}, \vphantom{x'_{i'}}
    \end{array}
  \end{alignat*}
  for all $x \in \bbR^{N_{l-1} + W_l + \cdots + W_L}$ and $x' \in \bbR^{N_{l} + W_{l+1} + \cdots + W_L}$.
  From this and $i\in[N_l]$, we obtain two derivatives: 
  \begin{align}
    \label{eq:tmp3-1}
    \big(\DF{\ext{\tau}_l}(x)\big)_{*,\, N_{l-1}+(W_l-N_l+i)} = e_i,
    \qquad
    \big(\DF{\ext{\sigma}_l^\gamma}(x')\big)_{*,\,i} = e_i \cdot \DF{\sigma_{l,i}^{\gamma(l,i)}}(x'_i),
  \end{align}
  where $e_i \in \bbR^{N_l + W_{l+1} + \cdots + W_L}$ denotes the standard unit vector with $1$ at the $i$-th coordinate,
  $\smash{\DF{\sigma_{l,i}^{\gamma(l,i)}}}(x'_i)$ is considered as a scalar value,
  and both equalities are by $i \in [N_l]$.
  Using this, we obtain the following equation for $x \defeq (z_{l-1}(w), w_l, \ldots, w_L)$
  and for any $\gamma \in \Gamma$ with $w \in \smash{\clR^\gamma}$:
  \begin{align}
    \nonumber
    \PDF{\ext{z}_l^\gamma}{N_{l-1}+(W_l-N_l+i)}(x)
    &=
    \Big(\DF{\ext{z}_l^\gamma}(x) \Big)_{*,\, N_{l-1}+(W_l-N_l+i)}
    \\ \nonumber
    &=
    \Big( \DF{\ext{z}_{l+1}^\gamma}((\ext{\sigma}_l^\gamma \circ \ext{\tau}_l)(x))
    \cdot \DF{\ext{\sigma}_l^\gamma}(\ext{\tau}_l(x)) \cdot \DF{\ext{\tau}_l}(x)\Big)_{*,\, N_{l-1}+(W_l-N_l+i)}
    \\ \nonumber
    &=
    \DF{\ext{z}_{l+1}^\gamma}((\ext{\sigma}_l^\gamma \circ \ext{\tau}_l)(x))
    \cdot \DF{\ext{\sigma}_l^\gamma}(\ext{\tau}_l(x)) \cdot \Big( \DF{\ext{\tau}_l}(x)\Big)_{*,\, N_{l-1}+(W_l-N_l+i)}
    \\ \nonumber
    &=
    \DF{\ext{z}_{l+1}^\gamma}((\ext{\sigma}_l^\gamma \circ \ext{\tau}_l)(x))
    \cdot \Big( \DF{\ext{\sigma}_l^\gamma}(\ext{\tau}_l(x)) \Big)_{*,\,i}
    \\ \nonumber
    &=
    \Big( \DF{\ext{z}_{l+1}^\gamma}((\ext{\sigma}_l^\gamma \circ \ext{\tau}_l)(x)) \Big)_{*,\,i}
    \cdot \DF{\sigma_{l,i}^{\gamma(l,i)}}(\ext{\tau}_{l,i}(x))
    \\ \label{eq:tmp3-2}
    &=
    \PDF{\ext{z}_{l+1}^\gamma}{i}(z_l(w), w_{l+1}, \ldots, w_L)
    \cdot \DF{\sigma_{l,i}^{\gamma(l,i)}}(y_{l,i}(w)),
    \phantom{\Big)_{*}}
  \end{align}
  where the first two lines use $\smash{\ext{z}_l^\gamma} = \smash{\ext{z}_{l+1}^\gamma} \circ \smash{\ext{\sigma}_l^\gamma} \circ \ext{\tau}_l$
  and that $\smash{\ext{z}_l^\gamma}$, $\smash{\ext{z}_{l+1}^\gamma}$, $\smash{\ext{\sigma}_l^\gamma}$, and $\ext{\tau}_l$
  are differentiable (by \Cref{lem:extz-good}),
  the fourth and fifth lines use \Cref{eq:tmp3-1},
  and the last line uses \Cref{lem:exttau-extsigma-one-step} with $w \in \smash{\clR^\gamma}$.
  
  Next, we derive a sufficient condition for the conclusion
  by using \Cref{eq:tmp3-2} and applying \Cref{lem:djz-gamma-exist-gen} to $(l,w,j)$ with $j \defeq W_{<l} + (W_l - N_l + i)$,
  where $W_{<l} \defeq W_1 + \cdots + W_{l-1}$.
  Note that the lemma is applicable here due to the following:
  $W_{<l} < j \leq W$, because $W_l \geq N_l$ (as $\tau_l$ has bias parameters by assumption)
  and $1 \leq i \leq N_l$ (as $(l,i) \in \Idx$);
  and $\ext{z}_l$ is differentiable with respect to its $(N_{l-1}+(W_l-N_l+i))$-th argument at $(z_{l-1}(w), w_l \ldots, w_L)$,
  because
  \[
  \smash{\PDF{\ext{z}_l}{N_{l-1}+(W_l-N_l+i)}}(z_{l-1}(w), w_l, \ldots, w_L) = \smash{\PDF{z_L}{W_{<l}+(W_l-N_l+i)}}(w) \neq \bot
  \]
  where the equality follows from that
  $z_L(w') = \ext{z}_l(z_{l-1}(\smash{w'_1}, \ldots, \smash{w'_{l-1}}, 0, \ldots, 0), \smash{w'_l}, \ldots, \smash{w'_L})$ for all $w' \in \bbR^W$
  by \Cref{lem:extz-z-equiv},
  and the inequality from that $z_L$ is differentiable at $w$ (by assumption).
  Let $(v_1, \ldots, v_W) \defeq w$ be the scalar values of $w$.
  By applying \Cref{lem:djz-gamma-exist-gen} to $(l,w,j)$, it holds that
  there are $\gamma_+, \gamma_- \in \Gamma$, $\{t^+_n\}_{n \in \bbN} \subseteq (v_j, \infty)$, and
  $\{t^-_n\}_{n \in \bbN} \subseteq (-\infty, v_j)$ such that $w \in \smash{\clR^{\gamma_+}} \cap \smash{\clR^{\gamma_-}}$ and
  \begin{gather}
    \nonumber
    \PDF{\ext{z}_{l}^{\gamma_+}}{N_{l-1}+(W_l-N_l+i)}(x)
    = \PDF{\ext{z}_{l}^{\gamma_-}}{N_{l-1}+(W_l-N_l+i)}(x),
    \\[0.25em]
    \label{eq:tmp3-3}
    \{(v_1, \ldots, v_{j-1}, t^+_n, v_{j+1}, \ldots, v_W)\}_{n \in \bbN}
    \subseteq \mcR^{\gamma_+},
    \qquad\quad
    \lim_{n \to \infty} t^+_n = v_j,
    \\
    \label{eq:tmp3-4}
    \{(v_1, \ldots, v_{j-1}, t^-_n, v_{j+1}, \ldots, v_W)\}_{n \in \bbN}
    \subseteq \mcR^{\gamma_-},
    \qquad\quad
    \lim_{n \to \infty} t^-_n = v_j,
  \end{gather}
  where $x \defeq (z_{l-1}(w), w_l, \ldots, w_L)$.
  By the first line and \Cref{eq:tmp3-2} with $w \in \smash{\clR^{\gamma_+}} \cap \smash{\clR^{\gamma_-}}$, we have
  \begin{align*}
    & \PDF{\ext{z}_{l+1}^{\gamma_+}}{i}(x')
    \cdot \DF{\sigma_{l,i}^{\gamma_+(l,i)}}(y_{l,i}(w))
    = \PDF{\ext{z}_{l+1}^{\gamma_-}}{i}(x')
    \cdot \DF{\sigma_{l,i}^{\gamma_-(l,i)}}(y_{l,i}(w)),
  \end{align*}
  where $x' \defeq (z_l(w), w_{l+1}, \ldots, w_L)$.
  From this, and since $\smash{\PDF{\ext{z}_{l+1}^{\gamma}}{i}}(x')$ is the same for all $\gamma \in \Gamma$
  with $w \in \smash{\clR^{\gamma}}$ (by assumption),
  we immediately obtain the conclusion (i.e.,
  $\smash{\PDF{\ext{z}_{l+1}^{\gamma}}{i}}(x') = (0, \ldots, 0)$ for all  $\gamma \in \Gamma$ with $w \in \smash{\clR^{\gamma}}$)
  if the following holds:
  \begin{align}
    \label{eq:tmp3-5}
    \DF{\sigma_{l,i}^{\gamma_+(l,i)}}(y_{l,i}(w)) \neq \DF{\sigma_{l,i}^{\gamma_-(l,i)}}(y_{l,i}(w)).
  \end{align}
  Hence, to prove the conclusion, it suffices to show \Cref{eq:tmp3-5}.

  Finally, we prove \Cref{eq:tmp3-5} in two steps.
  We first show that there are $\delta^+, \delta^- > 0$ such that
  \begin{align}
    \label{eq:tmp3-6}
    \big(y_{l,i}(w), y_{l,i}(w)+\delta^+\big) \subseteq \mcI_{l,i}^{\gamma_+(l,i)},
    \qquad
    \big(y_{l,i}(w)-\delta^-, y_{l,i}(w)\big) \subseteq \mcI_{l,i}^{\gamma_+(l,i)}.
  \end{align}
  Fix $j \defeq W_{<l} + (W_l - N_l + i)$ and $(v_1, \ldots, v_W) \defeq w$ as above.
  Observe that we have
  \begin{align*}
    y_{l,i}(v_1, \ldots, v_{j-1}, t^+_n, v_{j+1}, \ldots, v_W)
    &\in \mcI_{l,i}^{\gamma_+(l,i)}  \text{ for all $n \in \bbN$},
    \\
    y_{l,i}(v_1, \ldots, v_{j-1}, t^+_n, v_{j+1}, \ldots, v_W)
    &> y_{l,i}(v_1, \ldots, v_W) = y_{l,i}(w) \text{ for all $n \in \bbN$},
    \\
    \lim_{n \to \infty} y_{l,i}(v_1, \ldots, v_{j-1}, t^+_n, v_{j+1}, \ldots, v_W)
    &= y_{l,i}(v_1, \ldots, v_W) = y_{l,i}(w),
  \end{align*}
  where the first line uses \Cref{eq:tmp3-3},
  the third line uses \Cref{eq:tmp3-3} and that $y_{l,i}$ is continuous (by \Cref{lem:yz-yz-gamma-good}),
  and the second line uses the following and that $t^+_n > v_j$ for all $n \in \bbN$:
  for all $t \in \bbR$,
  \begin{align*}
    y_{l,i}(v_1, \ldots, v_{j-1}, t, v_{j+1}, \ldots, v_W)
    &= \tau'_{l,i}(z_{l-1}(w), v_{W_{<l}+1}, \ldots, v_{W_{<l}+(W_l-N_l)}) + t,
  \end{align*}
  which holds since $z_{l-1}$ depends only on its first $W_{<l}$ arguments,
  $\tau_l$ has bias parameters, and $j = W_{<l} + (W_l-N_l+i) > W_{<l}$.
  By these results, and since $\smash{\mcI_{l,i}^{\gamma_+(l,i)}}$ is an interval,
  there is $\delta^+ > 0$ satisfying \Cref{eq:tmp3-6};
  similarly, there is $\delta^->0$ satisfying \Cref{eq:tmp3-6},
  due to \Cref{eq:tmp3-4} and $t^-_n < v_j$ for all $n$.
  
  We next show that \Cref{eq:tmp3-5} indeed holds.
  By \Cref{eq:tmp3-6} and $\sigma_{l,i} = \smash{\sigma_{l,i}^{k}}$ on $\smash{\mcI_{l,i}^{k}}$ for all $k$,
  we have
  \begin{align*}
    \sigma_{l,i} = {\sigma_{l,i}^{\gamma_+(l,i)}} \text{ on } \big[ y_{l,i}(w), y_{l,i}(w)+\delta^+ \big),
    \qquad
    \sigma_{l,i} = {\sigma_{l,i}^{\gamma_-(l,i)}} \text{ on } \big( y_{l,i}(w)-\delta^-, y_{l,i}(w) \big],
  \end{align*}
  where the inclusion of $y_{l,i}(w)$ is by that $\sigma_{l,i}$ and $\smash{\sigma_{l,i}^k}$ are continuous for all $k$.
  From this, we have
  \begin{align*}
    \DF{\sigma_{l,i}^{\gamma_+(l,i)}}\big(y_{l,i}(w)\big)
    &=
    \lim_{h \to 0^+} \frac{1}{h}\Big(\sigma_{l,i}(y_{l,i}(w) + h) - \sigma_{l,i}(y_{l,i}(w))\Big),
    \\
    \DF{\sigma_{l,i}^{\gamma_-(l,i)}}\big(y_{l,i}(w)\big)
    &=
    \lim_{h \to 0^-} \frac{1}{h}\Big(\sigma_{l,i}(y_{l,i}(w) + h) - \sigma_{l,i}(y_{l,i}(w))\Big).
  \end{align*}
  Suppose here that \Cref{eq:tmp3-5} does not hold, i.e.,
  $\smash{\DF{\sigma_{l,i}^{\gamma_+(l,i)}}}(y_{l,i}(w)) = \smash{\DF{\sigma_{l,i}^{\gamma_-(l,i)}}}(y_{l,i}(w))$.
  Then, 
  \begin{align*}
    \lim_{h \to 0} \frac{1}{h}\Big(\sigma_{l,i}(y_{l,i}(w) + h) - \sigma_{l,i}(y_{l,i}(w))\Big)
    = \DF{\sigma_{l,i}^{\gamma_+(l,i)}}\big(y_{l,i}(w)\big) \neq \bot,
  \end{align*}
  where the inequality is by that $\smash{\sigma_{l,i}^{\gamma_+(l,i)}}$ is differentiable.
  This implies $\DF{\sigma_{l,i}}\big(y_{l,i}(w)\big) \neq \bot$,
  which contradicts to that $\sigma_{l,i}$ is non-differentiable at $y_{l,i}(w)$ (by assumption).
  Hence, \Cref{eq:tmp3-5} should hold.
\end{proof}

\subsection{\Cref{thm:inc-zero-bias} (Main Lemmas)}

\begin{lemma}
  \label{lem:djz-gamma-exist-spc}
  Let $w \in \bbR^W$ and $j \in [W]$.
  Suppose that $z_L$ is differentiable with respect to its $j$-th argument at $w$
  (i.e., $\PDF{z_L}{j}(w) \neq \bot$).
  Then, there is $\gamma \in \Gamma$ such that $w \in \clR^\gamma$ and
  \begin{align*}
    \PDF{z_L}{j}(w) = \PDF{z_L^\gamma}{j}(w).
  \end{align*}
\end{lemma}
\begin{proof}
  Consider $w \in \bbR^W$ and $j \in [W]$ stated above.
  First, by \Cref{lem:extz-z-equiv}, and since $z_0 = \smash{z_0^\gamma}$ is a constant function,
  we have $z_L(w') = \ext{z}_1(z_0(0, \ldots, 0), w')$
  and $\smash{z_L^\gamma}(w') = \ext{z}_1^\gamma(z_0(0, \ldots, 0), w')$ for all $w' \in \bbR^W$ and $\gamma \in \Gamma$.
  From this, we have
  \begin{align*}
    \PDF{z_L}{j}(w) &= \PDF{\ext{z}_1}{N_0 + j}\big(z_0(0, \ldots, 0), w\big) = \PDF{\ext{z}_1}{N_0 + j}\big(z_0(w), w\big),
    \\
    \PDF{z_L^\gamma}{j}(w) &= \PDF{\ext{z}_1^\gamma}{N_0 + j}\big(z_0(0, \ldots, 0), w\big) = \PDF{\ext{z}_1^\gamma}{N_0 + j}\big(z_0(w), w\big)
    \quad\text{for all $\gamma \in \Gamma$,}
  \end{align*}
  where the second and fourth equalities follow from that $z_0$ is a constant function.
  Second, by \Cref{lem:djz-gamma-exist-gen} applied to $(l=1, w, j)$, 
  there is $\gamma \in \Gamma$ such that
  \[
  w \in \smash{\clR^{\gamma}},\qquad
  \PDF{\ext{z}_1}{N_0+j}\big(z_0(w), w\big) = \smash{\PDF{\ext{z}_1^{\gamma}}{N_0+j}}\big(z_0(w), w\big).
  \]
  Here \Cref{lem:djz-gamma-exist-gen} is applicable,
  because $\PDF{\ext{z}_1}{N_0 + j}(z_0(w), w) = \PDF{z_L}{j}(w) \neq \bot$ (by the above and by assumption).
  From these results, there is $\gamma \in \Gamma$
  such that $w \in \smash{\clR^{\gamma}}$ and $\PDF{z_L}{j}(w) = \smash{\PDF{z_L^\gamma}{j}}(w)$.
\end{proof}

\begin{lemma}
  \label{lem:djz-gamma-indep}
  Let $w \in \bbR^W$.
  Suppose that the following hold:
  \begin{itemize}
  \item $z_L$ is differentiable at $w$.
  \item For all $l \in [L]$, if $\tau_l$ does not have bias parameters,
  then $\sigma_{l,i}$ is differentiable at $y_{l,i}(w)$ for all $i \in [N_l]$.
  \end{itemize}
  Then, for all $ \gamma_1, \gamma_2 \in \Gamma$ with $w \in \smash{\clR^{\gamma_1}} \cap \smash{\clR^{\gamma_2}}$,
  \begin{align*}
    \DF{z_L^{\gamma_1}}(w) = \DF{z_L^{\gamma_2}}(w).
  \end{align*}
\end{lemma}
\begin{proof}
  Let $w \in \bbR^W$.
  Consider the following claim:
  for all $l \in [L+1]$ and $\gamma_1, \gamma_2 \in \Gamma$, if $w \in \clR^{\gamma_1} \cap \clR^{\gamma_2}$, then
  \begin{align*}
    \DF{\ext{z}_l^{\gamma_1}}\big( z_{l-1}(w), w_l, \ldots, w_L \big)
    &=
    \DF{\ext{z}_l^{\gamma_2}}\big( z_{l-1}(w), w_l, \ldots, w_L \big).
  \end{align*}
  Note that the claim implies the conclusion:
  for any $\gamma_1, \gamma_2 \in \Gamma$ with $w \in \smash{\clR^{\gamma_1} \cap \clR^{\gamma_2}}$, 
  \begin{align*}
    \DF{z_L^{\gamma_1}}(w)
    &= \Big( \DF{\ext{z}_1^{\gamma_1}}(z_{0}(0,\ldots,0), w) \Big)_{*,\, N_0+1: N_0+W}
    = \Big( \DF{\ext{z}_1^{\gamma_2}}(z_{0}(0,\ldots,0), w) \Big)_{*,\, N_0+1: N_0+W}
    = \DF{z_L^{\gamma_2}}(w),
  \end{align*}
  where the first and third equalities follow from that
  $\smash{z_L^{\gamma}}(w') = \smash{\ext{z}_1^{\gamma}}(z_0(0,\ldots,0), w')$ for all $\gamma \in \Gamma$ and $w' \in \bbR^W$
  (by \Cref{lem:extz-z-equiv} and since $z_0^\gamma=z_0$ is a constant function),
  and $\smash{\ext{z}_1^{\gamma}}$ is differentiable for all $\gamma \in \Gamma$
  (by \Cref{lem:extz-good});
  and the second equality is by the claim for $l=1$ and that $z_0$ is a constant function.
  We prove the claim by induction on $l$ (starting from $L+1$).

  \paragraph{\bf Case $l = L+1$.}
  The claim clearly holds, since $\smash{\ext{z}_{L+1}^\gamma}$ is the identity function for all $\gamma \in \Gamma$.
  
  \paragraph{\bf Case $l < L+1$.}
  To show the claim, we first analyze the derivatives mentioned in the claim.
  Let $\gamma \in \Gamma$ with $w \in \smash{\clR^\gamma}$,
  and consider any $x \in \bbR^{N_{l-1} + W_l + \cdots + W_L}$ and $x' \in \bbR^{N_l + W_{l+1} + \cdots + W_L}$.
  Recall the definition of $\smash{\ext{z}_l^\gamma}$ and $\smash{\ext{\sigma}_l^\gamma}$:
  for all $i \in [N_l+W_{l+1}+\cdots+W_L]$,
  \begin{align*}
    \ext{z}_l^\gamma(x)
    &= (\ext{z}_{l+1}^\gamma \circ \ext{\sigma}_l^\gamma \circ \ext{\tau}_l)(x),
    \qquad\quad
    \ext{\sigma}_{l,i}^\gamma(x')
    =
    \begin{cases}
      \sigma_{l,i}^{\gamma(l,i)}(x'_{i}) & \text{if $i \leq N_l$}
      \\
      x'_{i} & \text{if $i>N_l$}.
    \end{cases}
  \end{align*}
  Since every function in the RHS of the above equation is differentiable
  (by \Cref{lem:extz-good}),
  the following hold for all $i \in [N_l+W_{l+1}+\cdots+W_L]$:
  \begin{align}
    \label{eq:tmp2-1}
    \DF{\ext{z}_l^\gamma}(x)
    &= \DF{\ext{z}_{l+1}^\gamma}\big((\ext{\sigma}_l^\gamma \circ \ext{\tau}_l)(x)\big)
    \cdot \DF{\ext{\sigma}_l^\gamma}\big(\ext{\tau}_l(x)\big) \cdot \DF{\ext{\tau}_l}(x),
    \\ \nonumber
    \big(\DF{\ext{\sigma}_l^\gamma}(x')\big)_{*,\,i}
    &=
    \begin{cases}
      e_i \cdot \DF{\sigma_{l,i}^{\gamma(l,i)}}(x'_{i}) & \text{if $i \leq N_l$}
      \\
      e_i & \text{if $i>N_l$},
    \end{cases}
  \end{align}
  where the first line uses the chain rule,
  $e_i \in \bbR^{N_l+W_{l+1}+\cdots+W_L}$ denotes the standard unit vector with $1$ at the $i$-th coordinate,
  and $\smash{\DF{\sigma_{l,i}^{\gamma(l,i)}}}(x'_{i})$ is considered as a scalar value.
  By the second line, the following holds for all $i \in [N_l + W_{l+1}+\cdots+W_L]$:
  \begin{align}
    \label{eq:tmp2-2}
    \hspace{-2em}
    \begin{aligned}
      &\Big(\DF{\ext{z}_{l+1}^\gamma}((\ext{\sigma}_l^\gamma \circ \ext{\tau}_l)(x))
      \cdot \DF{\ext{\sigma}_l^\gamma}(\ext{\tau}_l(x)) \Big)_{*,\, i}
      \!=
      \PDF{\ext{z}_{l+1}^\gamma}{i}\big((\ext{\sigma}_l^\gamma \circ \ext{\tau}_l)(x)\big)
      \cdot
      \begin{cases}
        \DF{\sigma_{l,i}^{\gamma(l,i)}}\big( \ext{\tau}_{l,i}(x) \big) & \text{\!\!if $i \leq N_l$}
        \\
        1 & \text{\!\!if $i > N_l$}.
      \end{cases}
    \end{aligned}
    \hspace{-2em}
  \end{align}
  We can further simplify the two term in the RHS when $x= (z_{l-1}(w), w_l, \ldots, w_L)$, as follows:
  \begin{align}
    \label{eq:tmp2-3}
    \begin{aligned}
      \DF{\ext{z}_{l+1}^{\gamma}}\big( (\ext{\sigma}_l^\gamma \circ \ext{\tau}_l)(x) \big)
      &= \DF{\ext{z}_{l+1}^{\gamma}} (x'), 
      \quad
      \DF{\sigma_{l,i}^{\gamma(l,i)}}\big( \ext{\tau}_{l,i}(x) \big)
      = \DF{\sigma_{l,i}^{\gamma(l,i)}}\big(y_{l,i}(w)\big) \;\text{for all $i \in [N_l]$},
    \end{aligned}
  \end{align}
  where $x' \defeq (z_{l}(w), w_{l+1}, \ldots, w_L)$
  and both equalities are by \Cref{lem:exttau-extsigma-one-step} with $w \in \smash{\clR^\gamma}$.

  We now prove the claim.
  Let $\gamma_1, \gamma_2 \in \Gamma$ with $w \in \smash{\clR^{\gamma_1} \cap \clR^{\gamma_2}}$,
  and fix $x \defeq (z_{l-1}(w), w_l, \ldots, w_L)$ and $x' \defeq (z_{l}(w), w_{l+1}, \ldots, w_L)$.
  By induction hypothesis on $l+1$, we obtain
  \begin{align}
    \label{eq:tmp2-4}
    \DF{\ext{z}_{l+1}^{\gamma_1}}(x') 
    &= \DF{\ext{z}_{l+1}^{\gamma_2}}(x'). 
  \end{align}
  Since we want to show $\smash{\DF{\ext{z}_l^{\gamma_1}}}(x) = \smash{\DF{\ext{z}_l^{\gamma_2}}}(x)$,
  it suffices to show the following due to \Cref{eq:tmp2-1,eq:tmp2-2,eq:tmp2-3,eq:tmp2-4}:
  for all $i \in [N_l]$,
  \begin{align}
    \label{eq:tmp2-5}
    \begin{aligned}
      & \PDF{\ext{z}_{l+1}^{\gamma_1}}{i}(x') 
      \cdot \DF{\sigma_{l,i}^{\gamma_1(l,i)}}\big(y_{l,i}(w)\big)
      =
      \PDF{\ext{z}_{l+1}^{\gamma_1}}{i}(x') 
      \cdot \DF{\sigma_{l,i}^{\gamma_2(l,i)}}\big(y_{l,i}(w)\big).
    \end{aligned}
  \end{align}
  Let $i \in [N_l]$.
  We prove \Cref{eq:tmp2-5} by case analysis on $i$.

  {\it Subcase 1: $\sigma_{l,i}$ is non-differentiable at $y_{l,i}(w)$.}
  To show \Cref{eq:tmp2-5}, it suffices to show that
  \[
  \PDF{\ext{z}_{l+1}^{\gamma_1}}{i}(x') 
  = (0, \ldots, 0).
  \]
  We obtain this equation by applying \Cref{lem:djz-gamma-zero} to $(w, (l,i), \gamma_1)$.
  Note that the lemma is applicable here because:
  $z_L$ is differentiable at $w$ (by assumption);
  $\sigma_{l,i}$ is non-differentiable at $y_{l,i}(w)$ and so $\tau_l$ has bias parameters (by assumption);
  $\PDF{\smash{\ext{z}_{l+1}^{\gamma}}}{i}(x')$ is independent of $\gamma$
  for all $\gamma \in \Gamma$ with $w \in \smash{\clR^{\gamma}}$ (by induction hypothesis on $l+1$);
  and $w \in \smash{\clR^{\gamma_1}}$.

  {\it Subcase 2: $\sigma_{l,i}$ is differentiable at $y_{l,i}(w)$.}
  To show \Cref{eq:tmp2-5}, it suffices to show that for all $j \in [2]$,
  \begin{align}
    \label{eq:tmp2-6}
    \DF{\smash{\sigma_{l,i}^{\gamma_j(l,i)}}}\big(y_{l,i}(w)\big) = \DF{\sigma_{l,i}}\big(y_{l,i}(w)\big).
  \end{align}
  Let $j \in [2]$.
  If $y_{l,i}(w) \in \smash{\mcI_{l,i}^{\gamma_j(l,i)}}$, then we obtain \Cref{eq:tmp2-6} as follows:
  \[
  \DF{\sigma_{l,i}^{\gamma_j(l,i)}}\big(y_{l,i}(w)\big)
  = \adf{\sigma_{l,i}}\big(y_{l,i}(w)\big)
  = \DF{\sigma_{l,i}}\big(y_{l,i}(w)\big),
  \]
  where the first equality holds because $y_{l,i}(w) \in \smash{\mcI_{l,i}^{\gamma_j(l,i)}}$ and
  $\smash{\{(\sigma_{l,i}^{k}, \mcI_{l,i}^{k}) \}_{k \in [K_{l,i}]}}$ defines $\adf{\sigma_{l,i}}$;
  and the second equality holds because $\sigma_{l,i}$ is differentiable at $y_{l,i}(w)$
  and $\adf{\sigma_{l,i}}$ is an extended derivative of $\sigma_{l,i}$.
  If $y_{l,i}(w) \notin \smash{\mcI_{l,i}^{\gamma_j(l,i)}}$, then we obtain \Cref{eq:tmp2-6} directly from \Cref{lem:f-same-df-same}
  applied to $(\smash{\sigma_{l,i}^{\gamma_j(l,i)}}\!\!, \sigma_{l,i}, y_{l,i}(w))$.
  Note that the lemma is applicable here because:
  $\smash{\sigma_{l,i}^{\gamma_j(l,i)}}$ and $\sigma_{l,i}$ are differentiable at $y_{l,i}(w)$;
  they coincide on $\smash{\mcI_{l,i}^{\gamma_j(l,i)}}$;
  $y_{l,i}(w) \in \smash{\cl(\mcI_{l,i}^{\gamma(l,i)})}$ (by $w \in \smash{\clR^\gamma}$);
  and $\intr(\smash{\mcI_{l,i}^{\gamma(l,i)}}) \neq \emptyset$ (by $y_{l,i}(w) \notin \smash{\mcI_{l,i}^{\gamma(l,i)}}$).
  Therefore, \Cref{eq:tmp2-5} holds and this completes the proof.
\end{proof}

\subsection{\Cref{thm:inc-zero-bias} (Main Proof)}
\label{sec:pf-inc-zero-bias-main}

{\bf \Cref{thm:inc-zero-bias}.}
{\it
  If $z_L$ has bias parameters, then for all $w \in \bbR^W$ at which $z_L$ is differentiable,
  \begin{align*}
    \ADF{z_L}(w) = \DF{z_L}(w).
  \end{align*}
  This implies that $|\incM{z_L}| = 0$.
}
\begin{proof} 
  Let $w \in \bbR^W$ such that $z_L$ is differentiable at $w$ (i.e., $\DF{z_L}(w) \neq \bot$).
  By \Cref{lem:r-gamma-partition}, there is (unique) $\gamma \in \Gamma$ such that $w \in \mcR^\gamma$.
  Using the $\gamma$, we obtain the conclusion:
  \begin{align*}
    \DF{z_L}(w)
    &= \big[\, \PDF{z_L}{1}(w) \;\big|\; \cdots \;\big|\; \PDF{z_L}{W}(w) \,\big]
    \\
    &= \big[\, \PDF{z_L^{\gamma_1}}{1}(w) \;\big|\; \cdots \;\big|\; \PDF{z_L^{\gamma_W}}{W}(w) \,\big]
    \quad\text{for some $\gamma_j \in \Gamma$ with $w \in \clR^{\gamma_j}$ ($j \in [W]$)}
    \\
    &= \big[\, \PDF{z_L^{\gamma}}{1}(w) \;\big|\; \cdots \;\big|\; \PDF{z_L^{\gamma}}{W}(w) \,\big]
    \\
    &= \DF{z_L^{\gamma}}(w)
    = \ADF{z_L}(w).
  \end{align*}
  Here the second line uses \Cref{lem:djz-gamma-exist-spc} with that $z_L$ is differentiable at $w$.
  The third line uses \Cref{lem:djz-gamma-indep} with the following:
  $z_L$ is differentiable at $w$;
  $\tau_l$ has bias parameters for all $l \in [L]$ (by assumption);
  and $w \in \mcR^\gamma \subseteq \smash{\clR^\gamma}$ and $w \in \smash{\clR^{\gamma_j}}$ for all $j \in [W]$ (by the second line).
  The last line uses \Cref{lem:r-gamma-adf} with $w \in \mcR^\gamma$.
\end{proof}

\subsection{Lemmas (Technical: Part 3)}

\begin{lemma}
  \label{lem:can-avoid-touching-nal}
  Let $z_L$ be a neural network, $w \in \bbR^W$, and $A \subseteq \Idx$.
  Suppose that
  $w \notin \pbd( \{u \in \bbR^W |\, y_{l,i}(u) = c \})$ for all $(l,i) \in A$ and $c \in \ndf{\sigma_{l,i}}$.
  Then, there is a neural network $z'_L$
  (which consists of $\tau'_l$, $\sigma'_{l,i}$, $y'_l$, $z'_l$, and $\smash{\adf{\sigma'_{l,i}}}$) satisfying the following conditions:
  \begin{enumerate}[label=\protect\textcircled{\arabic*}]
  \item $\DF{z'_L}(w) = \DF{z_L}(w)$, $\ADF{z'_L}(w) = \ADF{z_L}(w)$, and $\tau'_l = \tau_l$ for all $l \in [L]$.
  \item $y'_{l,i}(w) \notin \ndf{\sigma'_{l,i}}$ for all $(l,i) \in A$.
  \end{enumerate}
\end{lemma}
\begin{proof}
  Consider the setup given above.
  Define a function $f$ from neural networks to $\bbN$ as:
  \[
  f(z'_L) \defeq \Big| \Big\{ \, (l,i,c) \;\Big|\;
  (l,i) \in A,\, c \in \ndf{\sigma'_{l,i}},\,  w \in \intr\big(\{u \in \bbR^W |\, y'_{l,i}(u)=c \}\big)
  \, \Big \} \Big|.
  \]
  Note that $f(z_L) \in \bbN$ (i.e., $f(z_L)< \infty$),
  because $\sigma_{l,i}$ is continuous, piecewise-analytic and so $|\ndf{\sigma_{l,i}}| < \infty$ for all $(l,i) \in \Idx$
  (by \Cref{thm:int-deriv-basic}).
  The proof proceeds by induction on $f(z_L)$.

  \paragraph{\bf Case $f(z_L)=0$.}
  We claim that $z_L$ satisfies \textcircled{1}-\textcircled{2}.
  Clearly, it satisfies \textcircled{1}.
  Further, it also satisfies \textcircled{2}:
  by the assumption and $f(z_L)=0$, we have that for all $(l,i) \in A$ and $c \in \ndf{\sigma_{l,i}}$,
  \begin{align*}
    w \notin \pbd\big( \{u \in \bbR^W \mid y_{l,i}(u) = c \}\big) \cup \intr\big( \{u \in \bbR^W \mid y_{l,i}(u) = c \}\big)
    = \{u \in \bbR^W \mid y_{l,i}(u) = c \},
  \end{align*}
  which implies that $y_{l,i}(w) \neq c$.
  Hence, $y_{l,i}(w) \notin \ndf{\sigma_{l,i}}$ for all $(l,i) \in A$, as desired.
  
  \paragraph{\bf Case $f(z_L)>0$.}
  Since $f(z_L)>0$, there are $(l{},i{}) \in A$ and $c{} \in \ndf{\sigma_{l{},i{}}}$ such that
  \( w \in \intr(\{u \in \bbR^W |\, y_{l{},i{}}(u)=c{} \}). \)
  This implies that there is an open $U{} \subseteq \bbR^W$ such that $w \in U{}$ and $y_{l{},i{}}(u) = c{}$ for all $u \in U{}$.
  Let $\smash{z'_L}$ be the exactly same neural network as $z_L$
  except that it uses different $\smash{\sigma'_{l{},i{}}}$ and $\smash{\adf{\sigma'_{l{},i{}}}}$:
  \begin{align*}
    \sigma'_{l{},i{}}(x) \defeq \adf{\sigma_{l{},i{}}}(c{}) \cdot (x-c{}) + \sigma_{l{},i{}}(c{}),
    \qquad\quad
    \adf{\sigma'_{l{},i{}}}(x) \defeq \adf{\sigma_{l{},i{}}}(c{}).
  \end{align*}
  Note that $\sigma'_{l,i}$ and $\adf{\sigma'_{l,i}}$ satisfy the assumptions in \Cref{sec:nn}--\Cref{sec:ad}:
  the former is continuous and piecewise-analytic (since it is differentiable),
  and the latter is an extended derivative of the former
  (since the former is differentiable and $\smash{\adf{\sigma'_{l,i}}} = \smash{\DF{\sigma'_{l,i}}}$).
  Moreover, $\smash{z'_L}$ satisfies \textcircled{1}
  because $y_{l,i}(u) = c$ for all $u \in U$,
  and because $\smash{\sigma'_{l,i}}(c) = \sigma_{l,i}(c)$ and $\smash{\adf{\sigma'_{l,i}}}(c) = \smash{\adf{\sigma_{l,i}}}(c)$.
  Further, we have
  that $w \notin \pbd( \{u \in \bbR^W |\, \smash{y'_{l',i'}(u)} = c' \})$ for all $(l',i') \in A$ and all $c' \in \smash{\ndf{\sigma'_{l',i'}}}$,
  and that
  \[ f(z'_L) = f(z_L) -1, \]
  where both results follow from
  $\smash{\ndf{\sigma'_{l,i}}} = \emptyset$,
  $\smash{\ndf{\sigma'_{l',i'}}} = \ndf{\sigma_{l',i'}}$ for all $(l',i') \neq (l,i)$,
  and $\smash{y'_{l',i'}} = y_{l',i'}$ on $U$ for all $(l',i') \in \Idx$.
  Hence, we can apply induction to $z'_L$,
  and by induction hypothesis, there is a neural network $z''_L$ such that
  $(z''_L, z'_L)$ (instead of $(z'_L, z_L)$) satisfies \textcircled{1}-\textcircled{2}.
  From this, and since $(z'_L,z_L)$ satisfies \textcircled{1} (by the above),
  we conclude $(z''_L, z_L)$ satisfies \textcircled{1}-\textcircled{2}, as desired.
\end{proof}

\begin{lemma}
  \label{lem:inc-ubound-nobias-pre}
  We have
  \begin{align*}
    \incR{z_L} \subseteq \bigcup_{(l,i) \in \Idx} \, \bigcup_{c \in \ndf{\sigma_{l,i}} \cap S_l}
    \pbd\big(\{ w \in \bbR^W \;|\; y_{l,i}(w) = c \}\big),
  \end{align*}
  where $S_l \subseteq \bbR$ is defined by $S_l \defeq \emptyset$ if $\tau_l$ has bias parameters, and $S_l \defeq \bbR$ otherwise.
\end{lemma}
\begin{proof}
  Let $U \subseteq \bbR^W$ be the RHS of the above equation:
  \begin{align*}
    U & \defeq
    \bigcup_{\substack{l \in [L]: S_l = \bbR}} \,
    \bigcup_{i \in [N_l]} \, \bigcup_{c \in \ndf{\sigma_{l,i}}} \pbd\big( \{w \in \bbR^W \mid y_{l,i}(w) = c \}\big).
  \end{align*}
%
  Then, it suffices to show that for any $w \in \bbR^W$, $w \notin U$ implies $w \notin \incR{z_L}$.
  Consider any $w \in \bbR^W$ with $w \notin U$.
  We want to show $w \notin \incR{z_L}$.
  If $z_L$ is not differentiable at $w$, then $w \notin \incR{z_L}$ clearly holds by the definition of $\incR{-}$.
  Hence, assume that $z_L$ is differentiable at $w$.
  By the definition of $\incR{-}$, it suffices to show the following:
  \begin{align}
    \label{eq:lem:inc-ubound-nobias-pre-0}
    \ADF{z_L}(w) = \DF{z_L}(w).
  \end{align}
  We prove this in two steps.
  
  \paragraph{\bf Step 1.}
  Since $z_L$ at $w$ does not satisfy the assumption of \Cref{lem:djz-gamma-indep}
  (which we will apply to show \Cref{eq:lem:inc-ubound-nobias-pre-0}),
  we construct another neural network $z'_L$ that is identical to $z_L$ nearby $w$ while satisfying the assumption.
  To do so, we apply \Cref{lem:can-avoid-touching-nal} to $(z_L, w, A)$ with $A \defeq \{(l,i) \in \Idx \mid S_l = \bbR\}$.
  The lemma is applicable here, since
  $w \notin \pbd( \{v \in \bbR^W \mid y_{l,i}(v) = c \})$ for all $(l,i) \in A$ and $c \in \ndf{\sigma_{l,i}}$ (by $w \notin U$).
  Hence, by \Cref{lem:can-avoid-touching-nal}, we get a neural network $z'_L$
  (which consists of $\smash{\tau'_l}$, $\smash{\sigma'_{l,i}}$, $\smash{y'_l}$, $\smash{z'_l}$, and $\smash{\adf{\sigma'_{l,i}}}$) satisfying the following conditions:
  \begin{enumerate}[label=\protect\textcircled{\arabic*}]
  \item $\DF{z'_L}(w) = \DF{z_L}(w)$, $\ADF{z'_L}(w) = \ADF{z_L}(w)$, and $\tau'_l = \tau_l$ for all $l \in [L]$.
  \item $y'_{l,i}(w) \notin \ndf{\sigma'_{l,i}}$ for all $(l,i) \in A$.
  \end{enumerate}
  
  \paragraph{\bf Step 2.}
  We now prove \Cref{eq:lem:inc-ubound-nobias-pre-0} based on $z'_L$.
  Let $\Gamma'$, $\mcR'$, and $\clR'$ be the counterparts of $\Gamma$, $\mcR$, and $\clR$ for $\smash{z'_L}$.
  Then, by \Cref{lem:r-gamma-partition}, there is $\gamma' \in \Gamma'$ such that $w \in \smash{{\mcR'}^{\gamma'}}$.
  Using $z'_L$ and $\gamma'$, we obtain \Cref{eq:lem:inc-ubound-nobias-pre-0}:
  \begin{align*}
    \DF{z_L}(w)
    &= \DF{z'_L}(w)
    \\
    &= \big[\, \PDF{z'_L}{1}(w) \;\big|\; \cdots \;\big|\; \PDF{z'_L}{W}(w) \,\big]
    \\
    &= \big[\, \PDF{{z'_L}^{\gamma'_1}}{1}(w) \;\big|\; \cdots \;\big|\; \PDF{{z'_L}^{\gamma'_W}}{W}(w) \,\big]
    \quad\text{for some $\gamma'_j \in \Gamma'$ with $w \in {\clR'}^{\gamma'_j}$ ($j \in [W]$)}
    \\
    &= \big[\, \PDF{{z'_L}^{\gamma'}}{1}(w) \;\big|\; \cdots \;\big|\; \PDF{{z'_L}^{\gamma'}}{W}(w) \,\big]
    \\
    &= \DF{{z'_L}^{\gamma'}}(w)
    = \ADF{z'_L}(w)
    = \ADF{z_L}(w).
  \end{align*}
  Here the first and last lines use \textcircled{1} and \Cref{lem:r-gamma-adf} with $w \in \smash{{\mcR'}^{\gamma'}}$.
  The third line uses \Cref{lem:djz-gamma-exist-spc} with that $\smash{z'_L}$ is differentiable at $w$ (by \textcircled{1}).
  The fourth line uses \Cref{lem:djz-gamma-indep} with the following:
  $\smash{z'_L}$ is differentiable at $w$ (by \textcircled{1});
  for all $(l,i) \in \Idx$, if $\tau'_l$ does not have bias parameters,
  then $y'_{l,i}(w) \notin \ndf{\sigma'_{l,i}}$, i.e., $\sigma'_{l,i}$ is differentiable at $y'_{l,i}(w)$
  (by \textcircled{1} and \textcircled{2}); 
  and $w \in \smash{{\mcR'}^{\gamma'}} \subseteq \smash{{\clR'}^{\gamma'}}$ and $w \in \smash{{\clR'}^{\gamma'_j}}$
  for all $j \in [W]$ (by the third line).
\end{proof}

\subsection{\Cref{thm:inc-ubound-nobias} (Main Lemma)}

\begin{lemma}
  \label{lem:inc-ubound-nobias}
  Suppose that for every $l \in[L]$, one of the following holds:
  \begin{itemize}
    \item[(a)] $\tau_l$ has bias parameters, or
    \item[(b)] $\tau_l$ is well-structured biaffine.
  \end{itemize}
  In the case of (b), let $\phi_{l,i}$ be the partial map described in \Cref{lem:strong-bilinear} for all $i\in[N_{l}]$.
  Then,
  \begin{align*}
    \incM{z_L} \subseteq \bigcup_{(l,i) \in \Idx} \, \bigcup_{c \in A_{l,i}} B_{l,i}(c),
  \end{align*}
  where $\smash{A_{l,i}} \subseteq \bbR$ and $\smash{B_{l,i}(c)} \subseteq \Omega$  are defined as
  \begin{align*}
    A_{l,i} &\defeq
    \begin{cases}
      (\ndf{\sigma_{l,i}} \cap S_l)
      & \text{if $\tau_{l+1}$ satisfies the condition (a) or $l=L$} 
      \\
      (\ndf{\sigma_{l,i}} \cap S_l) \cup \bdz{\sigma_{l,i}}
      & \text{if $\tau_{l+1}$ satisfies the condition (b)},
    \end{cases}
    \\
    B_{l,i}(c) &\defeq
    \begin{cases}
      \{ w \in \Omega \;|\; y_{l,i}(w) = c  \}
      & \text{if $\tau_{l}$ satisfies the condition (a)}
      \\
      \{ w \in \Omega \;|\; y_{l,i}(w) = c  \land \bigvee_{j \in \dom(\phi_{l,i})} z_{l-1,\phi_{l,i}(j)}(w) \neq 0  \}
      & \text{if $\tau_{l}$ satisfies the condition (b)},
    \end{cases}
  \end{align*}
  and $S_l \subseteq \bbR$ is defined as $S_l \defeq \emptyset$ if $\tau_l$ has bias parameters, and $S_l \defeq \bbR$ otherwise.
\end{lemma}
\begin{proof}
  We obtain the conclusion
  by chaining \Cref{lem:inc-ubound-nobias-pre}, 
  \Cref{lem:pbd-bound-nobias-aux}
  (which is applicable by the assumption on $\tau_l$), 
  and $\ndfM{z_L} \cup \incM{z_L} = \big(\ndfR{z_L} \cup \incR{z_L}\big) \cap \Omega$.
\end{proof}

\subsection{\Cref{thm:inc-ubound-nobias} (Main Proof)}

{\bf \Cref{thm:inc-ubound-nobias}.}
{\it
  If $\tau_l$ either has bias parameters or is well-structured biaffine for all $l\in[L]$, then
  \begin{align*}
    \frac{|\incM{z_L}|}{|\Omega|}
    &\leq \frac{1}{|\bbM|} \smash{\sum_{(l,i) \in \Idx}}
    \Big| \big(\ndf{\sigma_{l,i}} \cap S_l\big)
    \cup \big(\bdz{\sigma_{l,i}} \cap S_{l+1} \big) \Big|,
  \end{align*}
  where $S_l \subseteq \bbR$ is defined by 
  \begin{align*}
    S_l \defeq
    \begin{cases}
      \emptyset~~ & \text{if $l > L$ or $\tau_l$ has bias parameters}
      \\ \bbR~~ & \text{otherwise}.
    \end{cases}
  \end{align*}
}
\begin{proof} 
  Observe that
  \begin{align}
    \label{eq:thm:inc-ubound-nobias-1}
    \incM{z_L}
    &\subseteq \bigcup_{(l,i) \in \Idx} \, \bigcup_{c \in A_{l,i}} B_{l,i}(c),
    \qquad\quad
    {|B_{l,i}(c)|} \leq {|\bbM|^{W-1}},
  \end{align}
  where $S_l \subseteq \bbR$, $\smash{A_{l,i}} \subseteq \bbR$ and $\smash{B_{l,i}(c)} \subseteq \Omega$ for $l \in [L]$
  are defined as in \Cref{lem:inc-ubound-nobias}.
  Here the first equation is by \Cref{lem:inc-ubound-nobias}
  and the second equation is by \Cref{lem:sol-count-bias,lem:sol-count-nobias},
  where these lemmas are applicable by the definition of $B_{l,i}(c)$ and
  because 
  $\tau_l$ either has bias parameters or is well-structured biaffine (by assumption).
  Observe further that
  \begin{align}
    \label{eq:thm:inc-ubound-nobias-2}
    A_{l,i} &= (\ndf{\sigma_{l,i}} \cap S_l) \cup (\bdz{\sigma_{l,i}} \cap S_{l+1}),
  \end{align}
  by the definition of $A_{l,i}$ and $S_l$,
  where we use $S_{L+1} \defeq \emptyset$.
  Combining the above observations, we obtain the conclusion:
  \begin{align*}
    \frac{|\incM{z_L}|}{|\Omega|}
    &\leq \sum_{(l,i) \in \Idx} \, \sum_{c \in A_{l,i}} \frac{|B_{l,i}(c)|}{|\Omega|}
    \leq \sum_{(l,i) \in \Idx} \big|(\ndf{\sigma_{l,i}} \cap S_l) \cup (\bdz{\sigma_{l,i}} \cap S_{l+1})\big| \cdot \frac{|\bbM|^{W-1}}{|\bbM|^W}, 
  \end{align*}
  where the first inequality uses \Cref{eq:thm:inc-ubound-nobias-1}
  and the second inequality uses \Cref{eq:thm:inc-ubound-nobias-1,eq:thm:inc-ubound-nobias-2}.
\end{proof}


\clearpage
\section{Lower Bounds on $|\ndfM{z_L}|$ and $|\incM{z_L}|$}
\label{sec:pf-lowerbd}

\subsection{\Cref{thm:ndf-lbound-bias} (Main Proof)}
\label{sec:pf-lowerbd-ndf-bias}

{\bf \Cref{thm:ndf-lbound-bias}.}
{\it
  For any $\bbM \subseteq \bbR$ and $n, \alpha \in \bbN$ with $1 \leq |\bbM| < \infty$, $n \geq 2$, and $\alpha \leq |\bbM|/(n-1)$,
  there is a neural network $z_L : \bbR^W \to \bbR$ that satisfies
  \begin{align*}
    \frac{|\ndfM{z_L}|}{|\Omega|}
    \geq \frac{1}{2} \cdot \frac{1}{|\bbM|} \sum_{(l,i) \in \Idx} | \ndf{\sigma_{l,i}} |
  \end{align*}
  and the following: $z_L$ has bias parameters, it has $n+1$ neurons,
  and $|\ndf{\sigma_{1,i}}| = \alpha$ for all $i \in [N_1]$.
}
\begin{proof}
  Consider any $\bbM \subseteq \bbR$ and $n, \alpha \in \bbN$ that satisfy the assumption.
  We claim that there is a neural network $z_L$ that has $L = 2$ layers, $N = n+1$ neurons, and $W = n+1$ parameters,
  and satisfies the given inequality.

  We first define a few components to be used in the network.
  Let $\{x_1, \ldots, x_\alpha\} \subseteq \bbM$ be distinct machine-representable numbers,
  and $h : \bbR \to \bbR$ be a continuous, piecewise-analytic function such that $\ndf{h} = \{x_1, \ldots, x_\alpha\}$.
  Note that such $x_j$ always exists since $|\bbM| \geq \alpha$ (by assumption).
  Using $h$, define a function $f: \bbR^W \to \bbR$ as
  \begin{align*}
    f(w) &= w_{n+1} + \sum_{i \in [n]} h(w_i).
  \end{align*}
  We assume here (and in the rest of the proof) that $w \in \bbR^W$ is represented as $w = (w_1, \ldots, w_W)$ for $w_i \in \bbR$
  (instead of $w = (w_{1,1}, w_{1,2}, \ldots, w_{L,W_L})$ with $w_{l,j} \in \bbR$ as we assumed so far).
  
  Given these, we construct a neural network $z_L : \bbR^W \to \bbR$ that is essentially the same as $f$, as follows
  \begin{align*}
    z_0(w) &= 0 \in \bbR,
    \\
    y_1(w) &= (w_1, \ldots, w_n) \in \bbR^n,
    &
    z_1(w) &= (h(w_1), \ldots, h(w_n)) \in \bbR^n,
    \\
    y_2(w) &= f(x) \in \bbR,
    &
    z_2(w) &= f(x) \in \bbR.
  \end{align*}
  Then, $z_L$ has $2$ layers, $n+1$ neurons, and $n+1$ parameters, and $|\ndf{\sigma_{1,i}}| = |\ndf{h}| = \alpha$ for all $i$.
  Also, we can easily make all $\tau_l$ have bias parameters
  (e.g., by using $\tau_1(x,w_1, \ldots, w_n) = (x + w_1, \ldots, x + w_n)$).
  What remains is to prove that $z_L$ satisfies the inequality in the conclusion.
  To do so, observe that
  \begin{align*}
    \ndfM{z_L}
    & \supseteq \{ w \in \Omega \mid w_i \in \ndf{h} \text{ for some } i \in [n] \}
    \\
    & = \Omega \setminus \{ w \in \Omega \mid w_i \notin \ndf{h} \text{ for all } i \in [n] \},
  \end{align*}
  which follows from the definition of $f$ and $\ndf{h} \subseteq \bbM$.
  From this, we have
  \begin{align*}
    \frac{|\ndfM{z_L}|}{|\Omega|}
    & \geq \frac{1}{|\bbM|^{n+1}} \Big(|\bbM|^{n+1} - |\bbM| \cdot (|\bbM|-\alpha)^{n} \Big)
    = 1 - \Big(1 - \frac{\alpha}{|\bbM|}\Big)^{n}
    \\
    & \geq 1 - \Big( 1 - n \frac{\alpha}{|\bbM|} + \frac{1}{2} n(n-1) \Big(\frac{\alpha}{|\bbM|}\Big)^2 \Big)
    = \frac{n \alpha}{|\bbM|} \Big( 1 - \frac{n-1}{2} \frac{\alpha}{|\bbM|} \Big)
    \\
    & \geq \frac{1}{2} \cdot \frac{n \alpha}{|\bbM|},
  \end{align*}
  where the first inequality uses $\ndf{h} \subseteq \bbM$ and $|\ndf{h}| = \alpha$,
  the second inequality follows from $(1-x)^n \leq 1 -nx + \frac{1}{2} n(n-1) x^2$ (for any $x \leq 1$ and $n \in \bbN$)
  and $\alpha \leq |\bbM|$,
  and the third inequality is by the assumption that $\alpha \leq |\bbM|/(n-1)$.
  By combining this result and
  \begin{align*}
    \frac{1}{|\bbM|} \sum_{(l,i) \in \Idx} |\ndf{\sigma_{l,i}}|
    = \frac{n \alpha}{|\bbM|},
  \end{align*}
  we obtain the desired inequality.
\end{proof}

\subsection{\Cref{thm:ndf-lbound-nobias} (Main Proof)}
\label{sec:pf-lowerbd-ndf-nobias}

{\bf \Cref{thm:ndf-lbound-nobias}.}
{\it
  For any $\bbM \subseteq \bbR$ and $n, \alpha \in \bbN$ with $1 \leq |\bbM| < \infty$, $n \geq 4$, and $\alpha \leq |\bbM|/(n-1)$,
  there is a neural network $z_L : \bbR^W \to \bbR$ that satisfies
  \begin{align*}
    &\frac{|\ndfM{z_L}|}{|\Omega|}
    \geq \frac{1}{9} \cdot \frac{1}{|\bbM|} {\sum_{(l,i) \in \Idx}}
    \Big| \ndf{\sigma_{l,i}} \cup \bdz{\sigma_{l,i}} \Big|
  \end{align*}
  and the following:
  (i) $\tau_l$ is well-structured biaffine without bias parameters for all $l<L$, and has bias parameters for $l=L$; 
  (ii) $z_L$ has $n \,{+}\, 1$ neurons;
  and (iii) $|\ndf{\sigma_{1,i}}| \,{=}\, \alpha$, $|\bdz{\sigma_{1,i}}| \,{=}\, 0$ for all $i$. 
  We get the same result for (i), (ii'), and (iii'):
  (ii') $z_L$ has $2n \,{+}\,1$ neurons;
  and (iii') $|\ndf{\sigma_{1,i}}| \,{=}\, 0$, $|\bdz{\sigma_{1,i}}| \,{=}\, \alpha$ for all~$i$. 
}
\begin{proof}
  We prove the two cases (one for (i), (ii), (iii), and the other for (i), (ii'), (iii')) as follows.
  Consider any $\bbM \subseteq \bbR$ and $n, \alpha \in \bbN$ that satisfy the assumption.
  Let $\{x_1, \ldots, x_\alpha\} \subseteq \bbM$ be distinct machine-representable numbers;
  such $x_j$ always exists since $|\bbM| \geq \alpha$ (by assumption).
  In the rest of the proof, we assume that $w \in \bbR^W$ is represented as $w = (w_1, \ldots, w_W)$ for $w_i \in \bbR$,
  as in the proof of \Cref{thm:ndf-lbound-bias} (see \Cref{sec:pf-lowerbd-ndf-bias}).
  
  {\bf First case.}
  Let $W = n+1$ and $h : \bbR \to \bbR$ be a continuous, piecewise-analytic function
  such that $\ndf{h} = \{x_1, \ldots, x_\alpha\}$ and $h(x) > 0$ for all $x \in \bbR$.
  Using this $h$, define a function $f: \bbR^{W} \to \bbR$ as
  \begin{align*}
    f(w) &= w_{n+1} + \sum_{i \in [n]} h(w_i).
  \end{align*}
  We now construct a neural network $z_L : \bbR^W \to \bbR$ that is essentially the same as $f$, as follows:
  \begin{align*}
    z_0(w) &= 1 \in \bbR,
    \\
    y_1(w) &= (w_1, \ldots, w_n) \in \bbR^n,
    &
    z_1(w) &= (h(w_1), \ldots, h(w_n)) \in \bbR^n,
    \\
    y_2(w) &= f(x) \in \bbR,
    &
    z_2(w) &= f(x) \in \bbR.
  \end{align*}
  Then, $z_L$ has $L=2$ layers, $N=n+1$ neurons, and $W=n+1$ parameters,
  and $|\ndf{\sigma_{1,i}}| = |\ndf{h}| = \alpha$ and $|\bdz{\sigma_{1,i}}| = |\bdz{h}| = 0$ for all $i$.
  Also, we can easily make $\tau_l$ be well-structured biaffine without bias parameters for all $l < L$,
  and make $\tau_L$ have bias parameters (e.g., by using $\tau_1(x,w_1, \ldots, w_n) = (x \cdot w_1, \ldots, x \cdot w_n)$).
  This shows that (i), (ii), and (iii) are satisfied.
  
  What remains is to prove that $z_L$ satisfies the inequality in the conclusion.
  To do so, observe that
  \begin{align*}
    \ndfM{z_L}
    & \supseteq \{ w \in \Omega \mid w_i \in \ndf{h} \text{ for some } i \in [n] \},
  \end{align*}
  which follows from the definition of $f$ and $\ndf{h} \subseteq \bbM$.
  From this, we have
  \begin{align*}
    \frac{|\ndfM{z_L}|}{|\Omega|}
    & \geq \frac{1}{2} \cdot \frac{n \alpha}{|\bbM|},
  \end{align*}
  as shown in the proof of \Cref{thm:ndf-lbound-bias} (see \Cref{sec:pf-lowerbd-ndf-bias}).
  Here we used $\ndf{h} \subseteq \bbM$ and $|\ndf{h}| = \alpha$,
  as well as $\alpha \leq |\bbM|/(n-1)$ and $n \geq 2$ (by assumption).
  Further, observe that
  \begin{align*}
    \frac{1}{|\bbM|} \sum_{(l,i) \in \Idx} \big|\ndf{\sigma_{l,i}} \cup \bdz{\sigma_{l,i}} \big|
    = \frac{n \alpha + 1}{|\bbM|}
    = \Big(1 + \frac{1}{n \alpha}\Big) \cdot \frac{n \alpha}{|\bbM|}
    \leq \frac{3}{2} \cdot \frac{n \alpha}{|\bbM|},
  \end{align*}
  where the inequality uses $n \geq 2$ and $\alpha \geq 1$ (by assumption).
  From these results, we obtain the desired inequality.

  {\bf Second case.}
  Let $W = n+2$ and  $h : \bbR \to \bbR$ be an analytic function
  such that $h(x_j) = 0$ and $\DF{h}(x_j) = 1$ for all $j \in [\alpha]$, and $|\bdz{h}| = \alpha$.
  We remark that such a function $h$ always exists due to Hermite interpolation \cite{BurdenFB15}.
  Using this $h$, define a function $f: \bbR^W \to \bbR$ as
  \begin{align*}
    f(w) &= w_{n+2} + \sum_{i \in [n]} \relu(h(w_i) \cdot w_{n+1}).
  \end{align*}
  We now construct a neural network $z_L : \bbR^W \to \bbR$ that is essentially the same as $f$, as follows:
  \begin{align*}
    z_0(w) &= 1, 
    \\
    y_1(w) &= (w_1, \ldots, w_n), 
    &
    z_1(w) &= (h(w_1), \ldots, h(w_n)), 
    \\
    y_2(w) &= (h(w_1) \cdot w_{n+1}, \ldots, h(w_n) \cdot w_{n+1}), 
    &
    z_2(w) &= (\relu(h(w_1) \cdot w_{n+1}), \ldots, \relu(h(w_n) \cdot w_{n+1})), 
    \\
    y_3(w) &= f(x), 
    &
    z_3(w) &= f(x). 
  \end{align*}
  Then, $z_L$ has $L=3$ layers, $N=2n+1$ neurons, and $W=n+2$ parameters,
  and $|\ndf{\sigma_{1,i}}| = |\ndf{h}| = 0$ and $|\bdz{\sigma_{1,i}}| = |\bdz{h}| = \alpha$ for all $i$.
  Also, we can easily make $\tau_l$ be well-structured biaffine without bias parameters for all $l < L$,
  and make $\tau_L$ have bias parameters, as discussed above.
  This shows that (i), (ii'), and (iii') are satisfied.
  
  What remains is to prove that $z_L$ satisfies the inequality in the conclusion.
  To do so, observe that
  \begin{align*}
    \ndfM{z_L}
    & \supseteq \{ w \in \Omega \mid w_{n+1} \neq 0 \text{ and } w_i \in \bdz{h} \text{ for some } i \in [n] \}
    \\
    & = \Omega \setminus \big(\{w \in \Omega \mid w_{n+1} = 0\} \cup \{ w \in \Omega \mid w_i \notin \bdz{h} \text{ for all } i \in [n] \}\big),
  \end{align*}
  which follows from the definition of $f$ and $\bdz{h} \subseteq \bbM$.
  From this, we have
  \begin{align*}
    \frac{|\ndfM{z_L}|}{|\Omega|}
    & \geq \frac{1}{|\bbM|^{n+2}} \Big(|\bbM|^{n+2} - |\bbM|^{n+1} - |\bbM|^2 \cdot (|\bbM|-\alpha)^{n} \Big)
    \\
    &\geq \frac{1}{2} \cdot \frac{n \alpha}{|\bbM|} - \frac{1}{|\bbM|}
    = \Big(\frac{1}{2} - \frac{1}{n\alpha} \Big) \cdot \frac{n \alpha}{|\bbM|}
    \\
    &\geq \frac{1}{4} \cdot \frac{n \alpha}{|\bbM|},
  \end{align*}
  where the second inequality follows from an argument in the proof of \Cref{thm:ndf-lbound-bias} (see \Cref{sec:pf-lowerbd-ndf-bias}),
  and the third inequality uses $n \geq 4$ and $\alpha \geq 1$ (by assumption).
  Note that when proving the second inequality, we used $\bdz{h} \subseteq \bbM$ and $|\bdz{h}| = \alpha$,
  as well as $\alpha \leq |\bbM|/(n-1)$ and $n \geq 2$ (by assumption).
  Further, observe that
  \begin{align*}
    \frac{1}{|\bbM|} \sum_{(l,i) \in \Idx} \big|\ndf{\sigma_{l,i}} \cup \bdz{\sigma_{l,i}} \big|
    = \frac{n \alpha + n + 1}{|\bbM|}
    = \Big(1 + \frac{1}{\alpha} + \frac{1}{n \alpha}\Big) \cdot \frac{n \alpha}{|\bbM|}
    \leq \frac{9}{4} \cdot \frac{n \alpha}{|\bbM|},
  \end{align*}
  where the inequality uses $n \geq 4$ and $\alpha \geq 1$ (by assumption).
  From these results, we obtain the desired inequality.
\end{proof}

\subsection{\Cref{thm:inc-lbound-nobias} (Main Proof)}
\label{sec:pf-lowerbd-inc-nobias}

{\bf \Cref{thm:inc-lbound-nobias}.}
{\it
  For any $\bbM \subseteq \bbR$ and $n, \alpha \in \bbN$ with $1 \leq |\bbM| < \infty$, $n \geq 4$, and $\alpha \leq |\bbM|/(n-1)$,
  there is a neural network $z_L : \bbR^W \to \bbR$ that satisfies
  \begin{align*}
    &\frac{|\incM{z_L}|}{|\Omega|}
    \geq \frac{1}{13} \cdot \frac{1}{|\bbM|}
    \sum_{(l,i) \in \Idx} \Big| \ndf{\sigma_{l,i}} \cup \bdz{\sigma_{l,i}} \Big|
  \end{align*}
  and the following:
  (i) $\tau_l$ is well-structured biaffine without bias parameters for all $l<L$, and has bias parameters for $l=L$; 
  (ii) $z_L$ has $2n \,{+}\, 1$ neurons;
  and (iii) $|\ndf{\sigma_{1,i}}| \,{=}\, \alpha$, $|\bdz{\sigma_{1,i}}| \,{=}\, 0$ for all $i$. 
  We get the same result for (i), (ii'), and (iii'):
  (ii') $z_L$ has $3n \,{+}\,1$ neurons;
  and (iii') $|\ndf{\sigma_{1,i}}| \,{=}\, 0$, $|\bdz{\sigma_{1,i}}| \,{=}\, \alpha$ for all~$i$. 
}
\begin{proof}
  We prove the two cases (one for (i), (ii), (iii), and the other for (i), (ii'), (iii')) as follows.
  Consider any $\bbM \subseteq \bbR$ and $n, \alpha \in \bbN$ that satisfy the assumption.
  Let $\{x_1, \ldots, x_\alpha\} \subseteq \bbM$ be distinct machine-representable numbers;
  such $x_j$ always exists since $|\bbM| \geq \alpha$ (by assumption).
  In the rest of the proof, we assume that $w \in \bbR^W$ is represented as $w = (w_1, \ldots, w_W)$ for $w_i \in \bbR$,
  as in the proof of \Cref{thm:ndf-lbound-bias} (see \Cref{sec:pf-lowerbd-ndf-bias}).
  
  {\bf First case.}
  Let $W = n+1$.
  Without loss of generality, assume that $\alpha$ is even and $0 < x_1 < \cdots < x_{\alpha/2}$;
  other cases can be handled in a similar way.
  Consider a continuous, piecewise-analytic function $h : \bbR \to \bbR$ that satisfies the following conditions:
  for all $j \in [\alpha/2]$, $h(x_j) = 1$ if $j$ is odd, and $h(x_j) = 2$ if $j$ is even;
  $\ndf{h} \cap (0, \infty) = \{x_1, \ldots, x_{\alpha/2}\}$;
  $h$ is piecewise linear, constant on $[x_{\alpha/2}, \infty)$, and even (i.e., $h(x) = h(-x)$ for all $x \in \bbR$).
  For this $h$, consider a (consistent) extended derivative $\adf{h} : \bbR \to \bbR$
  that takes the slope of the right piece of the function at non-differentiable points:
  e.g., $\adf{h}(x_2) = (h(x_3) - h(x_2)) / (x_3 - x_2)$ and $\adf{h}(-x_2) = (h(-x_1) - h(-x_2)) / (-x_1 + x_2)$.
  Using this $h$, define a function $f: \bbR^{W} \to \bbR$ as
  \begin{align*}
    f(w) &= w_{n+1} + \sum_{i \in [n]} h(w_i) - h(-w_i).
  \end{align*}
  Then, by using a similar approach taken in the proof of \Cref{thm:ndf-lbound-nobias} (see \Cref{sec:pf-lowerbd-ndf-nobias}),
  we can construct a neural network $z_L : \bbR^W \to \bbR$ that is essentially the same as $f$ and satisfies the following:
  $z_L$ has $L=2$ layers, $N=2n+1$ neurons, and $W=n+1$ parameters
  (where $2n$ neurons are at layer 1 and $1$ neuron is at layer 2);
  $\tau_l$ is well-structured biaffine without bias parameters for all $l < L$, and has bias parameters for $l = L$;
  and $|\ndf{\sigma_{1,i}}| = |\ndf{h}| = \alpha$ and $|\bdz{\sigma_{1,i}}| = |\bdz{h}| = 0$ for all $i$.
  This shows that (i), (ii), and (iii) are satisfied.
  
  What remains is to prove that $z_L$ satisfies the inequality in the conclusion.
  To do so, observe that
  \begin{align*}
    \incM{z_L}
    & \supseteq \{ w \in \Omega \mid w_i \in \{x_1, \ldots, x_{\alpha/2}\} \text{ for some } i \in [n] \},
  \end{align*}
  which follows from the definition of $f$ and $\{x_1, \ldots, x_{\alpha/2}\} \subseteq \bbM$.
  From this, we have
  \begin{align*}
    \frac{|\incM{z_L}|}{|\Omega|}
    & \geq \frac{1}{4} \cdot \frac{n \alpha}{|\bbM|}
  \end{align*}
  by a similar argument to that in the proof of \Cref{thm:ndf-lbound-bias} (see \Cref{sec:pf-lowerbd-ndf-bias}).
  Here we used $\{x_1, \ldots, x_{\alpha/2}\} \subseteq \bbM$
  as well as $\alpha \leq |\bbM|/(n-1)$ and $n \geq 2$ (by assumption).
  Further, observe that
  \begin{align*}
    \frac{1}{|\bbM|} \sum_{(l,i) \in \Idx} \big|\ndf{\sigma_{l,i}} \cup \bdz{\sigma_{l,i}} \big|
    = \frac{2n \alpha + 1}{|\bbM|}
    = \Big(2 + \frac{1}{n \alpha}\Big) \cdot \frac{n \alpha}{|\bbM|}
    \leq \frac{9}{4} \cdot \frac{n \alpha}{|\bbM|},
  \end{align*}
  where the inequality uses $n \geq 4$ and $\alpha \geq 1$ (by assumption).
  From these results, we obtain the desired inequality.

  {\bf Second case.}
  Let $W = n+2$ and  $h : \bbR \to \bbR$ be an analytic function
  such that $h(x_j) = 0$ and $\DF{h}(x_j) = 1$ for all $j \in [\alpha]$, and $|\bdz{h}| = \alpha$.
  Using this $h$, define a function $f: \bbR^W \to \bbR$ as
  \begin{align*}
    f(w) &= w_{n+2} + \sum_{i \in [n]} \relu(h(w_i) \cdot w_{n+1}) - \relu(-h(w_i) \cdot w_{n+1}),
  \end{align*}
  and let $\adf{\relu} = \indc{(0, \infty)}$.
  By using an approach similar to the above,
  we can construct a neural network $z_L : \bbR^W \to \bbR$ that is essentially the same as $f$ and satisfies the following:
  $z_L$ has $L=3$ layers, $N=3n+1$ neurons, and $W=n+2$ parameters
  (where $n$ neurons are at layer 1, $2n$ neurons at layer 2, and $1$ neuron at layer 3);
  $\tau_l$ is well-structured biaffine without bias parameters for all $l < L$, and has bias parameters for $l = L$;
  and $|\ndf{\sigma_{1,i}}| = |\ndf{h}| = 0$ and $|\bdz{\sigma_{1,i}}| = |\bdz{h}| = \alpha$ for all $i$.
  This shows that (i), (ii'), and (iii') are satisfied.
  
  What remains is to prove that $z_L$ satisfies the inequality in the conclusion.
  To do so, observe that
  \begin{align*}
    \incM{z_L}
    & \supseteq \{ w \in \Omega \mid w_{n+1} \neq 0 \text{ and } w_i \in \bdz{h} \text{ for some } i \in [n] \},
  \end{align*}
  which follows from the definition of $f$ and $\bdz{h} \subseteq \bbM$.
  From this, we have
  \begin{align*}
    \frac{|\incM{z_L}|}{|\Omega|}
    &\geq \frac{1}{4} \cdot \frac{n \alpha}{|\bbM|},
  \end{align*}
  as shown in the proof of \Cref{thm:ndf-lbound-nobias} (see \Cref{sec:pf-lowerbd-ndf-nobias}).
  Here we used $\bdz{h} \subseteq \bbM$ and $|\bdz{h}| = \alpha$,
  as well as $1 \leq \alpha \leq |\bbM|/(n-1)$ and $n \geq 4$ (by assumption).
  Further, observe that
  \begin{align*}
    \frac{1}{|\bbM|} \sum_{(l,i) \in \Idx} \big|\ndf{\sigma_{l,i}} \cup \bdz{\sigma_{l,i}} \big|
    = \frac{n \alpha + 2n + 1}{|\bbM|}
    = \Big(1 + \frac{2}{\alpha} + \frac{1}{n \alpha}\Big) \cdot \frac{n \alpha}{|\bbM|}
    \leq \frac{13}{4} \cdot \frac{n \alpha}{|\bbM|},
  \end{align*}
  where the inequality uses $n \geq 4$ and $\alpha \geq 1$ (by assumption).
  From these results, we obtain the desired inequality.
\end{proof}


\clearpage
\section{Computation of Standard Derivatives}
\label{sec:pf-corderiv}

\subsection{Lemmas (Basic)}

\begin{lemma}
  \label{lem:func-merge-diffl}
  Let $f, f_1, \ldots, f_n : \bbR^d \to \bbR^{d'}$ ($n \in \bbN$),
  $x \in \bbR^d$, and $U \subseteq \bbR^d$ be an open neighborhood of $x$.
  Suppose that for all $y \in U$, $f(y) = f_i(y)$ for some $i \in [n]$.
  Also, assume that $f(x) = f_i(x)$ for all $i \in [n]$, and $\DF{f_i}(x) = \DF{f_j}(x) \neq \bot$ for all $i, j \in [n]$.
  Then,
  \[
  \DF{f}(x) = \DF{f_i}(x) \neq \bot \qquad\text{for all $i \in [n]$}.
  \]
\end{lemma}
\begin{proof}
  Consider the setup of the statement.
  By the assumption, it suffices to show that $\DF{f}(x) = \DF{f_1}(x)$,
  which is equivalent to the following:
  for all $\epsilon > 0$, there exists $\delta > 0$ such that for all $h \in \bbR^d$,
  \begin{align*}
    0 < \|h\| < \delta \quad\implies\quad \frac{\| f(x+h) - f(x) - \DF{f_1}(x) \cdot h \|}{\|h\|} < \epsilon,
  \end{align*}
  where $\|\cdot\|$ denotes the $\ell_2$-norm.
  To show this, consider any $\epsilon > 0$.
  Since $\DF{f_i}(x) \neq \bot$ (by assumption), there is $\delta_i > 0$ for each $i \in [n]$ such that for all $h \in \bbR^d$, 
  \begin{align*}
    0 < \|h\| < \delta_i \quad\implies\quad \frac{\| f_i(x+h) - f_i(x) - \DF{f_i}(x) \cdot h \|}{\|h\|}
    = \frac{\| f_i(x+h) - f_1(x) - \DF{f_1}(x) \cdot h \|}{\|h\|}
    < \epsilon,
  \end{align*}
  where the equality is by assumption.
  Choose $0 < \delta < \min\{\delta_i \mid i \in [n]\}$ such that $\{x + h \mid \|h\| < \delta\} \subseteq U$,
  which is possible because $U$ is an open neighborhood of $x$.
  Then, for all $h \in \bbR^d$, $0 < \|h\| < \delta$ implies that
  \begin{align*}
    \frac{\| f(x+h) - f(x) - \DF{f_1}(x) \cdot h \|}{\|h\|} 
    = \frac{\| f_j(x+h) - f_1(x) - \DF{f_1}(x) \cdot h \|}{\|h\|} < \epsilon
  \end{align*}
  for some $j \in [n]$, where the equality is by assumption and $x+h \in U$ and the inequality is by $\delta < \delta_j$.
  This proves $\DF{f}(x) = \DF{f_1}(x)$ as desired.
\end{proof}

\subsection{Lemmas (Technical: Part 1)}
\label{sec:pf-apd}

In this subsection, we formally define the partial derivative ${\APDF{z_L}}/{\partial z_{l,i}} \in \bbR^{N_L}$ of $z_L$ with respect to $z_{l,i}$
that {\em reverse-mode} automatic differentiation computes (as a byproduct of computing $\ADF{z_L}$).
To do so, we fix $l' \in [L]$ and $w' \in \bbR^W$,
and define ${\APDF{z_L}}/{\partial z_{l',i}} \in \bbR^{N_L}$ at $w'$ ($i \in [N_{l'}]$)
in a similar way we defined $\ADF{z_L}$ in \edit{\Cref{sec:pf-ad}}. 

We first define a program $\code{\ttQ}$ (different from $\ttP$ in \edit{\Cref{sec:pf-ad}}) 
that represents a function from ${\bbR^{N_{l'}}}$ to $\bbR$ as follows:
\begin{align*}
  \code{\ttQ} ::= \code{r} \;|\; \code{\ttx_{i}} \;|\; \code{f \ttlp \ttQ_1, \ldots, \ttQ_n \ttrp}
\end{align*}
where $\code{r} \in \bbR$, $i \in [N_{l'}]$, $f \in \{ \tau_{l,i}, \sigma_{l,i} \mid (l,i) \in \Idx, l > l' \}$, and $n \in \bbN$.
This definition says that a program $\ttQ$ can be either a real-valued constant $r$, a real-valued variable $\ttx_{i}$
denoting the neuron $z_{l',i}$,
or the application of a function $f : \bbR^n \to \bbR$ to subprograms $\smash{\ttQ_1}, \ldots, \smash{\ttQ_n}$.
We focus on particular programs $\smash{\code{\ttQ_{y_{l,i}}}}$ and $\smash{\code{\ttQ_{z_{l,i}}}}$ ($l > l'$)
that represent the neurons $y_{l,i}$ and $z_{l,i}$ but as functions of the neurons $z_{l',1}, \ldots, z_{l',N_{l'}}$
(instead of functions of parameters $w_{1,1}, w_{1,2}, \ldots, w_{L,W_L}$).
These programs are defined in a canonical way as follows:
\begin{align*}
  \code{\ttQ_{y_{l,i}}}
  &\defeq \code{\tau_{l,i} \ttlp \ttQ_{z_{l-1,1}}, \ldots, \ttQ_{z_{l-1,N_{l-1}}}, w'_{l,1}, \ldots, w'_{l,W_l} \ttrp},
  \\
  \code{\ttQ_{z_{l,i}}}
  &\defeq \code{\sigma_{l,i} \ttlp \ttQ_{y_{l,i}} \ttrp},
\end{align*}
where ${\code{\ttQ_{z_{l',i}}}} \defeq {\ttx_{i}}$ for $i \in [N_{l'}]$ represents the projection function from $\bbR^{N_{l'}}$ to $\bbR$.
Note that $w'_{l,j}$ in the above equation is not a variable but a constant,
while $\ttx_i$ in the definition of $\ttQ_{z_{l',i}}$ is a variable.

Given a program $\code{\ttQ}$, we define the function $\sem{\code{\ttQ}} : \bbR^{N_{l'}} \to \bbR$ that  $\ttQ$ represents,
and the function $\semad{\code{\ttQ}} : \bbR^{N_{l'}} \to \bbR^{1 \times N_{l'}}$
that reverse-mode automatic differentiation computes for $\ttQ$ (as a byproduct of computing other derivatives):
\begin{align*}
  \sem{\code{r}}(x) &\defeq r,
  \\
  \sem{\code{\ttx_{i}}}(x) &\defeq x_{i},
  \\
  \sem{\code{f \ttlp \ttQ_1, \ldots, \ttQ_n \ttrp}}(x)
  & \defeq f\big(\sem{\code{\ttQ_1}}(x), \ldots, \sem{\code{\ttQ_n}}(x)\big),
  \\
  \semad{\code{r}}(x) &\defeq \mathbb{0}, 
  \\
  \semad{\code{\ttx_{i}}}(x) &\defeq \mathbb{1}_{i}, 
  \\
  \!\!
  \semad{\code{f \ttlp \ttQ_1, \ldots, \ttQ_n \ttrp}}(x)
  & \defeq \adf{f}\big(\sem{\code{\ttQ_1}}(x), \ldots, \sem{\code{\ttQ_n}}(x)\big)
  \cdot \big[ \semad{\code{\ttQ_1}}(x) \,\big/ \cdots \big/\, \semad{\code{\ttQ_n}}(x) \big].
\end{align*}
Here $(x_1, \ldots, x_{N_{l'}}) \defeq x$ denote the scalar values of $x$,
the notation $\mathbb{0}, \mathbb{1}_{i} \in \bbR^{1 \times W}$ denote the zero matrix
and the matrix whose entries are all zeros except for a single one at the $i$-th entry,
$\adf{f} : \bbR^n \to \bbR^{1 \times n}$ denotes a ``derivative'' of $f$ used by automatic differentiation,
and $[M_1 \,/\cdots/\, M_n]$ denotes the matrix that stacks up matrices $M_1, \ldots, M_n$ vertically.
Note that the definitions of $\sem{\ttQ}$ and $\semad{\code{\ttQ}}$ are almost the same as that of $\sem{\ttP}$ and $\semad{\code{\ttP}}$ in \edit{\Cref{sec:pf-ad}}. 

Using the above definitions, ${\APDF{z_L}}/{\partial z_{l',i}}$ at $w'$ for $i \in [N_{l'}]$
(i.e., the partial derivative of $z_L$ with respect to $z_{l',i}$ at $w'$ that reverse-mode automatic differentiation computes)
can be defined as follows:
\begin{align*}
  \frac{\APDF{z_L}}{\partial z_{l',i}} \text{ at } w'
  \defeq \Big[ \semad{\code{\ttQ_{z_{L,1}}}}(z_{l'}(w')) \,\big/ \cdots \big/\, \semad{\code{\ttQ_{z_{L,N_L}}}}(z_{l'}(w')) \Big]_{1:N_L, \,i}
  \in \bbR^{N_L}.
\end{align*}

\Cref{lem:apdf-dextgamma} (shown below) shows that
${\APDF{z_L}}/{\partial z_{l,i}}$ can be expressed in terms of $\ext{z}_{l+1}^\gamma$ (defined in \Cref{sec:pf-inc}),
as $\ADF{z_L}$ can be expressed in terms of $z_L^\gamma$ (\Cref{lem:r-gamma-adf}).
We will rely on this lemma in the rest of this section, when working with ${\APDF{z_L}}/{\partial z_{l,i}}$.

\begin{lemma}
  \label{lem:apdf-dextgamma}
  Let $\gamma \in \Gamma$. Then, for all $l \in [L]$ and $w \in \mcR^\gamma$,
  \begin{align*}
    \frac{\APDF{z_L}}{\partial z_{l,i}} \text{ at } w
    &= \PDF{\ext{z}_{l+1}^\gamma}{i}\big( z_l(w), w_{l+1}, \ldots, w_L \big).
  \end{align*}
\end{lemma}
\begin{proof}
  The proof is similar to \Cref{lem:r-gamma-adf}, except that it uses \Cref{lem:extz-extzgamma-equiv} instead of \Cref{lem:r-gamma-f-df};
  thus, we omit it.
\end{proof}

\subsection{Lemmas (Technical: Part 2)}

\begin{lemma}
  \label{lem:stdderiv-part1-1}
  Let $w \in \bbR^W$.
  Suppose that for all $(l,i) \in \Idx$,
  $y_{l,i}(w) \in \ndf{\sigma_{l,i}}$ implies that
  \[ \PDF{\ext{z}_{l+1}^\gamma}{i}(z_l(w), w_{l+1}, \ldots, w_L) = (0, \ldots, 0) \]
  for the $\gamma \in \Gamma$ with $w \in \mcR^\gamma$.
  Then, for all $l \in [L+1]$ and $\gamma_1, \gamma_2 \in \Gamma$ with $w \in \clR^{\gamma_1} \cap \clR^{\gamma_2}$,
  \[ \DF{\ext{z}_{l}^{\gamma_1}}(z_{l-1}(w), w_{l}, \ldots, w_L)
  = \DF{\ext{z}_{l}^{\gamma_2}}(z_{l-1}(w), w_{l}, \ldots, w_L).\]
\end{lemma}
\begin{proof}
  The proof is similar to that of \Cref{lem:djz-gamma-indep},
  except that this lemma assumes that certain partial derivatives are all zero
  while \Cref{lem:djz-gamma-indep} derives this assumption (in addition to proving the conclusion of this lemma).
  Let $w \in \bbR^W$ that satisfies the assumption of this lemma.
  The proof proceeds by induction on $l$ (starting from $l = L+1$).

  {\bf Case $l = L+1$.}
  In this case, $\ext{z}_{L+1}^\gamma$ is the identity function for all $\gamma \in \Gamma$.
  Hence, the conclusion clearly holds.

  {\bf Case $l < L+1$.}
  For simple notation, let $x \defeq (z_{l-1}(w), w_l, \ldots, w_L)$ and $x' \defeq (z_{l}(w), w_{l+1}, \ldots, w_L)$.  
  Observe that the following hold for any $\gamma \in \Gamma$ with $w \in \smash{\clR^{\gamma}}$,
  due to \Cref{eq:tmp2-1,eq:tmp2-2,eq:tmp2-3} in the proof of \Cref{lem:djz-gamma-indep}:
  \begin{align*}
    \DF{\ext{z}_l^\gamma}(x)
    &= \DF{\ext{z}_{l+1}^\gamma}\big((\ext{\sigma}_l^\gamma \circ \ext{\tau}_l)(x)\big)
    \cdot \DF{\ext{\sigma}_l^\gamma}\big(\ext{\tau}_l(x)\big) \cdot \DF{\ext{\tau}_l}(x),
    \\
    \Big(\DF{\ext{z}_{l+1}^\gamma}((\ext{\sigma}_l^\gamma \circ \ext{\tau}_l)(x))
    \cdot \DF{\ext{\sigma}_l^\gamma}(\ext{\tau}_l(x)) \Big)_{*,\, i}
    &=
    \PDF{\ext{z}_{l+1}^{\gamma}}{i} (x')
    \cdot
    \begin{cases}
      \DF{\sigma_{l,i}^{\gamma(l,i)}}\big(y_{l,i}(w)\big) & \text{\!\!if $i \leq N_l$}
      \\
      1 & \text{\!\!if $i > N_l$}.
    \end{cases}
  \end{align*}
  Using this observation, we prove the conclusion for $l$.
  Let $\gamma_1, \gamma_2 \in \Gamma$ with $w \in \smash{\clR^{\gamma_1} \cap \clR^{\gamma_2}}$.
  We want to show $\smash{\DF{\ext{z}_l^{\gamma_1}}}(x) = \smash{\DF{\ext{z}_l^{\gamma_2}}}(x)$.
  By induction hypothesis on $l+1$, we obtain $\DF{\ext{z}_{l+1}^{\gamma_1}}(x') = \DF{\ext{z}_{l+1}^{\gamma_2}}(x')$.
  From this and the above equation, it suffices to show the following claim for all $i \in [N_l]$:
  \begin{align*}
    & \PDF{\ext{z}_{l+1}^{\gamma_1}}{i}(x') 
    \cdot \DF{\sigma_{l,i}^{\gamma_1(l,i)}}\big(y_{l,i}(w)\big)
    =
    \PDF{\ext{z}_{l+1}^{\gamma_1}}{i}(x') 
    \cdot \DF{\sigma_{l,i}^{\gamma_2(l,i)}}\big(y_{l,i}(w)\big).
  \end{align*}
  Let $i \in [N_l]$.
  We prove this claim by case analysis on $i$.

  {\it Subcase 1: $y_{l,i}(w) \in \ndf{\sigma_{l,i}}$.}
  Observe that for the $\gamma \in \Gamma$ with $w \in \mcR^\gamma$, we have
  \[ \PDF{\ext{z}_{l+1}^{\gamma_1}}{i}(x') = \PDF{\ext{z}_{l+1}^{\gamma}}{i}(x') = (0, \ldots, 0), \]
  where the first equality is by induction hypothesis on $l+1$ with $w \in \mcR^\gamma \subseteq \clR^\gamma$,
  and the second equality by assumption with $y_{l,i}(w) \in \ndf{\sigma_{l,i}}$.
  This directly implies the claim.

  {\it Subcase 2: $y_{l,i}(w) \notin \ndf{\sigma_{l,i}}$.}
  To show the claim, it suffices to show that for all $j \in [2]$,
  \begin{align*}
    \DF{\smash{\sigma_{l,i}^{\gamma_j(l,i)}}}\big(y_{l,i}(w)\big) = \DF{\sigma_{l,i}}\big(y_{l,i}(w)\big).
  \end{align*}
  This is exactly the same as \Cref{eq:tmp2-6} in the proof of \Cref{lem:djz-gamma-indep},
  and we can prove this in the exact same way as before.
  Therefore, the claim holds and this completes the proof.
\end{proof}

\begin{lemma}
  \label{lem:stdderiv-part1-2}
  Let $w \in \bbR^W$.
  Suppose that for all $\gamma_1, \gamma_2 \in \Gamma$ with $w \in \clR^{\gamma_1} \cap \clR^{\gamma_2}$,
  \[ \DF{z_L^{\gamma_1}}(w) = \DF{z_L^{\gamma_2}}(w).\]
  Then, $z_L$ is differentiable at $w$.
\end{lemma}
\begin{proof}
  Consider the setup of this lemma.
  To apply \Cref{lem:func-merge-diffl}, we show the following claims for $\Gamma_w \defeq \{\gamma \in \Gamma \mid w \in \clR^\gamma\}$:
  \begin{itemize}
  \item[(i)] For some open neighborhood $U \subseteq \bbR^W$ of $w$,
    if $w' \in U$, then $z_L(w') = z_L^\gamma(w')$ for some $\gamma \in \Gamma_w$.
  \item[(ii)] $z_L(w) = z_L^\gamma(w)$ for all $\gamma \in \Gamma_w$.
  \item[(iii)] $\DF{z_L^{\gamma_1}}(w) = \DF{z_L^{\gamma_2}}(w) \neq \bot$ for all $\gamma_1, \gamma_2 \in \Gamma_w$.
  \end{itemize}
  If these claims hold, then \Cref{lem:func-merge-diffl} implies that $\DF{z_L}(w) \neq \bot$ (i.e., $z_L$ is differentiable at $w$).
  So what remains is to show these claims.
  First, (iii) follows from the assumption of this lemma and that $z_L^\gamma$ is analytic for all $\gamma \in \Gamma$.
  Second, (ii) follows from \Cref{lem:exttau-extsigma-one-step}.
  Finally, (i) holds as follows.
  Consider any $\gamma \in \Gamma \setminus \Gamma_w$.
  Then, by $w \notin \clR^\gamma$ and the definition of $\clR^\gamma$, there is $(l,i) \in \Idx$ such that
  $y_{l,i}(w) \in A$ and $A \cap \smash{\mcI_{l,i}^{\gamma(l,i)}} = \emptyset$ for some open $A \subseteq \bbR$.
  Since $y_{l,i}$ is continuous and $\mcR^\gamma \subseteq y_{l,i}^{-1}(\smash{\mcI_{l,i}^{\gamma(l,i)}})$,
  the set $U^\gamma \defeq y_{l,i}^{-1}(A)$ is an open neighborhood of $w$ such that
  $U^\gamma \cap \mcR^\gamma = \emptyset$.
  We now define \[ U \defeq \bigcap_{\gamma \in \Gamma \setminus \Gamma_w} U^\gamma. \]
  Then, because $\Gamma$ is finite, $U$ is an open neighborhood of $w$ such that
  $U \cap \smash{\bigcup_{\gamma \in \Gamma \setminus \Gamma_w}} \mcR^\gamma = \emptyset$.
  Using this, we obtain (i) as follows:
  for any $w' \in U$, we have $w' \notin \smash{\bigcup_{\gamma \in \Gamma \setminus \Gamma_w}} \mcR^\gamma$
  and so $w' \in \mcR^\gamma$ for some $\gamma \in \Gamma_w$ (by \Cref{lem:r-gamma-partition});
  this implies that $z_L(w') = z_L^\gamma(w')$ (by \Cref{lem:r-gamma-f-df}).
  This completes the proof.
\end{proof}

\subsection{\Cref{thm:cor-deriv-bias,thm:cor-deriv-nobias} (Main Lemmas)}

\begin{lemma}
  \label{lem:stdderiv-main-1}
  Let $w \in \bbR^W$.
  Suppose that the following holds:
  \begin{itemize}
  \item For all $(l,i) \in \Idx$, $y_{l,i}(w) \in \ndf{\sigma_{l,i}}$ implies that
    $\PDF{\ext{z}_{l+1}^\gamma}{i}(z_l(w), w_{l+1}, \ldots, w_L) = \vec{0}$ for the $\gamma \in \Gamma$ with $w \in \mcR^\gamma$.
  \end{itemize}
  Then, we have the following:
  \begin{itemize}
  \item $w \notin \ndfR{z_L}$ (i.e., $z_L$ is differentiable at $w$).
  \end{itemize}
\end{lemma}
\begin{proof}
  Consider the setup of this lemma.
  For all $\gamma_1, \gamma_2 \in \Gamma$ with $w \in \clR^{\gamma_1} \cap \clR^{\gamma_2}$,
  \[
  \DF{z_L^{\gamma_1}}(w)
  = \DF{\ext{z}_{1}^{\gamma_1}}(z_{0}(w), w_{1}, \ldots, w_L)
  = \DF{\ext{z}_{1}^{\gamma_2}}(z_{0}(w), w_{1}, \ldots, w_L)
  = \DF{z_L^{\gamma_2}}(w),
  \]
  where the first and last equalities are by \Cref{lem:extz-z-equiv},
  and the second equality is by \Cref{lem:stdderiv-part1-1}.
  Then, by applying \Cref{lem:stdderiv-part1-2}, we obtain that $z_L$ is differentiable at $w$, as desired.
\end{proof}

\begin{lemma}
  \label{lem:stdderiv-main-2}
  Let $w \in \bbR^W$.
  Suppose that the following hold:
  \begin{itemize}
  \item $w \notin \ndfR{z_L}$ (i.e., $z_L$ is differentiable at $w$).
  \item For all $(l,i) \in \Idx$, $y_{l,i}(w) \in \ndf{\sigma_{l,i}}$ implies that $\tau_l$ has bias parameters.
  \end{itemize}
  Then, we have the following:
  \begin{itemize}
  \item $w \notin \ndfR{z_L} \cup \incR{z_L}$ (i.e., $\ADF{z_L}(w) = \DF{z_L}(w) \neq \bot$).
  \item For all $(l,i) \in \Idx$, $y_{l,i}(w) \in \ndf{\sigma_{l,i}}$ implies that
    $\PDF{\ext{z}_{l+1}^\gamma}{i}(z_l(w), w_{l+1}, \ldots, w_L) = \vec{0}$ for the $\gamma \in \Gamma$ with $w \in \mcR^\gamma$.
  \end{itemize}
\end{lemma}
\begin{proof}
  Consider the setup in the statement.
  By exactly following the proof of \Cref{thm:inc-zero-bias} (given in \Cref{sec:pf-inc-zero-bias-main}) under this setup,
  we obtain the conclusion of \Cref{thm:inc-zero-bias}: $\ADF{z_L}(w) = \DF{z_L}(w)$, which implies the first conclusion of this lemma.
  Moreover, the second conclusion was already shown in the proof of \Cref{lem:djz-gamma-indep}
  (which has the same assumption as this lemma),
  especially in Subcase 1 of Case $l < L+1$ in the proof.
  This completes the proof.
\end{proof}

\subsection{\Cref{thm:cor-deriv-bias,thm:cor-deriv-nobias} (Main Proofs)}

{\bf \Cref{thm:cor-deriv-bias}.}
{\it
  If $z_L$ has bias parameters, then the following are equivalent for all $w \in \bbR^W$.
  \begin{itemize}
  \item $z_L$ is non-differentiable at $w$.
  \item $y_{l,i}(w) \in \ndf{\sigma_{l,i}}$
    and $\smash{\APDF{z_L}} / \partial z_{l,i} \neq \smash{\vec{0}}$ at $w$ for some $(l,i)\in\Idx$.
  \end{itemize}
}
\begin{proof}
  Let $w \in \bbR^W$. Suppose that $z_L$ has bias parameters.
  Then, by \Cref{lem:stdderiv-main-1,lem:stdderiv-main-2}, the following are equivalent:
  \begin{itemize}
  \item[(i)] $w \notin \ndfR{z_L}$ (i.e., $z_L$ is differentiable at $w$).
  \item[(ii)] For all $(l,i) \in \Idx$, $y_{l,i}(w) \in \ndf{\sigma_{l,i}}$ implies that
    $\PDF{\ext{z}_{l+1}^\gamma}{i}(z_l(w), w_{l+1}, \ldots, w_L) = \vec{0}$ for the $\gamma \in \Gamma$ with $w \in \mcR^\gamma$.
  \end{itemize}
  By taking the negation of (i)-(ii) and applying \Cref{lem:apdf-dextgamma} to (ii),
  we obtain the conclusion.
\end{proof}

{\bf \Cref{thm:cor-deriv-nobias}.}
{\it
  Let $w \in \bbR^W$.
  If $y_{l,i}(w) \notin \ndf{\sigma_{l,i}}$ for all $(l,i) \in \Idx$ such that
  $\tau_l$ does not have bias parameters or $\smash{\APDF{z_L}} / \partial z_{l,i} \neq \smash{\vec{0} }$ at $w$,
  then
  \begin{align*}
    \ADF{z_L}(w) = \DF{z_L}(w) \neq \bot.
    \\[-1.2em]
  \end{align*}
}
\begin{proof}
  Let $w \in \bbR^W$. Suppose that it satisfies the given assumption, which is equivalent to the following by \Cref{lem:apdf-dextgamma}:
  \begin{itemize}
  \item[] For all $(l,i) \in \Idx$, $y_{l,i}(w) \in \ndf{\sigma_{l,i}}$ implies that
    \begin{itemize}
    \item[(i)] $\tau_l$ has bias parameters, and
    \item[(ii)] $\PDF{\ext{z}_{l+1}^\gamma}{i}(z_l(w), w_{l+1}, \ldots, w_L) = \vec{0}$ for the $\gamma \in \Gamma$ with $w \in \mcR^\gamma$.
    \end{itemize}
  \end{itemize}
  First, by \Cref{lem:stdderiv-main-1} with (ii), we have
  \begin{itemize}
  \item[]
    \begin{itemize}
    \item[(iii)] $w \notin \ndfR{z_L}$ (i.e., $z_L$ is differentiable at $w$).
    \end{itemize}
  \end{itemize}
  Next, by \Cref{lem:stdderiv-main-2} with (i) and (iii), we have the conclusion:
  \begin{itemize}
  \item[]
    \begin{itemize}
    \item[] $w \notin \ndfR{z_L} \cup \incR{z_L}$ (i.e., $\ADF{z_L} = \DF{z_L}(w) \neq \bot$).
      \qedhere
    \end{itemize}
  \end{itemize}
\end{proof}


\clearpage
\section{Computation of Clarke Subderivatives} 
\label{sec:pf-subdiff}

\subsection{Lemmas (Basic)}

\begin{definition}
  Let $A \subseteq \bbR^n$ and $x \in \bbR^n$ (where $A$ does not need to contain $x$).
  For $B \subseteq \bbR^n$, we say that {\em $A$ has $B$-directions around $x$}
  if for all $b \in B$, there is $\delta > 0$ such that \( \{x + t b \mid t \in (0, \delta)\} \subseteq A. \)
  We say that {\em $A$ has sufficient directions around $x$}
  if $A$ has $B$-directions around $x$ for some $B \subseteq \bbR^n$ with \(  \spann(B) = \bbR^n, \)
  where $\spann(B) \defeq \{\smash{\sum_{i=1}^k} t_i  b_i \mid k \in \bbN, t_i \in \bbR, b_i \in B\}$ denotes the span of $B$.
\end{definition}

\begin{lemma}
  \label{lem:dir-basic}
  Let $A \subseteq \bbR^n$ and $x \in \bbR^n$.
  \begin{enumerate}
  \item If $x \in \intr(A)$, then $A$ has $\bbR^n$-directions around $x$.
  \item Let $\alpha \in \{\pm 1\}$, $\epsilon \in \bbR_{>0} \cup \{\infty\}$, and $f : \bbR^n \to \bbR$.
    If $f$ is differentiable at $x$ and 
    \[A = \{y \in \bbR^n \mid \alpha \cdot (f(y) - f(x)) \in (0, \epsilon)\},\]
    then $A$ has $B$-directions around $x$ for
    \[B \defeq \{y \in \bbR^n \mid \alpha \cdot (\DF{f}(x) \cdot y) \in (0, \infty) \}.\]
  \item Let $A',B \subseteq \bbR^n$.
    If $A$ has $B$-directions around $x$ and  $A \subseteq A'$,
    then $A'$ has $B$-directions around $x$.
  \item Let $A', B, B' \subseteq \bbR^n$.
    If $A$ has $B$-directions around $x$ and $A'$ has $B'$-directions around $x$,
    then $(A \cap A')$ has $(B \cap B')$-directions around $x$.
  \item If $A$ has $B$-directions around $x$ for some nonempty, open $B \subseteq \bbR^n$,
    then $A$ has sufficient directions around $x$.
  \end{enumerate}
\end{lemma}
\begin{proof}
  The proofs of (1), (3), and (4) are straightforward, so we omit them.

  \paragraph{\bf Proof of (2).}
  Consider the setup stated above.
  Assume that $f$ is differentiable at $x$, and let $b \in B$.
  We want to show there is $\delta > 0$ such that
  $\{x + t b \mid t \in (0, \delta) \} \subseteq A$.
  We show this when $\alpha = 1$; we omit the case when $\alpha =-1$, as the proof is similar.
  Observe that since $f$ is differentiable at $x$, there is $\delta'>0$ such that for all $h \in \bbR^n$,
  \begin{align}
    \label{eq:dir-basic-pf-1}
    0 < \|h\| < \delta' \implies \frac{|f(x+h) - f(x) - \DF{f}(x) \cdot h|}{\|h\|} < \frac{\DF{f}(x) \cdot b}{2 \|b\|},
  \end{align}
  where $\| \cdot \|$ denotes the $\ell_2$-norm.
  Here we used $\DF{f}(x) \cdot b >0$ and $\|b\|>0$, which hold by $b \in B$ and the definition of $B$.
    
  We claim that $\{x + t  b \mid t \in (0, \delta) \} \subseteq A$ holds for the following choice of $\delta$:
  \[ \delta \defeq \min\Big\{ \frac{\delta'}{\|b\|},\, \frac{2 \epsilon}{3 (\DF{f}(x) \cdot b)} \Big\} > 0. \]
  To show this, consider any $t \in (0, \delta)$.
  It suffices to show $x+ t b \in A$.
  Observe that for $h = t  b$, we have
  $0 < \|h\| = \|t b\| < \delta  \|b\| \leq (\delta' / \|b\|) \cdot \|b\| = \delta'$. Hence, \Cref{eq:dir-basic-pf-1} implies that
  \begin{gather*}
    \textstyle
    \big|f(x + t b ) - f(x) - \DF{f}(x) \cdot (t b)\big|
    < \|t b\| \cdot  \frac{1}{2 \|b\|} (\DF{f}(x) \cdot b) = \frac{1}{2} (\DF{f}(x) \cdot b)t,
    \\ \textstyle
    0 < \frac{1}{2}  (\DF{f}(x) \cdot b)t
    < f(x + tb) - f(x)
    < \frac{3}{2} (\DF{f}(x) \cdot b)t  < \epsilon,
  \end{gather*}
  where the second line uses $\DF{f}(x) \cdot b >0$ and $t < \delta \leq \frac{2}{3} \epsilon / {(\DF{f}(x) \cdot b)}$.
  From this, and by the definition of $A$, we have $x+ tb \in A$ as desired.
  
  \paragraph{\bf Proof of (5).}
  This follows from the fact that
  the span of any nonempty, open set in $\bbR^n$ is $\bbR^n$.
\end{proof}

\begin{lemma}
  \label{lem:suff-dir-df}
  Let $f, g : \bbR^n \to \bbR^m$, $A \subseteq \bbR^n$, and $x \in \bbR^n$.
  Suppose that $f = g$ on $A \cup \{x\}$, $A$ has sufficient directions around $x$, and $f$ and $g$ are differentiable at $x$.
  Then, \[ \DF{f}(x) = \DF{g}(x). \]
\end{lemma}
\begin{proof}
  Consider the setup stated above.
  Since $A$ has sufficient directions around $x$, there is $B \subseteq \bbR^n$ such that
  $A$ has $B$-directions around $x$ and $\spann(B) = \bbR^n$.
  We claim that
  $\DF{f}(x) \cdot b = \DF{g}(x) \cdot b$ {for all $b \in B$}.
  Note that this claim implies the conclusion:
  by the claim and $\spann(B) = \bbR^n$, we have $\DF{f}(x) \cdot v = \DF{g}(x) \cdot v$ for all $v \in \bbR^n$,
  and so $\DF{f}(x) = \DF{g}(x)$.
  
  We now prove the above claim. Let $b \in B$.
  Note that it suffices to show:
  \[ \|(\DF{f}(x) - \DF{g}(x)) \cdot b\| < \|b\|\epsilon \qquad\text{for all $\epsilon > 0$}, \]
  since this implies $(\DF{f}(x) - \DF{g}(x)) \cdot b = 0$,
  where $\|\cdot\|$ denotes the $\ell_2$-norm.
  Let $\epsilon > 0$.
  Since $f$ and $g$ are differentiable at $x$,
  there is $\delta > 0$ such that for any $t \in (0,\delta)$,
  \begin{align}
    \label{eq:suff-dir-df-pf-1}
    \frac{\|f(x+tb) - f(x) - \DF{f}(x) \cdot (tb)\|}{\|tb\|} < \frac{\epsilon}{2}, \qquad
    \frac{\|g(x+tb) - g(x) - \DF{g}(x) \cdot (tb)\|}{\|tb\|} < \frac{\epsilon}{2}.
  \end{align}
  Also, since $b \in B$, there is $\delta'>0$ such that $\{x + tb \mid t \in (0, \delta')\} \subseteq A$.
  Fix $t \defeq \min\{\delta, \delta'\}/2 > 0$.
  Then, we obtain the desired equation based on this $t$:
  \begin{align*}
    \big\|\big(\DF{f}(x) - \DF{g}(x)\big) \cdot b\big\|
    &= \frac{1}{t} \big\|\big(\DF{f}(x) - \DF{g}(x)\big) \cdot (tb) \big\|
    \\
    &= \frac{1}{t} \big\|\big(f(x+tb) - f(x) - \DF{f}(x) \cdot (tb)\big) - \big(f(x+tb) - f(x) - \DF{g}(x) \cdot (tb)\big) \big\|
    \\
    &= \frac{1}{t} \big\|\big(f(x+tb) - f(x) - \DF{f}(x) \cdot (tb)\big) - \big(g(x+tb) - g(x) - \DF{g}(x) \cdot (tb)\big) \big\|
    \\
    &= \frac{1}{t} \Big(\big\|f(x+tb) - f(x) - \DF{f}(x) \cdot (tb) \big\| + \big\|g(x+tb) - g(x) - \DF{g}(x) \cdot (tb)\big\| \Big)
    \\
    &= \frac{1}{t} \cdot \Big(\frac{\epsilon}{2} + \frac{\epsilon}{2}\Big) \|tb\|
    = \|b\| \epsilon,
  \end{align*}
  where the third line uses that $f=g$ on $A \cup \{x\}$ (by assumption) and $x+tb \in A$ (by $t < \delta'$),
  and the last line uses \Cref{eq:suff-dir-df-pf-1} (by $t < \delta$).
\end{proof}



\begin{lemma}
  \label{lem:ineq-set-sat}
  Let $n \in \bbN$, $\{d_i \in \bbN\}_{i \in [n]}$ such that $d_1 < \cdots <d_n$, and $\{f_i : \bbR^{d_i-1} \to \bbR\}_{i \in [n]}$.
  Then, for any $c \in \bbR^n$, there is $u \in \bbR^{d_n}$ such that
  \begin{align*}
    \text{$f_i(u_1, \ldots, u_{d_i-1}) + u_{d_i} = c_i$}
    \qquad\text{for all $i \in [n]$}.
  \end{align*}
\end{lemma}
\begin{proof}
  The proof proceeds by induction on $n$.

  \paragraph{\bf Case $n=1$.}
  For any $c \in \bbR$,
  $u \defeq (0, \ldots, 0, c - f_1(0, \ldots, 0)) \in \bbR^{d_1}$
  satisfies the desired equation.

  \paragraph{\bf Case $n>1$.}
  Let $c \in \bbR^n$.
  By induction hypothesis on $n-1$,
  there is $v \in \bbR^{d_{n-1}}$ such that
  $f_i(v_1, \ldots, v_{d_i-1}) + v_{d_i} = c_i$ for all $i \in [n-1]$.
  Define $u \defeq (v, 0, \ldots, 0, c_n - f_n(v, 0, \ldots, 0)) \in \bbR^{d_n}$.
  Then, $u$ satisfies the desired equations, since $(u_1, \ldots, u_{d_{n-1}}) = v$ by $d_{n-1} < d_n$.
\end{proof}

\subsection{Lemmas (Technical)}

In the following subsections, we consider a piecewise-$C^1$ (not piecewise-differentiable) representation of each $\sigma_{l,i}$,
using the same notation in the previous sections.
Formally, we make the following definitions.

\begin{definition*}
  \label{def:minrep-sigma-c1}
  For each $(l,i) \in \Idx$, let \[ \{(\mcI_{l,i}^k, \sigma_{l,i}^k) \}_{k \in [K_{l,i}]} \]
  be a piecewise-$C^1$ representation of $\sigma_{l,i} : \bbR \to \bbR$
  that defines $\adf{\sigma_{l,i}}$, 
  where $K_{l,i} \in \bbN$, $\smash{\mcI_{l,i}^k} \subseteq \bbR$, and $\smash{\sigma_{l,i}^k} : \bbR \to \bbR$.
  We assume that the representation satisfies:
  \begin{gather*}
    \text{$\bigcup_{k \in [K_{l,i}]} \bd(\mcI_{l,i}^k) = \bigcup_{k \in [K_{l,i}]} \pbd(\mcI_{l,i}^k) = \ncdf{\sigma_{l,i}}$},
  \end{gather*}
  where $\ncdf{f} \subseteq \bbR$ denotes the set of real numbers at which $f : \bbR \to \bbR$ is not continuously differentiable.
  If $\adf{\sigma_{l,i}}$ is consistent, we further assume that the representation satisfies the following:
  \begin{align*}
    \text{$\intr(\mcI_{l,i}^k) \neq \emptyset$ for all $k \in [K_{l,i}]$}.
  \end{align*}
  Note that such a representation always exists by \Cref{thm:int-deriv-basic}.
  Based on these new representations $\{(\mcI_{l,i}^k, \sigma_{l,i}^k) \}_{k \in [K_{l,i}]}$,
  we define $\Gamma$, $\mcR^\gamma$, $y_l^\gamma$, $z_l^\gamma$, and $\sigma_l^\gamma$ for $\gamma \in \Gamma$ and $l \in [L]$,
  as we defined them in \Cref{sec:pf-nn};
  we omit their definitions here.
  \qed
\end{definition*}

Since we consider a piecewise-$C^1$ (not piecewise-differentiable) representation of $\sigma_{l,i}$,
we have \Cref{lem:yz-yz-gamma-good-c1} (shown below) that is stronger than \Cref{lem:yz-yz-gamma-good}.
Moreover, \Cref{lem:r-gamma-partition,lem:r-gamma-equiv,lem:r-gamma-f-df,lem:r-gamma-adf} continue to hold under the new representations;
the proofs are exactly the same as before, so we omit them.

\begin{lemma}
  \label{lem:yz-yz-gamma-good-c1}
  For all $l \in [L]$ and $\gamma \in \Gamma$,
  $y_l$ and $z_l$ are continuous, and $\smash{y_l^\gamma}$ and $\smash{z_l^\gamma}$ are $C^1$.
\end{lemma}
\begin{proof}
  The continuity of $y_l$ and $z_l$ follows directly from that
  $\tau_{l'}$, $\pi_{l'}$, and $\sigma_{l',i'}$ are continuous for all $(l',i') \in \Idx$.
  Similarly, the continuous differentiability of $\smash{y_l^\gamma}$ and $\smash{z_l^\gamma}$ follows directly from that
  $\tau_{l'}$, $\pi_{l'}$, and $\smash{\sigma_{l',i'}^{k'}}$ are $C^1$ for all $(l',i') \in \Idx$ and $k' \in [K_{l',i'}]$.
\end{proof}

\subsection{\Cref{thm:clarke-subdiff-bias,thm:clarke-subdiff-nobias} (Main Lemmas)}

\begin{lemma}
  \label{lem:int-r-gamma-suff-dir}
  Let $\gamma \in \Gamma$ and $w \in \mcR^\gamma$.
  Suppose that for all $l \in [L]$,
  if $\tau_l$ does not have bias parameters, then $y_{l,i}(w) \notin \pbd(\smash{\mcI_{l,i}^{\gamma(l,i)}})$ for all $i \in [N_l]$. 
  Also, assume that $\adf{\sigma_{l,i}}$ is consistent for all $(l,i) \in \Idx$.
  Then, \[ \text{$\intr(\mcR^\gamma)$ has sufficient directions around $w$.} \]
\end{lemma}
\begin{proof}
  First, observe that
  \begin{align*}
    \intr(\mcR^\gamma)
    &= \intr\Big( \bigcap_{(l,i) \in \Idx} \big\{w' \in \bbR^W \mid y_{l,i}(w') \in \mcI_{l,i}^{\gamma(l,i)}\big\} \Big)
    \\
    &= \intr\Big( \bigcap_{(l,i) \in \Idx} \big\{w' \in \bbR^W \mid y_{l,i}^\gamma(w') \in \mcI_{l,i}^{\gamma(l,i)}\big\} \Big)
    \\
    &= \bigcap_{(l,i) \in \Idx} \intr\Big( \big\{w' \in \bbR^W \mid y_{l,i}^\gamma(w') \in \mcI_{l,i}^{\gamma(l,i)}\big\} \Big)
    \\
    &\supseteq \bigcap_{(l,i) \in \Idx} A_{l,i}
    \qquad\text{for } A_{l,i} \defeq \big\{w' \in \bbR^W \mid y_{l,i}^\gamma(w') \in \intr(\mcI_{l,i}^{\gamma(l,i)}) \big\},
  \end{align*}
  where the second line uses \Cref{lem:r-gamma-equiv},
  the third line uses that $\intr(U \cap V) = \intr(U) \cap \intr(V)$ for any $U, V \subseteq \bbR^n$,
  and the fourth line uses that $\intr(f^{-1}(U)) \supseteq f^{-1}(\intr(U))$ for any $U \subseteq \bbR^m$ and continuous $f : \bbR^n \to \bbR^m$.
  Note that $\smash{A_{l,i}}$ is open, since $\intr(\smash{\mcI_{l,i}^{\gamma(l,i)}})$ is open
  and $\smash{y_{l,i}^\gamma}$ is continuous (by \Cref{lem:yz-yz-gamma-good-c1}).
  
  Next, we show that it suffices to find some $B_{l,i} \subseteq \bbR^W$ for every $(l,i) \in \Idx$ such that
  \begin{itemize}
  \item[(i)] $A_{l,i}$ has $B_{l,i}$-directions around $w$, and
  \item[(ii)] $\bigcap_{(l,i) \in \Idx} B_{l,i}$ is nonempty and open.
  \end{itemize}
  Suppose that there are such $B_{l,i}$'s.
  By applying \Cref{lem:dir-basic}-(4) to (i), we have
  \[ \textstyle \text{$\bigcap_{(l,i) \in \Idx} A_{l,i}$ has $\bigcap_{(l,i) \in \Idx} B_{l,i}$-directions around $w$.} \]
  By applying \Cref{lem:dir-basic}-(3) to the above and $\bigcap_{(l,i) \in \Idx} A_{l,i} \subseteq \intr(\mcR^\gamma)$, we have
  \[ \textstyle \text{$\intr(\mcR^\gamma)$ has $\bigcap_{(l,i) \in \Idx} B_{l,i}$-directions around $w$.} \]
  By applying \Cref{lem:dir-basic}-(5) to the above and (ii), we obtain the desired conclusion:
  \[ \textstyle \text{$\intr(\mcR^\gamma)$ has sufficient directions around $w$.} \]

  What remains is to show that there is $B_{l,i}$ satisfying (i) and (ii).
  We claim that the $B_{l,i}$ defined below satisfies (i) and (ii):
  \begin{align*}
    B_{l,i} =
    \begin{cases}
      \bbR^W & \text{if } w \in A_{l,i} \\
      \{ v \in \bbR^W \mid \alpha_{l,i} \cdot (\DF{y_{l,i}^\gamma}(w) \cdot v) \in (0, \infty) \} & \text{if } w \notin A_{l,i},
    \end{cases}
  \end{align*}
  where $\alpha_{l,i} \in \{\pm1\}$ is defined as
  \begin{align*}
    \alpha_{l,i} =
    \begin{cases}
      1 & \text{if $w \notin A_{l,i}$ and $y_{l,i}^\gamma(w) = \inf \mcI_{l,i}^{\gamma(l,i)}$} \\
      -1 & \text{if $w \notin A_{l,i}$ and $y_{l,i}^\gamma(w) = \sup \mcI_{l,i}^{\gamma(l,i)}$}.
    \end{cases}
  \end{align*}
  Before proving (i) and (ii), we point out that $B_{l,i}$ is well-defined.
  In particular, $\smash{\DF{y_{l,i}^\gamma}(w)}$ exists since $\smash{y_{l,i}^\gamma}$ is differentiable
  (by \Cref{lem:yz-yz-gamma-good-c1});
  and $\smash{\alpha_{l,i}}$ is well-defined (i.e., the cases in the definition of $\smash{\alpha_{l,i}}$ covers all possible cases)
  since $w \notin \smash{A_{l,i}}$ implies
  \begin{align}
    \label{eq:int-r-gamma-suff-dir-1}
    y_{l,i}^\gamma(w) \in \pbd(\mcI_{l,i}^{\gamma(l,i)}) = \{\inf {\mcI_{l,i}^{\gamma(l,i)}}, \sup {\mcI_{l,i}^{\gamma(l,i)}}\}.
  \end{align}
  Here the equality comes from  that $\smash{\mcI_{l,i}^{\gamma(l,i)}}$ is an interval in $\bbR$,
  and the inclusion comes from:
  \begin{align}
    \label{eq:int-r-gamma-suff-dir-2}
    \smash{y_{l,i}^\gamma}(w) &= y_{l,i}(w),
    &
    \smash{y_{l,i}^\gamma}(w) &\notin \intr(\smash{\mcI_{l,i}^{\gamma(l,i)}}),
    &
    y_{l,i}(w) &\in \smash{\mcI_{l,i}^{\gamma(l,i)}},
  \end{align}
  where the first equation is by \Cref{lem:r-gamma-f-df} and $w \in \mcR^\gamma$,
  the second equation by  $w \notin A_{l,i}$, and the third equation by $w \in \mcR^\gamma$.

  We now prove that the $B_{l,i}$ defined above satisfies (i) and (ii).

  \paragraph{\bf Proof of (i).}
  Consider $(l,i) \in \Idx$.
  If $w \in A_{l,i}$, then $A_{l,i}$ has $\bbR^W$-directions around $w$ by \Cref{lem:dir-basic}-(1),
  since $w \in \intr(A_{l,i}) = A_{l,i}$ (as $A_{l,i}$ is open); hence, (i) holds for this case.
  For the other case, suppose that $w \notin A_{l,i}$.
  Let $\epsilon_{l,i} \in \bbR \cup \{\infty\}$ be the length of the interval $\smash{\mcI_{l,i}^{\gamma(l,i)}}$.
  Then,
  \[
  \epsilon_{l,i} > 0, \qquad A_{l,i} = \{v \in \bbR^W \mid \alpha_{l,i} \cdot (y_{l,i}^\gamma(v) - y_{l,i}^\gamma(w)) \in (0, \epsilon_{l,i}) \}.
  \]
  Here the former holds, since we have $\smash{\intr(\mcI_{l,i}^{\gamma(l,i)})} \neq \emptyset$ 
  (by \Cref{def:minrep-sigma-c1}) and that $\adf{\sigma_{l,i}}$ is consistent (by assumption).
  The latter holds, since $\intr(\smash{\mcI_{l,i}^{\gamma(l,i)}})$ is either
  $(\smash{y_{l,i}^\gamma(w)}, \smash{y_{l,i}^\gamma(w)} + \epsilon_{l,i})$ or $(\smash{y_{l,i}^\gamma(w)} - \epsilon, \smash{y_{l,i}^\gamma(w)})$
  by $w \notin A_{l,i}$ and \Cref{eq:int-r-gamma-suff-dir-1}.
  By these two observations, and since $\smash{y_{l,i}^\gamma}$ is differentiable,
  \Cref{lem:dir-basic}-(2) is applicable to $(A_{l,i}, B_{l,i}, w)$ and directly implies (i).

  \paragraph{\bf Proof of (ii).}
  First, $\smash{\bigcap_{(l,i) \in \Idx}} B_{l,i}$ is open as desired, since every $B_{l,i}$ is open and $\Idx$ is finite.
  Second, we show that $\smash{\bigcap_{(l,i) \in \Idx}} B_{l,i}$ is nonempty.
  Let $\Idx' \defeq \{(l,i) \in \Idx \mid w \notin A_{l,i}\}$.
  By the definition of $B_{l,i}$, what we want to show is that for some $v' \in \bbR^W$,
  \[
  \alpha_{l,i} \cdot (\DF{y_{l,i}^\gamma}(w) \cdot v') > 0 \qquad\text{for all $(l,i) \in \Idx'$}.
  \]
  Since $\alpha_{l,i} \neq 0$ for all $(l,i) \in \Idx'$,
  it suffices to show that for some $v' \in \bbR^W$,
  \begin{align}
    \label{eq:int-r-gamma-suff-dir-3}
    \DF{y_{l,i}^\gamma}(w) \cdot v'  = \alpha_{l,i} \qquad\text{for all $(l,i) \in \Idx'$}.
  \end{align}
  To prove this, we analyze the above equation as follows.
  Consider any $(l,i) \in \Idx'$.
  Then, we have $w \notin A_{l,i}$, which implies $y_{l,i}(w) \in \pbd(\smash{\mcI_{l,i}^{\gamma(l,i)}})$
  by \Cref{eq:int-r-gamma-suff-dir-1,eq:int-r-gamma-suff-dir-2}.
  From this, $\tau_l$ has bias parameters (by assumption). 
  So, for all $v = (v_1, \ldots, v_W) \in \bbR^{W}$, 
  \begin{align}
    \nonumber
    y_{l,i}^\gamma(v)
    &= \tau_{l,i} \big( z_{l-1}^\gamma(v), \pi_l(v) \big)
    \\ \nonumber
    &= \tau_{l,i} \big( z_{l-1}^\gamma(v_1, \ldots, v_{W'}, 0, \ldots, 0),
    (v_{W'+1}, \ldots, v_{W' + W_l}) \big)
    \\
    \label{eq:int-r-gamma-suff-dir-4}
    &= \tau'_{l,i} \big( z_{l-1}^\gamma(v_1, \ldots, v_{W'}, 0, \ldots, 0),
    (v_{W'+1}, \ldots, v_{W' + (W_l - N_l)} \big) + v_{W' + (W_l - N_l + i)},
  \end{align}
  where the second line uses $W' \defeq W_1 + \cdots + W_{l-1}$ and
  the fact that $\smash{z_{l-1}^\gamma}$ depends only on the parameters of $\tau_1, \ldots, \tau_{l-1}$,
  and the third line uses that $\tau_l$ has bias parameters. 
  Let $\psi_{l,i} \defeq W' + (W_l - N_l + i)$.
  Since the first term in  \Cref{eq:int-r-gamma-suff-dir-4} does not depend on $v_{\psi_{l,i}}, \ldots, v_W$,
  the following holds for all $j \geq \psi_{l,i}$:
  \begin{align*}
    \big(\DF{y_{l,i}^\gamma}(w)\big)_{j}
    &
    =
    \begin{cases}
      1 & \text{if $j=\psi_{l,i}$}
      \\
      0 & \text{if $j > \psi_{l,i}$}.
    \end{cases}
  \end{align*}
  From this, the following holds for all $v \in \bbR^W$:
  \begin{align*}
    \DF{y_{l,i}^\gamma}(w) \cdot v
    &= \sum_{j \in [W]} \big(\DF{y_{l,i}^\gamma}(w)\big)_j \cdot v_j
    = f_{l,i}(v_1, \ldots, v_{\psi_{l,i}-1}) + v_{\psi_{l,i}},
  \end{align*}
  where $f_{l,i} : \bbR^{\psi_{l,i}-1} \to \bbR$ is defined as
  $
  f_{l,i}(u) \defeq \smash{\sum_{j \in [\psi_{l,i}-1]}} \smash{\big(\DF{y_{l,i}^\gamma}(w)\big){}_j} \cdot u_j.
  $
  Hence, what we planned to show (i.e., \Cref{eq:int-r-gamma-suff-dir-3} holds for some $v' \in \bbR^W$)
  is equivalent to the following: for some $v' \in \bbR^W$,
  \begin{align}
    \label{eq:int-r-gamma-suff-dir-5}
    f_{l,i}(v'_1, \ldots, v'_{\psi_{l,i}-1}) + v'_{\psi_{l,i}} = \alpha_{l,i}
    \qquad\text{for all $(l,i) \in \Idx'$}.
  \end{align}
  Since $\psi_{l,i} \neq \psi_{l',i'}$ for any $(l,i) \neq (l',i')$,
  \Cref{lem:ineq-set-sat} implies that there is $v' \in \bbR^W$ satisfying \Cref{eq:int-r-gamma-suff-dir-5}.
  This proves (ii), and concludes the proof.
\end{proof}

\begin{lemma}
  \label{lem:ad-consistent}
  Let $\gamma \in \Gamma$ and $w \in \mcR^\gamma$.
  Suppose that $\intr(\mcR^\gamma)$ has sufficient directions around $w$.
  Then,
  \begin{align*}
    \ADF{z_L}(w) =
    \begin{cases}
      \DF{z_L}(w)
      & \text{if\, $\DF{z_L}(w) \neq \bot$}
      \\
      \lim_{n \to \infty} \DF{z_L}(w'_n) \text{ for some $w'_n \to w$} 
      & \text{if\, $\DF{z_L}(w) = \bot$}.
    \end{cases}
  \end{align*}
\end{lemma}
\begin{proof}
  Let $\gamma \in \Gamma$ and $w \in \mcR^\gamma$ such that
  $\intr(\mcR^\gamma)$ has sufficient directions around $w$.
  By \Cref{lem:r-gamma-f-df,lem:r-gamma-adf},
  \begin{align}
    \label{eq:ad-consistent-1}
    z_L(w') &= z_L^\gamma(w') \;\;\land\;\; \ADF{z_L}(w') = \DF{z_L^\gamma}(w') \qquad\text{for all } w' \in \mcR^\gamma.
  \end{align}
  We prove the conclusion for each of the two cases: $\DF{z_L}(w) \neq \bot$ and $\DF{z_L}(w) = \bot$.

  \paragraph{\bf Case 1:} $\DF{z_L}(w) \neq \bot$ (i.e., $z_L$ is differentiable at $w$).
  We want to show \[\ADF{z_L}(w) = \DF{z_L}(w).\]
  This holds as follows: \[ \ADF{z_L}(w) = \DF{z_L^\gamma}(w) = \DF{z_L}(w), \]
  where the first equality is by \Cref{eq:ad-consistent-1},
  and the second equality follows directly from \Cref{lem:suff-dir-df} applied to $(z_L^\gamma, z_L, \mcR^\gamma, w)$.
  Here \Cref{lem:suff-dir-df} is applicable since its preconditions are satisfied:
  $z_L^\gamma$ is differentiable at $w$ (by \Cref{lem:yz-yz-gamma-good-c1});
  $z_L$ is differentiable at $w$ (by assumption);
  $z_L^\gamma = z_L$ on $\intr(\mcR^\gamma) \cup \{w\}$ (by \Cref{eq:ad-consistent-1});
  and $\intr(\mcR^\gamma)$ has sufficient directions around $w$
  (by assumption).

  \paragraph{\bf Case 2:} $\DF{z_L}(w) = \bot$ (i.e., $z_L$ is not differentiable at $w$).
  We want to show:
  \begin{align}
    \label{eq:ad-consistent-2}
    \ADF{z_L}(w) & = \lim_{n \to \infty} \DF{z_L}(w'_n) \qquad\text{for some $w'_n \to w$}. 
  \end{align}
  Since $\intr(\mcR^\gamma)$ has sufficient directions around $w$ (by assumption),
  there is $\{w'_n \in \intr(\mcR^\gamma)\}_{n \in \bbN}$ such that $w'_n \to w$.
  We show that these $w'_n$ satisfy \Cref{eq:ad-consistent-2} as follows:
  \begin{align*}
    \ADF{z_L}(w) = \DF{z_L^\gamma}(w) = \lim_{n \to \infty} \DF{z_L^\gamma}(w'_n) = \lim_{n \to \infty} \DF{z_L}(w'_n),
  \end{align*}
  where the first equality is by \Cref{eq:ad-consistent-1},
  the second equality uses that $\smash{\DF{z_L^\gamma}}$ is continuous
  (by \Cref{lem:yz-yz-gamma-good-c1}),
  and the third equality uses that $\smash{\DF{z_L^\gamma}}(w_n') = \DF{z_L}(w'_n)$ for all $n$
  (since $w'_n \in \intr(\mcR^\gamma)$ and $\smash{z_L^\gamma} = z_L$ on $\mcR^\gamma$ by \Cref{eq:ad-consistent-1}).
  This concludes the proof.
\end{proof}

\subsection{\Cref{thm:clarke-subdiff-bias,thm:clarke-subdiff-nobias} (Main Proofs)}

{\bf \Cref{thm:clarke-subdiff-bias}.}
{\it
  If $z_L$ has bias parameters and $\adf{\sigma_{l,i}}$ is consistent for all $(l,i) \in \Idx$, then for all $w \in \bbR^W$, 
  \begin{align*}
    \ADF{z_L}(w) =
    \begin{cases}
      \DF{z_L}(w) & \text{if $\DF{z_L}(w) \neq \bot$}
      \\
      \begin{array}{@{}l@{}}
        \lim_{n \to \infty} \DF{z_L}(w'_n)
        \;\;\text{for some $w'_n \to w$}
      \end{array}
      & \text{if $\DF{z_L}(w) = \bot$}.
    \end{cases}
  \end{align*}
  This implies that $\ADF{z_L}$ is a Clarke subderivative of~$z_L$.
}
\begin{proof} 
  This theorem is a special case of \Cref{thm:clarke-subdiff-nobias}; we omit the proof.
\end{proof}

{\bf \Cref{thm:clarke-subdiff-nobias}.}
{\it
  Let $w \in \bbR^W$
  and assume that $\adf{\sigma_{l,i}}$ is {consistent} for all $(l,i) \in \Idx$.
  If $y_{l,i}(w) \notin \ncdf{\sigma_{l,i}}$ for all $(l,i) \in \Idx$
  such that $\tau_l$ does not have bias parameters, 
  then
  \begin{align*}
    \ADF{z_L}(w) =
    \begin{cases}
      \DF{z_L}(w) & \text{if $\DF{z_L}(w) \neq \bot$}
      \\
      \begin{array}{@{}l@{}}
        \lim_{n \to \infty} \DF{z_L}(w'_n)
        \\[-2pt]
        \;\;\;\text{for some $w'_n \to w$}
      \end{array}
      & \text{if $\DF{z_L}(w) = \bot$}
    \end{cases}
  \end{align*}
  and so $\ADF{z_L}(w)$ is a Clarke subderivative of $z_L$ at $w$.
}
\begin{proof} 
  Let $w \in \bbR^W$ that satisfies the assumption in the statement.
  By \Cref{lem:r-gamma-partition},
  there is $\gamma \in \Gamma$ such that $w \in \mcR^\gamma$.
  Note that \Cref{lem:int-r-gamma-suff-dir} is applicable to $(\gamma, w)$ because:
  $\smash{\adf{\sigma_{l,i}}}$ is consistent for all $(l,i) \in \Idx$ (by assumption); 
  and for all $l \in [L]$, if $\tau_l$ does not have bias parameters, 
  then $y_{l,i}(w) \notin \ncdf{\sigma_{l,i}}$ and so $y_{l,i}(w) \notin \smash{\pbd(\mcI_{l,i}^{\gamma(l,i)})}$ for all $i \in [N_l]$,
  where the former follows from the assumption and the latter from $\smash{\pbd(\mcI_{l,i}^{\gamma(l,i)})} \subseteq \ncdf{\sigma_{l,i}}$
  (by \Cref{def:minrep-sigma-c1}).
  Hence, \Cref{lem:int-r-gamma-suff-dir} implies that $\intr(\mcR^\gamma)$ has sufficient directions around $w$,
  which subsequently implies the conclusion by \Cref{lem:ad-consistent}.
\end{proof}


\end{document}
